%% file: main.tex
\documentclass[10pt,twocolumn]{article}
\input{preamble.tex}
\title{\bfseries Rubix: Global Correspondence-Free Point Set Alignment\\ through Assignment Geometry\\{\large\mdseries A Preprint}}
\author{Subhransu S. Bhattacharjee, Dylan Campbell, Rahul Shome\\
\normalsize School of Computing, Australian National University\\
\normalsize Corresponding author: \href{mailto:Subhransu.Bhattacharjee@anu.edu.au}{\nolinkurl{Subhransu.Bhattacharjee@anu.edu.au}}}
\date{October 8, 2026}

\begin{document}
\raggedbottom
\onecolumn
\maketitle
\begin{center}
\begin{minipage}{\textwidth}
\input{sections/abstract.tex}
\input{sections/teaser.tex}
\end{minipage}
\end{center}
\twocolumn
\input{sections/introduction.tex}
\input{sections/related_work.tex}

\input{sections/notation.tex}
\input{sections/problem.tex}
\input{sections/theory.tex}
\input{sections/rubix.tex}
\input{sections/spatial.tex}
\input{sections/experiments.tex}

\clearpage
\twocolumn
\input{sections/discussion.tex}

\bibliographystyle{abbrv}
\bibliography{refs}

\input{appendix/content.tex}

\end{document}

%% file: preamble.tex
\usepackage[letterpaper,margin=0.875in]{geometry}
\usepackage[utf8]{inputenc}
\usepackage[T1]{fontenc}
\usepackage{etoolbox}
\patchcmd{\subsection}{\large}{\normalsize}{}
  {\PackageError{rubix}{Subsection font adjustment failed}{Check the article class subsection definition.}}
\usepackage{microtype}
\usepackage{amsmath,amssymb,amsthm}
\renewenvironment{abstract}
  {\small\begin{center}\bfseries\abstractname\end{center}\begin{quotation}\noindent\ignorespaces}
  {\end{quotation}}
\usepackage{graphicx}
\usepackage{placeins}

\usepackage{booktabs}
\usepackage{array}
\usepackage{xurl}
\usepackage{algorithm}
\usepackage[noend]{algorithmic}
\usepackage[font=small]{caption}
\usepackage{cite}
\AtBeginEnvironment{thebibliography}{\raggedright}
\usepackage{xcolor}
\definecolor{linkblue}{RGB}{0,72,186}
\usepackage{hyperref}
\usepackage[capitalise,noabbrev]{cleveref}
\crefname{algorithm}{Algorithm}{Algorithms}
\crefname{remark}{Remark}{Remarks}
\hypersetup{colorlinks=true,linkcolor=linkblue,citecolor=linkblue,urlcolor=linkblue,filecolor=linkblue}
\DeclareRobustCommand{\doi}[1]{\href{https://doi.org/#1}{doi:\,\nolinkurl{#1}}}
\hypersetup{pdftitle={Rubix: Global Correspondence-Free Point Set Alignment through Assignment Geometry},pdfauthor={Subhransu S. Bhattacharjee, Dylan Campbell, Rahul Shome},pdfcreator={},pdfproducer={}}

\newtheorem{theorem}{Theorem}
\newtheorem{proposition}[theorem]{Proposition}
\newtheorem{lemma}[theorem]{Lemma}
\newtheorem{corollary}[theorem]{Corollary}
\theoremstyle{definition}
\newtheorem{definition}[theorem]{Definition}
\newtheorem{remark}[theorem]{Remark}
\AddToHook{env/lemma/begin}{\crefalias{theorem}{lemma}}
\AddToHook{env/proposition/begin}{\crefalias{theorem}{proposition}}
\AddToHook{env/corollary/begin}{\crefalias{theorem}{corollary}}
\AddToHook{env/definition/begin}{\crefalias{theorem}{definition}}
\AddToHook{env/remark/begin}{\crefalias{theorem}{remark}}

\DeclareMathOperator*{\argmax}{arg\,max}

\DeclareMathOperator{\conv}{conv}
\DeclareMathOperator{\Arg}{Arg}
\renewcommand{\Re}{\operatorname{Re}}
\renewcommand{\Im}{\operatorname{Im}}
\newcommand{\PW}{\mathrm{PW}}

\input{numbers.tex}

\usepackage{enumitem}

%% file: numbers.tex
\newcommand{\FracTransposition}{26\%}

\newcommand{\FracEnclosing}{36\%}

\newcommand{\FixedGauss}{84}

\newcommand{\BnBCapped}{21\%}

%% file: sections/abstract.tex
\begin{abstract}
Procrustes--Wasserstein alignment jointly estimates a matching and rotation without supplied correspondences, but alternating minimization can stop at suboptimal solutions. Rubix solves the equally weighted planar problem globally under squared Euclidean loss. Each matching $\sigma$ of two centered $n$-point sets defines a complex correlation $z_\sigma=\sum_i\bar x_i y_{\sigma(i)}$. Their convex hull is the permutation polygon: supporting vertices give optimal matchings at fixed rotations, and the farthest vertex gives the global alignment. We prove the sharp bound of $n(n-1)$ vertices for $n\ge2$, answering Rote's rotation--assignment open problem. In exact arithmetic, assignment queries recover the polygon in $\mathcal O(n^5)$ operations. Assignment-based bounds extend the approach to three-dimensional rotations and partial matching at a supplied translation through branch-and-bound. On timed MPEG-7 shape pairs, Rubix attains every numerical reference value in 12\,ms on average, 50 times faster than a rotation grid at the same accuracy. Its distances improve shape retrieval and noisy crystal classification over alternating minimization.
\end{abstract}

%% file: sections/teaser.tex
\centering
\captionsetup{type=figure}
\includegraphics[width=\linewidth]{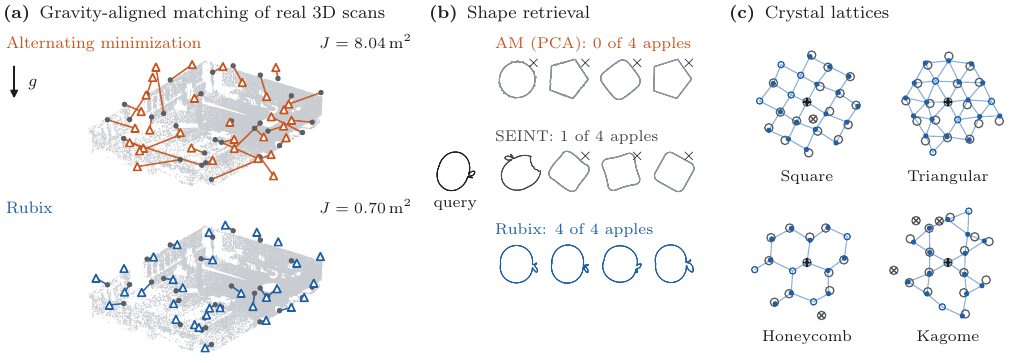}
\caption{\textbf{Global matching in scans, shapes and crystals.} (a)~An ETH Office scan 1--2 gravity subproblem: 32 target points (dark) and source fits from the same $45^\circ$ initial yaw (triangles), with AM and Rubix matchings; $g$ denotes gravity and $J$ mean squared matching error. (b)~Four nearest MPEG-7 silhouettes to an apple query under AM, SEINT~\cite{lin2026seint} (best setting) and Rubix. PW neighbors use fitted orientations; SEINT uses input orientations. (c)~Four defective lattice neighborhoods aligned by penalized partial Rubix to their nearest defect-free references, all correctly classified. Observations are rings, central sites cross-hairs, and reference atoms bonded. Matched reference atoms are solid, vacancies hollow and interstitials crossed (\Cref{sec:exp-unbalanced}).}\label{fig:gravity-application}

%% file: sections/introduction.tex
\section{Introduction}

\begin{figure*}[t]
\centering
\includegraphics[width=\textwidth]{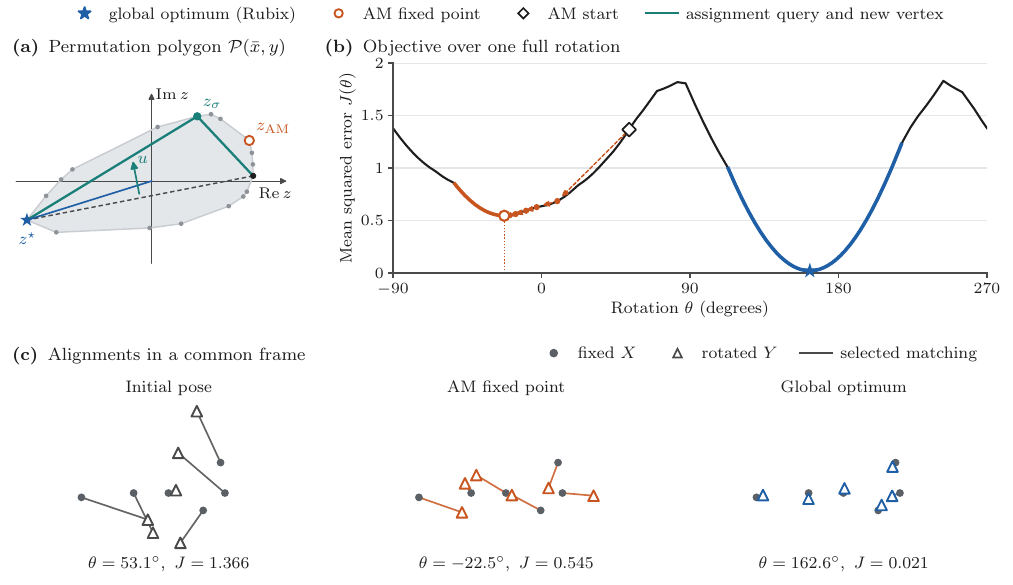}
\caption{\textbf{Rubix on an exact six-point example.} (a)~The permutation polygon is the convex hull of matching correlations $z_\sigma=\sum_i\bar x_iy_{\sigma(i)}$. An assignment query in direction $u$ finds an extreme correlation beyond an unexplored chord (dashed); 36 queries recover all 18 vertices. The farthest vertex $z^\star$ gives the optimal matching and rotation $\theta^\star=-\Arg z^\star$; AM stops at $z_{\mathrm{AM}}$. (b)~The objective $J(\theta)=\min_\sigma\{S-2\Re(e^{\mathrm i\theta}z_\sigma)\}/n$ over one turn, where $S=\sum_i|x_i|^2+\sum_j|y_j|^2$. Each smooth piece belongs to one polygon vertex. AM (orange) reaches a local minimum; dashed arrows join its reassigned iterates. The global minimum is blue. Rubix never samples this illustrative landscape. (c)~Fixed $X$ and rotated $Y$ at the initial pose, AM fixed point and global optimum, each with its minimizing matching, in a common frame and scale.}\label{fig:overview}
\end{figure*}

Shape retrieval~\cite{latecki2000}, landmark registration~\cite{zikan1991}, laser-scan mapping~\cite{glores2020,kubelka2022} and crystal identification~\cite{larsen2016ptm} require comparing point sets with unknown correspondences. Procrustes--Wasserstein (PW) alignment jointly optimizes a transformation and a transport coupling~\cite{alvarezmelis2019,adamo2025}. Here, correspondence-free means that no correspondences are supplied: Rubix optimizes them internally. We begin with equally sized, uniformly weighted planar sets and squared Euclidean loss, so an optimal coupling can be chosen as a bijection. Centering removes translation, and nonnegative isotropic scale has a closed-form fit.

The matching and rotation are coupled unknowns. Fixing the rotation gives a linear assignment problem; fixing the matching gives a closed-form orthogonal fit~\cite{schonemann1966}. Alternating minimization (AM)~\cite{adamo2025} can stop at a suboptimal fixed point and change retrieval rankings. Branch-and-bound~\cite{li2007registration,hartley2009rotation} gives global guarantees, but bounds that allow each residual its own rotation can require extensive subdivision. Zikan's exact rotation sweep~\cite{zikan1991} showed roughly quadratic growth in network pivots experimentally, without a polynomial worst-case bound. \emph{Rote's rotation--assignment problem}~\cite{mitchell2009} asks whether polynomially many matchings contain an optimum for every rotation. General affine parametric assignment can require superpolynomially many~\cite{carstensen1983thesis,hrubes2023}; a polynomial bound must exploit rotation structure.

Rubix represents each matching by its complex \emph{correlation} (\Cref{fig:overview}). Rotation changes only the projection direction of these fixed correlations. Their convex hull, the \emph{permutation polygon}, therefore suffices for alignment: a supporting vertex gives an optimal matching at each rotation, and the farthest vertex gives the global alignment. Lexicographic assignment queries recover supporting vertices without enumerating matchings. Our sharp quadratic vertex bound makes full reconstruction polynomial. The same geometry explains AM: under our assignment initialization and tie rules, a nonzero initial vertex leads to a stationary vertex, which need not be farthest from the origin.

For partial and three-dimensional alignment, assignment queries instead bound regions of pose space for branch-and-bound. Queries at the vertices of an enclosure evaluate all matched pairs at a common pose parameter, retaining information lost by independent residual bounds; each returned matching also gives a feasible alignment. Known gravity reduces rotation to yaw; general rotations use unit-quaternion regions. Partial matching requires either a supplied translation or joint pose search. We also adapt the common-translation bound to GLORES's trimmed registration objective~\cite{glores2020}, preserving its objective and stopping rule. The guarantees concern matching cost; pose accuracy is evaluated separately.

\paragraph{Contributions.} (i)~We prove the sharp $n(n-1)$ vertex bound, including ties and repeated points, answering Rote's question~\cite{mitchell2009} (\Cref{sec:theory}). (ii)~Rubix solves planar alignment exactly in $\mathcal O(n^5)$ arithmetic operations, with at most $2n(n-1)$ assignment queries for rotations or $4n(n-1)$ with reflections, and characterizes AM fixed points (\Cref{sec:algorithms}). (iii)~Returned matchings and assignment duals verify global optimality exactly for rational inputs (\Cref{prop:optimality}). (iv)~For gravity-aligned, partial and full 3D alignment, we derive pose-region bounds, accuracy-dependent query counts and reusable certificates; the coupled yaw bound is second order, versus first order for per-correspondence and independent-edge bounds (\Cref{sec:spatial}). (v)~Experiments test the distances in retrieval, crystal classification and clustering, partial matching under incomplete overlap, and the bounds inside planar and spatial registration solvers (\Cref{sec:experiments}). Full proofs, search algorithms and experimental details appear in the appendices.

%% file: sections/related_work.tex
\section{Related Work}\label{sec:related}

\paragraph{PW alignment.} Alvarez-Melis et~al.~\cite{alvarezmelis2019} formulate transport under global invariances; Grave et~al.~\cite{grave2019} combine convex initialization with stochastic updates; Even et~al.~\cite{even2024} analyze Ping-Pong in a planted Gaussian model; and Adamo et~al.~\cite{adamo2025} study PW metrics and barycenters. Semidefinite relaxations recover noiseless isometric shapes that are asymmetric or bilaterally symmetric~\cite{maron2016} and certify planar landmark localization when tight, usually at moderate noise~\cite{korotkine2025}. SEINT~\cite{lin2026seint} compares SE($d$)-invariant transport features without alignment. Rubix gives an exact polynomial algorithm for equal-weight planar PW with squared loss.

\paragraph{Exact planar registration.} Zikan~\cite{zikan1991} sweeps rotations with unknown bijections, translation and optional scale, without a polynomial worst-case bound. Utriainen~\cite{utriainen2010} studies vector-weighted matching, approximation and global optimization. Polynomial enumeration for supplied pairs covers consensus with pseudoconvex residuals~\cite{olsson2008} and planar similarity with $\ell_1$ or truncated losses~\cite{ask2014}; our bound concerns squared-error bijections with unknown correspondences.

\paragraph{Assignment geometry.} The permutation polygon is the classical $c$-eigenpolygon of a diagonal matrix~\cite{marcus1979,li1994}. We extend the real-weight permutohedron bound of $n(n-1)$ vertices~\cite{ziegler1995,hrubes2023} to complex weights, ties and repeated points. This makes Eisner--Severance enumeration~\cite{eisner1976} polynomial for planar rotation scores, despite superpolynomial general parametric assignments~\cite{carstensen1983thesis,hrubes2023}. Polynomial enumeration of optimal partial matchings under translation remains open~\cite{benavraham2018}. Soubrier et~al.~\cite{soubrier2026} also project transport plans into covariance space, obtaining concave multilinear approximations of Gromov--Wasserstein (GW) distances in fixed dimension with bounded support and uniformly positive definite second moments; our PW reconstruction is exact.

\paragraph{Global assignment bounds.} Li--Hartley~\cite{li2007registration} bound centered equal-size registration by assignments and per-point residual changes. Related searches cover rotation space~\cite{hartley2009rotation}, geometric matching~\cite{breuel2003}, and camera pose with unknown correspondences~\cite{campbell2020gopac}. Lian--Zhang~\cite{lian2014concave} eliminate transformations and bound a concave correspondence objective by cardinality-constrained assignment. APM~\cite{lian2017} uses a low-rank concave quadratic when every source point has a counterpart; RPM-PA~\cite{lian2021polyhedral} uses polyhedral annexation. Later bounds use trilinear and bilinear relaxations~\cite{lian2023}, assignment with convex quadratics~\cite{lian2026rpmbp}, or transformation-box vertices~\cite{lian2026dcreg}. Rubix combines covariance geometry, quaternion fitting~\cite{horn1987}, rotation orbitopes~\cite{sanyal2011orbitopes,saunderson2015rotations} and positive-weight rational B\'ezier enclosures~\cite{piegl1997nurbs}. Its spatial searches terminate at positive tolerance without inheriting the planar polynomial bound.

\paragraph{Registration with gravity.} GLORES~\cite{glores2020} searches trimmed SE(2) registration globally; PLICP~\cite{censi2008} alternates exact point-to-line updates for fixed pairs. Both allow repeated target use. Known gravity enables constrained point-to-plane ICP~\cite{kubelka2022}, pair pruning followed by translation search with optimal yaw~\cite{cai2019}, and cylindrical-norm registration~\cite{abe2022}. Li et~al.~\cite{li2023gravity,li2024screw} develop deterministic decompositions and consensus searches, including correspondence-free translation from angles to gravity~\cite{li2023gravity} and subsequent gravity-based outlier rejection~\cite{li2026gravity}. Quatro~\cite{lim2022quatro} reduces TEASER++~\cite{yang2021teaser} to yaw; Quatro++~\cite{lim2024quatro} adds ground segmentation and tilt compensation. OptiPose~\cite{ivanov2025} minimizes truncated least squares over supplied pairs about a fixed axis. Our comparisons distinguish these objectives and input requirements.

\paragraph{Unmatched observations.} Trimmed ICP~\cite{chetverikov2002} selects nearest-neighbor residuals; partial transport~\cite{qin2022} controls matched mass. GW compares internal relations, with global methods for low-dimensional Euclidean instances~\cite{ryner2023}. Unbalanced GW~\cite{sejourne2021} relaxes marginal conservation, partial GW~\cite{chapel2020partial} fixes transported mass, and fused unbalanced GW~\cite{thual2022fugw} adds feature costs. These nonconvex coupling objectives differ from our rigid fit with a total-variation penalty on unmatched mass (\Cref{rem:unbalanced}). Cost-regularized unbalanced transport~\cite{pardini2026} jointly optimizes couplings and bounded-norm inner-product costs, guaranteeing stationary points.

\paragraph{Registration systems.} Generalized-CVO~\cite{zhang2026gcvo} locally optimizes kernel correlation on SE(3); Register Any Point (RAP)~\cite{pan2026rap} learns registration by flow matching. PREDATOR~\cite{huang2021predator} learns overlap attention, and GeoTransformer~\cite{qin2022geotransformer} matches superpoints through learned geometry. Coordinate-only Rubix can refine these systems. We report complete systems for context and isolate matching through a common stage after Quatro, RAP and Li et~al., with fixed subsets, initialization and outer processing (\Cref{sec:published-gravity-insertion}).

%% file: sections/notation.tex
\section{Notation and Preliminaries}\label{sec:notation}\label{sec:preliminaries}

\paragraph{Point sets and poses.} Write $X=(x_1,\ldots,x_n)$ and $Y=(y_1,\ldots,y_m)$ in $\mathbb R^d$, also viewed as matrices with point columns. Their means are $m_x$ and $m_y$; $x_i^c=x_i-m_x$ and $y_j^c=y_j-m_y$ are centered coordinates. A pose maps $Y$ toward $X$ by $x_i\approx sRy_j+t$, with $R\in\mathrm{SO}(d)$ or $R\in\mathrm O(d)$, translation $t\in\mathbb R^d$ and isotropic scale $s\ge0$. The orthogonal group $\mathrm O(d)$ satisfies $R^\top R=I$; its rotation subgroup $\mathrm{SO}(d)$ has determinant $+1$. We state when a cited method reverses the point-set roles.

\paragraph{Planar coordinates.} Identify $\mathbb R^2$ with $\mathbb C$ and represent point sets by $x,y\in\mathbb C^n$, with $\|x\|^2=\sum_i|x_i|^2$. For $z\in\mathbb C$, $\bar z$, $\Re z$ and $\Im z$ denote its conjugate, real and imaginary parts; $\Arg z\in(-\pi,\pi]$ is the principal argument for $z\ne0$, and $\langle z,w\rangle=\Re(\bar zw)$ is the real inner product. Rotations act by $z\mapsto e^{\mathrm i\theta}z$ and orientation-reversing orthogonal maps by $z\mapsto e^{\mathrm i\theta}\bar z$.

\paragraph{Matchings and assignment.} For $n=m$, a complete matching is a permutation $\sigma$ in the symmetric group $S_n$, with matrix $P_\sigma$. Given real scores $(c_{ij})$, linear assignment chooses $\sigma\in S_n$ maximizing $\sum_ic_{i,\sigma(i)}$. The Birkhoff polytope $B_n$ contains the nonnegative matrices with unit row and column sums~\cite{brualdi1977}. These are the convex combinations of permutation matrices, so linear scores have the same maximum over $B_n$ and its permutation matrices. A partial matching $M$ contains $k=|M|$ pairs with no repeated row or column. Trimmed nearest-neighbor objectives instead fix the number of pairs and may repeat targets (\Cref{sec:joint-gravity,sec:native-planar}). We use $P$ for a fractional plan between unit point masses and $\pi\in\Pi(\alpha,\beta)$ for a transport plan with marginals $\alpha$ and $\beta$.

\paragraph{Convex geometry.} The convex hull $\conv A$ is the set of convex combinations of points in $A$. A polygon vertex is not a convex combination of its other points; a nontrivial segment has two vertices and a singleton one. For a convex polygon $\mathcal Q\subset\mathbb C$ and nonzero direction $u\in\mathbb C$, the \emph{support function} $h_{\mathcal Q}(u)=\max_{z\in\mathcal Q}\Re(\bar uz)$ is its largest projection onto $u$. The line $\Re(\bar uz)=h_{\mathcal Q}(u)$ bounds $\mathcal Q$ on the side of smaller projections. Its intersection with $\mathcal Q$ is the supporting face: a vertex or edge when $\mathcal Q$ is two-dimensional. For an edge, $u$ is its outward normal. Throughout the planar theory, $u$ is complex, with $u=e^{-\mathrm i\theta}$ for rotation $\theta$; in \Cref{sec:gravity}, it denotes the real yaw coordinates $u(\theta)=(\cos\theta,\sin\theta)$.

\paragraph{Costs, bounds and tolerances.} We use $E$ for summed squared residual, $\mathcal L_\sigma$ for mean planar matching cost, and $J$ for a plotted objective normalized as its caption specifies. The \emph{incumbent} is the best feasible solution found, with cost $U$; $L$ or $L_\Omega$ bounds the residual from below, enclosing the optimum in $[L,U]$. For total centered energy $S$ and correlation upper bound $b_\Omega$, the residual lower bound is $L_\Omega=\max\{0,S-2b_\Omega\}$: residual gaps are twice score gaps. A mean-error tolerance $\varepsilon_{\mathrm{mean}}$ gives a summed tolerance $n$ times larger for complete matching, or $k$ times larger for fixed positive cardinality $k$. Certificates require exact arithmetic; floating-point reference values and gaps are numerical.

\paragraph{Further conventions.} The symbols $\|\cdot\|$, $\|\cdot\|_F$ and $\langle\cdot,\cdot\rangle_F$ denote the Euclidean norm, Frobenius norm and Frobenius inner product; $\operatorname{tr}$, $(\cdot)^\dagger$ and $\lambda_{\max}$ denote trace, Moore--Penrose inverse and largest eigenvalue. A star marks an optimum or optimizer and a hat a reconstructed quantity; neither implies uniqueness. Arithmetic complexity $\mathcal O(\cdot)$ is distinct from the orthogonal group $\mathrm O(d)$. Local symbols are defined where used.

%% file: sections/problem.tex
\section{Problem Formulation}\label{sec:setup}

At a fixed rotation, the best matching is the solution of a linear assignment problem. This section shows that the rotation enters this problem only as a direction of projection. We encode each matching by one complex number, its \emph{correlation}; a rotation then selects a projection of the correlations, and the joint optimum over rotations and matchings is a farthest point of their convex hull.

\subsection{Complete Planar PW Alignment}
We consider complete matching of two planar point sets $X=(x_1,\ldots,x_n)$ and $Y=(y_1,\ldots,y_n)$ in $\mathbb R^2$, with $n\ge2$ points each, equal weights $1/n$ and possibly repeated points. The transformations are the rotations, $G=\mathrm{SO}(2)$, or the rotations and reflections, $G=\mathrm O(2)$. The squared PW objective is
\begin{equation}\label{eq:pw}
\PW_G^2(X,Y)=\min_{R\in G,\,\sigma\in S_n}
\frac1n\sum_{i=1}^n\|x_i-Ry_{\sigma(i)}\|^2.
\end{equation}
We seek a global minimizer $(R^\star,\sigma^\star)\in G\times S_n$ and the optimal value; if several minimizers tie, any of them is a valid answer, and the case $n=1$ is immediate. The square root of~\eqref{eq:pw} is a distance between unlabeled point sets modulo $G$. This extends the $\mathrm O(d)$ formulation of Adamo et~al.~\cite{adamo2025} to $\mathrm{SO}(2)$: composing transformations and matchings gives the triangle inequality, and compactness shows that the distance is zero only for sets that agree up to $G$ and relabeling. If translation is also allowed, the optimal translation for a fixed $R$ is $m_x-Rm_y$, and the problem reduces to the centered sets.

\subsection{Matching Correlations}\label{sec:correlations}
In complex coordinates, the correlation of a matching $\sigma$ is $z_\sigma(\bar x,y)=\sum_i\bar x_iy_{\sigma(i)}$. It carries the entire dependence of the matching's cost on the rotation:
\begin{equation}\label{eq:rotation-cost}
\begin{aligned}
\mathcal{L}_\sigma(\theta)&=\frac1n\sum_i|x_i-e^{\mathrm i\theta}y_{\sigma(i)}|^2\\
&=\frac{\|x\|^2+\|y\|^2-2\Re(e^{\mathrm i\theta}z_\sigma(\bar x,y))}{n}.
\end{aligned}
\end{equation}
The norm terms depend neither on the matching nor on the rotation. Minimizing the cost at rotation $\theta$ is therefore the same as maximizing the projection $\Re(e^{\mathrm i\theta}z_\sigma)$ over the correlations, and the maximum of a projection over a finite set is attained at a vertex of its convex hull. It is convenient to work with general coefficient vectors $a,b\in\mathbb C^n$. We define $z_\sigma(a,b)=\sum_{i=1}^na_ib_{\sigma(i)}$ and the hull
\begin{equation}\label{eq:polygon}
\begin{aligned}
\mathcal{P}(a,b)&=\conv\{z_\sigma(a,b):\sigma\in S_n\}\\
&=\{a^\top Pb:P\in B_n\},
\end{aligned}
\end{equation}
which we call the \emph{permutation polygon}. The second equality holds because the linear map $P\mapsto a^\top Pb$ sends the Birkhoff polytope, whose vertices are the permutation matrices, onto this hull. The expression $a^\top Pb$ is bilinear and involves no conjugation; alignment uses $a=\bar x$ and $b=y$.

To evaluate the polygon at rotation $\theta$, we set $(a,b)=(\bar x,y)$ and $u=e^{-\mathrm i\theta}$. A linear assignment with scores $c_{ij}=\Re(\bar ua_ib_j)$ maximizes $\Re(\bar uz_\sigma)$. It returns the support value $h_{\mathcal P}(u)$ and an optimal matching at this rotation without forming the $n!$ correlations. Keeping one matching for each vertex of the polygon suffices to have an optimal matching at every rotation, including the directions in which an edge ties two vertices and the inputs for which several matchings share a correlation. Optimizing over the rotation as well replaces the projection of each correlation by its modulus, which gives the following reduction.

\begin{lemma}[reduction]\label{lem:reduction}
For $x,y\in\mathbb C^n$,
\begin{equation}\label{eq:pw-polygon}
\PW_{\mathrm{SO}(2)}^2(X,Y)=\frac{\|x\|^2+\|y\|^2}{n}-\frac2n\max_{z\in\mathcal{P}(\bar x,y)}|z|.
\end{equation}
A vertex $z_{\sigma^\star}$ attains the maximum. Its matching $\sigma^\star$ is optimal, with angle $-\Arg z_{\sigma^\star}$ when the correlation is nonzero. If the maximum modulus is zero, every angle is optimal. For $\mathrm{O}(2)$, take the larger maximum over $\mathcal{P}(\bar x,y)$ and $\mathcal{P}(\bar x,\bar y)$.
\end{lemma}

\begin{proof}[Proof sketch]
Maximizing $\Re(e^{\mathrm i\theta}z_\sigma)$ over $\theta$ in~\eqref{eq:rotation-cost} gives $|z_\sigma|$. The modulus is convex, so its maximum over the polygon is attained at a vertex, and every vertex is the correlation of a matching. \Cref{app:reduction} treats reflections and zero correlations.
\end{proof}

\subsection{Translation and Isotropic Scale}\label{sec:similarity}
For complete similarity alignment $x_i\approx sRy_{\sigma(i)}+t$ with $s\ge0$, the optimal translation is $t=m_x-sRm_y$, which leaves the centered coordinates $x_i^c=x_i-m_x$ and $y_j^c=y_j-m_y$. Let $S_X=\sum_i\|x_i^c\|^2$ and $S_Y=\sum_j\|y_j^c\|^2$, and let $r_\sigma=|z_\sigma|$ be the modulus of a correlation in either of the centered polygons of \Cref{lem:reduction}. Once the rotation is optimized, the residual of the matching $\sigma$ at scale $s$ is $S_X+s^2S_Y-2sr_\sigma$. For $S_Y>0$, the minimizing scale and the residual are
\begin{equation}\label{eq:similarity}
s_\sigma^\star=\frac{r_\sigma}{S_Y},\qquad
E_\sigma^\star=S_X-\frac{r_\sigma^2}{S_Y}.
\end{equation}
The residual decreases as $r_\sigma$ grows, so a farthest vertex also solves similarity alignment~\cite{zikan1991}. Three degenerate cases remain. If $S_Y=0$, the translation places every source point at $m_x$, the residual is $S_X$, and rotation and scale are not identifiable. If $S_Y>0$ and the largest correlation modulus is zero, the scale $s=0$ attains the minimum when it is allowed, whereas requiring $s>0$ leaves an infimum that is not attained. If $S_Y>0$ and the scale is restricted by fixed bounds $0<s_-\le s\le s_+<\infty$, we clip $r_\sigma/S_Y$ to $[s_-,s_+]$; the minimized residual remains nonincreasing in $r_\sigma$. Dividing by $n$ gives the normalization of~\eqref{eq:pw}. Partial registration does not reduce in this way, because it requires the means of the selected subsets.

\subsection{The Structural Restriction}
Solving each assignment problem in polynomial time does not bound how often the optimal matching changes as a parameter varies. Consider scores $D_0+\eta_{\mathrm{sh}}D_1$ that depend affinely on a real parameter $\eta_{\mathrm{sh}}$, with arbitrary $D_0,D_1\in\mathbb R^{n\times n}$, and define
\begin{equation}\label{eq:general-shadow}
\operatorname{sh}(P)=\big(\langle D_0,P\rangle_F,\langle D_1,P\rangle_F\big),
\end{equation}
where $\langle D,P\rangle_F=\sum_{ij}D_{ij}P_{ij}$. The image $\operatorname{sh}(B_n)$ is a planar projection, or \emph{shadow}, of the Birkhoff polytope. Maximizing the score at the parameter $\eta_{\mathrm{sh}}$ exposes a face of the shadow in the direction $(1,\eta_{\mathrm{sh}})$, so varying the parameter traverses a chain of boundary vertices. Arbitrary shadows can have $n^{\Omega(\log n)}$ vertices~\cite{hrubes2023}. Rotation scores are more restricted. The two coordinates of their shadow are the real and imaginary parts of $a^\top Pb$, whose complex coefficient matrix $ab^\top$ has rank at most one, although the two real score matrices can have rank two. \Cref{thm:main} bounds the number of vertices for this restricted class.

%% file: sections/theory.tex
\section{\texorpdfstring{A Sharp $n(n-1)$ Bound}{A Sharp n(n-1) Bound}}\label{sec:theory}

\begin{figure*}[t]
\centering
\includegraphics[width=\textwidth]{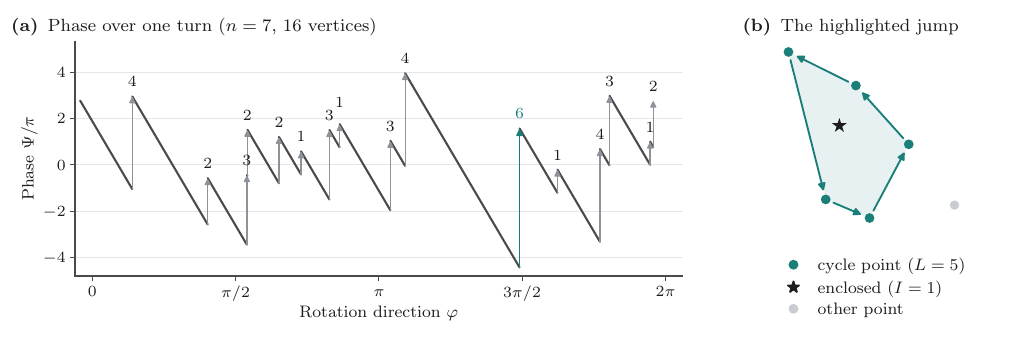}
\caption{\textbf{A fixed phase budget bounds the number of matching changes} ($n=7$). (a)~While one matching stays optimal, the phase $\Psi$ decreases with slope $-\binom n2$; when the optimal matching changes, it jumps by $\Delta\Psi=\pi(L-1+2I)$. The number above each jump is $\Delta\Psi/\pi$. Over one turn, the continuous decrease and the jumps both total $42\pi$. Each jump is at least $\pi$, so the 16 polygon vertices of this example satisfy the bound $n(n-1)=42$. (b)~The highlighted jump of $6\pi$: its reassignment cycle visits $L=5$ points counterclockwise and encloses $I=1$ further point (star).}\label{fig:phase}
\end{figure*}

The permutation polygon is defined by $n!$ correlations, yet exact alignment needs at most $n(n-1)$ of its vertices. We count the vertices by turning a supporting line through one full turn: the vertex that it touches changes each time its direction passes the outward normal of an edge. \Cref{sec:algorithms} then recovers these vertices by assignment queries.

\begin{theorem}\label{thm:main}
For every $n\ge2$ and all $a,b\in\mathbb C^n$, the polygon $\mathcal{P}(a,b)$ has at most $n(n-1)$ vertices. Consequently, for any two equally weighted planar sets of $n$ points there are at most $n(n-1)$ matchings such that, for every rotation, one of them is optimal. For generic point sets, that is, when $(\bar x,y)$ is generic in the sense of \Cref{def:generic}, the optimal matching is unique for all but finitely many angles and changes at most $n(n-1)$ times over a full turn. The bound is attained for every $n$ (\Cref{prop:tight}).
\end{theorem}

The proof follows one real quantity, the \emph{phase} of the optimal matching, as the supporting direction turns. The phase is the sum, over all pairs of rows, of the angle of the swap gain of the pair, the amount by which the correlation of the matching exceeds that of the matching with the two partners swapped (\Cref{sec:phase-argument}). While one matching stays optimal the phase decreases at a fixed rate, and each change of the optimal matching increases it by at least $\pi$. Over a full turn the phase returns to its initial value, so the jumps must restore a total decrease of $\pi n(n-1)$, which allows at most $n(n-1)$ changes. We first work with \emph{generic} $(a,b)$, for which the correlations of distinct matchings are distinct and no three of them are collinear, so that each transition involves exactly two matchings. Generic pairs are dense and have full measure (\Cref{lem:generic,app:genericity}), and a perturbation argument then extends the bound to all inputs. \Cref{app:counting-bound} gives the complete proof of \Cref{thm:main}.

\subsection{Optimal-Matching Transitions}\label{sec:conservation}
Two permutations $\sigma$ and $\tau$ differ by the permutation $\sigma^{-1}\tau$, which decomposes into disjoint cycles; each cycle passes partners around a set of rows. A generic transition consists of a single cycle, and a swap of two partners is the case of length two.
\begin{lemma}[cycle transitions]\label{lem:cycle}
For generic $(a,b)$ and $u\in\mathbb C$ with $|u|=1$, maximizing $\Re(\bar uz_\sigma)$ gives either a unique permutation or two permutations $\sigma,\tau$ for which $\sigma^{-1}\tau$ is a single nontrivial cycle.
\end{lemma}

\begin{proof}[Proof sketch]
Under genericity, a supporting line contains at most two correlations, so at most two matchings are optimal in any direction. If two are optimal, switching any single cycle of their difference leaves the projection unchanged: the cycles act on disjoint rows, so their changes of the projection add up to zero, and none of them is positive. Two nontrivial cycles would therefore produce a third optimal matching, a contradiction. The full proof is in \Cref{app:cycle-transitions}.
\end{proof}

For generic inputs with $n\ge3$, we traverse the boundary of $\mathcal P$ counterclockwise. By \Cref{lem:cycle}, the matchings at the two ends of an edge $e$ from $z_\sigma$ to $z_\tau$ differ by one cycle $\sigma^{-1}\tau=(i_1\,\dots\,i_{L_e})$ of length $L_e$. We set $j_r=\sigma(i_r)$, with indices taken modulo $L_e$, so that row $i_r$ gives up the point $b_{j_r}$ and receives $b_{j_{r+1}}$. Joining the points $b_{j_1},\ldots,b_{j_{L_e}}$ in cycle order gives a closed polygon $\Gamma_e$. This polygon lies in the plane of the input points, whereas the edge $e$ that defines it lies in the plane of correlations.

Two geometric properties of $\Gamma_e$ determine the phase jump at $e$. First, an optimal dual solution of the assignment problem at the transition defines a power diagram~\cite{aurenhammer1987} on the points $b_j$, and two points can exchange partners at the transition only if their cells are neighbors. The segments that join neighboring cells do not cross, so a cycle of length at least three forms a simple polygon. Second, traversing $\mathcal P$ counterclockwise fixes the orientation of the cycle, which is also counterclockwise. Both properties are proved in \Cref{sec:transition-geometry,app:jump}.

\subsection{The Phase Argument}\label{sec:phase-argument}

For rows $i<k$, we define the \emph{swap gain}
\begin{equation}\label{eq:swap-gain}
\Delta z_{ik}(\sigma)=(a_i-a_k)(b_{\sigma(i)}-b_{\sigma(k)}),
\end{equation}
which is the difference between the correlation of $\sigma$ and that of the matching obtained by swapping the partners of rows $i$ and $k$. Genericity makes every swap gain nonzero. If $\sigma$ is optimal in direction $u$, no swap can improve it, so $\Re(\bar u \Delta z_{ik}(\sigma))\ge0$ for every pair, and every term of the phase
\begin{equation}\label{eq:phase}
\Psi(\sigma,u)=\sum_{i<k}\Arg\big(\bar u\,\Delta z_{ik}(\sigma)\big)
\end{equation}
lies in $[-\pi/2,\pi/2]$. The principal argument is continuous on the closed right half-plane without the origin, so no term crosses a branch cut while $\sigma$ stays optimal. Between transitions, each term therefore decreases at unit rate as $u$ turns counterclockwise, and the $\binom n2$ terms together decrease by $2\pi\binom n2=\pi n(n-1)$ over a full turn. A full turn that starts away from an edge normal ends with the same unique optimal matching and the same phase, so the jumps at the transitions must balance this continuous decrease exactly.

\begin{lemma}[jump]\label{lem:jump}
Let $(a,b)$ be generic with $n\ge3$, and let the edge $e$ of $\mathcal{P}$ lead counterclockwise from $z_\sigma$ to $z_\tau$, with outward unit normal $u$ and cycle length $L_e$. For $L_e\ge3$, the polygon $\Gamma_e$ is simple and counterclockwise, and $I_e$ counts the off-cycle points $b_{j_0}$ strictly inside it. Set $I_e=0$ for $L_e=2$. Then
\begin{equation}\label{eq:jump}
\Psi(\tau,u)-\Psi(\sigma,u)=\pi(L_e-1+2I_e).
\end{equation}
\end{lemma}

\begin{proof}[Proof sketch]
For $L_e\ge3$, we move each reassigned point linearly from its old position to its new one along the corresponding edge of $\Gamma_e$. The power-diagram construction makes $\Gamma_e$ simple and counterclockwise, with no off-cycle point on its boundary. The moving points do not collide, because they follow distinct edges of the polygon and reach shared endpoints at different times. Optimality at both ends of the motion keeps every swap gain, multiplied by $\bar u$, in the closed right half-plane, so the contribution of each pair to the phase jump equals the continuous change of its argument. To evaluate these changes, we deform $\Gamma_e$ continuously into a regular polygon, which shows that the pairs of moving points contribute $\pi(L_e-1)$ in total. Each off-cycle point enclosed by $\Gamma_e$ contributes a further $2\pi$, and each off-cycle point outside contributes zero. For $L_e=2$, the swapped pair contributes $\pi$ directly, and its contributions relative to each off-cycle point cancel. Adding these terms gives~\eqref{eq:jump}. The full proof, including the orientation of the cycle and the two-point case, is in \Cref{app:jump}.
\end{proof}

\begin{lemma}[counting identity]\label{lem:counting}
For $n\ge3$ and generic $(a,b)$, every $\Gamma_e$ with $L_e\ge3$ is a simple, counterclockwise polygon, and
\begin{equation}\label{eq:counting}
\sum_{\text{edges }e\text{ of }\mathcal{P}}\big(L_e-1+2I_e\big)=n(n-1),
\end{equation}
with $L_e$ and $I_e$ as in \Cref{lem:jump}.
\end{lemma}

\begin{proof}[Proof]
The continuous decrease of the phase~\eqref{eq:phase} over a full turn is $\pi n(n-1)$, and the jumps must restore it. Substituting~\eqref{eq:jump} for each jump and dividing by $\pi$ gives the identity. \Cref{fig:phase} illustrates the case $n=7$; \Cref{app:phase-total,app:counting-bound} give the full accounting.
\end{proof}

\subsection{The Sharp Bound}
Every term of the counting identity is at least one, which gives the bound for generic inputs. The next lemma transfers it to inputs with tied or repeated correlations: a sufficiently small generic perturbation cannot reduce the number of vertices.
\begin{lemma}[semicontinuity]\label{lem:semicont}
For every $(a,b)$ and $\delta_{\mathrm{pert}}>0$ there is a generic $(a',b')$ within distance $\delta_{\mathrm{pert}}$ with $\#\mathrm{vert}\,\mathcal{P}(a',b')\ge\#\mathrm{vert}\,\mathcal{P}(a,b)$.
\end{lemma}

\begin{proof}[Proof sketch]
For each vertex of the original polygon we choose a direction that exposes it strictly. The score gaps in these finitely many directions are positive and persist under sufficiently small perturbations, so each of the disjoint neighborhoods of the original vertices contains an extreme point of the perturbed polygon. Generic pairs are dense by \Cref{lem:generic}, so the perturbation can be chosen generic. The full proof is in \Cref{app:semicontinuity}.
\end{proof}

\begin{proof}[Proof of \Cref{thm:main}]
For generic inputs with $n\ge3$, every summand in \Cref{lem:counting} is at least one, so the polygon has at most $n(n-1)$ edges and hence at most $n(n-1)$ vertices. For arbitrary inputs, a polygon that is a single point satisfies the bound directly; otherwise \Cref{lem:semicont} gives a nearby generic polygon with at least as many vertices, which extends the bound to every input. For $n=2$ there are only two permutations. The statement about matchings holds because one representative matching per vertex contains an optimal matching for every direction (\Cref{sec:correlations}), and under genericity only the two adjacent vertices tie at an edge normal. \Cref{prop:tight} proves sharpness. The complete arguments for the bound and for sharpness are in \Cref{app:counting-bound,app:tightness}, respectively.
\end{proof}

The bound is attained by inputs whose row coefficients $a_i$ are real: the optimal matching in a direction is then determined by the order of the projections of the points $b_j$. The proof uses the Minkowski sum $\mathcal Q_1+\mathcal Q_2=\{z+w:z\in\mathcal Q_1,w\in\mathcal Q_2\}$, which satisfies $h_{\mathcal Q_1+\mathcal Q_2}(u)=h_{\mathcal Q_1}(u)+h_{\mathcal Q_2}(u)$ because the two terms are maximized independently.

\begin{proposition}[tightness]\label{prop:tight}
For $n\ge2$, if $a\in\mathbb R^n$ has distinct entries and $b\in\mathbb C^n$ has distinct entries with pairwise non-parallel differences $b_j-b_l$, $j<l$, then $\mathcal{P}(a,b)$ has exactly $n(n-1)$ vertices.
\end{proposition}

\begin{proof}[Proof sketch]
We order $a_1<\dots<a_n$, express each coefficient through its consecutive differences, and regroup the terms to obtain Abel's summation identity
\begin{equation}\label{eq:abel-summation}
z_\sigma=a_1\sum_jb_j+\sum_{k=2}^n(a_k-a_{k-1})\sum_{l\ge k}b_{\sigma(l)}.
\end{equation}
In a direction $u$, all terms are maximized together by sending the largest row indices to the points $b_j$ with the largest projections $\Re(\bar ub_j)$. Since support functions add under Minkowski sums, $\mathcal{P}$ is a translate of a Minkowski sum of positive multiples of the polygons $Q_h=\conv\{\sum_{j\in S}b_j:|S|=h\}$, $1\le h\le n-1$. For $n\ge3$, each pair of points exchanges its order of projection twice during a full turn. Because the pairwise differences are non-parallel, only one pair ties at a time, and the two tied points are adjacent in that order. Their exchange changes exactly one top-$h$ subset and thus gives one edge of one $Q_h$. These $n(n-1)$ edges have distinct normals and remain edges of the Minkowski sum, so $\mathcal{P}$ has that many vertices. For $n=2$, the two distinct correlations form a segment. The full construction and the edge count are proved in \Cref{app:tightness}.
\end{proof}

\Cref{thm:main} answers Rote's question~\cite{mitchell2009}. It counts representative matchings and not all optimal ones: when all points coincide, all $n!$ matchings are optimal, but one representative suffices. For generic inputs, the counting identity shows that equality holds exactly when every transition swaps two points. For generic inputs with $n=3$, the number of vertices is six minus the number of transitions that are $3$-cycles.

Rotations apply a common phase $e^{\mathrm i\theta}$ to every product $a_ib_j$. The diagonal weights $\operatorname{diag}(\cos\eta_{\mathrm{diag}},\sin\eta_{\mathrm{diag}})$ can instead give $16>n(n-1)=12$ vertices for $n=4$ (\Cref{rem:special}).

\subsection{Rational Weights and Unequal Cardinalities}
Because \Cref{thm:main} permits repeated points, replacing each weighted point by copies of equal mass extends the vertex bound to transport with rational weights and to uniformly weighted sets of different sizes.

\begin{corollary}[rational weights]\label{cor:weights}
Let $\mu_X=\sum_i\alpha_i\delta_{x_i}$ and $\mu_Y=\sum_j\beta_j\delta_{y_j}$ with nonnegative rational weights summing to one and common denominator $N\ge1$, where $\delta_x$ denotes a unit point mass at $x$. Remove zero-weight atoms, and let $\Pi(\alpha,\beta)$ be the nonnegative matrices with row sums $\alpha$ and column sums $\beta$. Then $\{\sum_{ij}\pi_{ij}\bar x_iy_j:\pi\in\Pi(\alpha,\beta)\}$ has at most $\max\{1,N(N-1)\}$ vertices. In particular, for uniform weights on $n$ and $m$ points, it has at most $\max\{1,N_{\mathrm{uni}}(N_{\mathrm{uni}}-1)\}$ with $N_{\mathrm{uni}}=\operatorname{lcm}(n,m)$.
\end{corollary}

\begin{proof}[Proof sketch]
Replace each atom by as many copies of mass $1/N$ as the integer numerator of its weight. Aggregating the doubly stochastic matrices over these copies gives exactly the original transport polytope, so the weighted correlation polygon is a scaled equal-weight polygon. \Cref{thm:main} applies to it because repeated points are permitted; $N=1$ gives a single point. The full proof that replication preserves the set of feasible correlations is in \Cref{app:weights}.
\end{proof}

Replication is pseudo-polynomial: the bound is polynomial in the denominator $N$, which can be exponential in its bit length. Uniform weights on $n$ and $m$ points give $N=\operatorname{lcm}(n,m)\le nm$, which is polynomial in the numbers of points. For equally weighted sets of size $n\ge2$, \Cref{lem:reduction,thm:main} show that at most $n(n-1)$ representative matchings per polygon are needed, degenerate inputs included. \Cref{sec:algorithms} finds them by assignment queries.

%% file: sections/rubix.tex
\section{The Rubix Algorithm}\label{sec:algorithms}
Rubix reconstructs the permutation polygon by assignment queries, selects a farthest vertex, and recovers the optimal matching and rotation through \Cref{lem:reduction}. Each query either discovers a new vertex or confirms an edge between two known vertices, so the number of queries is proportional to the number of vertices.
\subsection{Polygon Reconstruction}\label{sec:polygon-reconstruction}
Given $x$ and $y$, we first center them if the translation is unknown (\Cref{sec:similarity}). \Cref{alg:rubix} then reconstructs $\mathcal{P}(\bar x,y)$ and returns its $V$ vertices, each with one matching that attains it. A vertex $z$ of largest modulus gives the optimal angle $-\Arg z$, or any angle if $z=0$, and the optimal value through~\eqref{eq:pw-polygon}. When reflections are allowed, we also reconstruct $\mathcal{P}(\bar x,\bar y)$ and compare the two maxima (\Cref{lem:reduction}). The translation and, where it is permitted, the scale follow from \Cref{sec:similarity}. The pseudocode is written for general $a,b$ so that it serves both polygons.

The algorithm reaches the polygon through an \emph{assignment oracle}. For a nonzero direction $u$, $\mathrm{Oracle}(u)$ returns a correlation that maximizes $\Re(\bar uz)$ and, among these maximizers, one that maximizes the secondary projection $\Re(\overline{\mathrm i u}z)$. The primary maximizers form the supporting face in direction $u$, and the secondary projection selects an endpoint of this face. The oracle therefore always returns a vertex, also when $u$ is normal to an edge and when several matchings share a correlation. A cubic-time assignment solver implements this lexicographic order by working with pairs of costs that are compared lexicographically and added componentwise. Around the oracle, we use the enumeration scheme of Eisner and Severance~\cite{eisner1976}.

The algorithm starts with two queries in opposite directions. If they return the same correlation, the polygon is a single point. Otherwise the two outputs are vertices that split the boundary into two arcs, and the algorithm keeps a stack of chords $(z,z')$, each of which stands for an unexplored boundary arc that runs counterclockwise from $z$ to $z'$. For such a chord, it queries the outward normal direction $u=-\mathrm i(z'-z)$ and obtains a vertex $z''$. The \emph{support gap} $h_{\mathcal P}(u)-\Re(\bar uz)=\Re(\bar u(z''-z))$ is positive precisely when $z''$ lies beyond the chord, in which case the arc is split at $z''$. A zero gap certifies that the chord is an edge of the polygon. When the stack is empty, the polygon is complete.

\begin{algorithm}[t]
\caption{Rubix: polygon reconstruction}\label{alg:rubix}
\begin{algorithmic}[1]
\STATE $z_-\leftarrow\mathrm{Oracle}(-1)$; $z_+\leftarrow\mathrm{Oracle}(1)$.
\STATE If $z_-=z_+$, return the singleton $\{z_+\}$.
\STATE Stack $\leftarrow\{(z_-,z_+),(z_+,z_-)\}$.
\WHILE{Stack is not empty}
\STATE Pop $(z,z')$; $u\leftarrow-\mathrm i(z'-z)$.
\STATE $z''\leftarrow\mathrm{Oracle}(u)$.
\STATE If $\Re(\bar u(z''-z))>0$, push $(z,z''),(z'',z')$; otherwise record the boundary segment $(z,z')$.
\ENDWHILE
\STATE Return the endpoints of the recorded segments.
\end{algorithmic}
\end{algorithm}

\begin{proposition}[polygon reconstruction]\label{prop:es}
For $n\ge2$, \Cref{alg:rubix} returns the vertices of $\mathcal{P}$ using exactly $2V\le2n(n-1)$ calls to this oracle.
\end{proposition}

\begin{proof}[Proof sketch]
The two initial arcs are explored as two binary trees whose internal nodes are the splits and whose leaves are the confirmed edges. Together the trees have $V-2$ internal nodes and $V$ leaves, which with the two initial queries gives $2+(V-2)+V=2V$ calls. A single point needs two calls and a segment four, and the lexicographic tie-breaking keeps the returned points extreme for degenerate inputs. \Cref{app:reconstruction} proves the invariant and the count.
\end{proof}

\begin{corollary}\label{cor:pw}
For $n\ge2$, Rubix computes planar $\PW$ exactly with at most $2n(n-1)$ lexicographic assignment problems (oracle calls) over $\mathrm{SO}(2)$ and $4n(n-1)$ over $\mathrm{O}(2)$, and $\mathcal{O}(n^5)$ arithmetic operations. For $n=1$ the solution is immediate.
\end{corollary}

\begin{proof}[Proof sketch]
Apply \Cref{prop:es} to $\mathcal{P}(\bar x,y)$ and, over $\mathrm{O}(2)$, to $\mathcal{P}(\bar x,\bar y)$, then use \Cref{lem:reduction}. Each call costs $\mathcal{O}(n^3)$ operations; see \Cref{app:pw-complexity}.
\end{proof}

\input{sections/optimality_guarantees.tex}

\subsection{Comparison with AM}\label{sec:am-comparison}
The polygon also describes the behavior of AM~\cite{adamo2025}. From a matching with nonzero correlation $z$, AM fits the rotation and then solves an assignment problem, which is a support query in the radial direction $z/|z|$. We call a nonzero vertex $z$ \emph{stationary} if $\Re(\bar z(w-z))\le0$ for all $w\in\mathcal P$, that is, if its radial direction supports the polygon at $z$; this is the first-order condition for maximizing $|z|$. A \emph{fixed point} of AM is a state whose matching and rotation are unchanged by a full update. A stationary vertex is a fixed point when the assignment step keeps the current matching whenever it is optimal; other tie-breaking rules can select another maximizer. Stationary vertices need not be farthest from the origin, which is how AM stops at suboptimal alignments (\Cref{fig:overview}).
\begin{proposition}[finite convergence and fixed points of AM]\label{prop:am}
For $\mathrm{SO}(2)$, the correlation moduli produced by AM~\cite{adamo2025} never decrease. Suppose that AM is initialized by an assignment step and that each assignment step returns a vertex and keeps the current matching when it is optimal. Then every matching that AM visits has a vertex correlation, and such a matching with nonzero correlation is a fixed point exactly when its vertex is stationary. Starting from a nonzero vertex correlation, these updates reach a fixed matching after at most $V-1$ matching changes, with $V-1\le n(n-1)-1$ for $n\ge2$. A farthest vertex is globally optimal. At zero correlation, any rotation is a valid update, and if $\mathcal{P}=\{0\}$, all rotations are globally optimal.
\end{proposition}

\begin{proof}[Proof sketch]
For $z\ne0$, set $u=z/|z|$ and let $z'$ be the next correlation. Then $|z'|\ge\Re(\bar uz')\ge\Re(\bar uz)=|z|$, and under the retention rule the second inequality is strict whenever the matching changes. Each change therefore reaches a vertex that has not been visited before, which allows at most $V-1$ changes. A matching is retained exactly when its correlation is optimal in its own radial direction, which is the stationarity of its vertex. \Cref{app:am} treats ties and zero correlations.
\end{proof}

We next extend this assignment geometry to partial and spatial registration, where the sharp planar vertex bound does not apply.

%% file: sections/optimality_guarantees.tex
\subsection{Verification and Pruning}\label{sec:optimality}
The reconstructed hull can be verified independently of the solver that produced it. For inputs with rational coordinates, the query directions $u=-\mathrm i(z'-z)$ are rational without normalization, so all assignment scores are rational. A dual solution $(f,g)$ of an assignment problem is feasible if $f_i+g_j\ge c_{ij}$ for all $i,j$, and it then bounds the score of every matching by $\sum_i f_i+\sum_j g_j$. If a returned matching attains this bound, its support value is verified. Checking the outward normal direction of every boundary edge in this way rules out correlations beyond the reconstructed hull.

In total, Rubix uses $\mathcal{O}(n^3V)$ arithmetic operations, where $V$ is summed over both polygons when reflections are allowed. The bit complexity depends in addition on the bit lengths of the input.

\begin{proposition}[Exact global verification]\label{prop:optimality}
Let $a,b$ have rational real and imaginary parts, and let $\widehat{\mathcal P}$ be the hull of the correlations recomputed from the returned matchings, with its distinct extreme points and their boundary order verified exactly. Suppose that assignment duals certify the following support values. If $\widehat{\mathcal P}$ is two-dimensional, they certify its support value in the outward normal direction of every edge. If it is a segment, they certify its line through the two opposite normal directions and its endpoints through the two extrema along the line. If it is a single point, they certify that the two opposite primary extrema coincide and that the two opposite secondary extrema coincide. Then $\widehat{\mathcal P}=\mathcal P(a,b)$, and comparing the squared moduli of its vertices gives $\max_{z\in\mathcal P(a,b)}|z|^2$ exactly.
\end{proposition}
\begin{proof}[Proof sketch]
Every point used to form $\widehat{\mathcal P}$ is the correlation of a matching, so $\widehat{\mathcal P}\subseteq\mathcal P$. The verified supporting half-planes contain $\mathcal P$ and intersect in $\widehat{\mathcal P}$, which gives the reverse inclusion. The checks for a segment and for a single point establish the same inclusion in lower dimension. By convexity, a farthest point is a vertex. The full proof is in \Cref{app:verification}.
\end{proof}

Comparing squared moduli avoids square roots during the maximization; the squared PW value is $c-2r_\star/n$, where $c=(\|x\|^2+\|y\|^2)/n$ and $r_\star\ge0$ is the maximum modulus. In floating point, comparisons can merge nearby vertices, so the double-precision values that serve as references in our experiments are numerical and are not certificates.

\paragraph{Pruning.} When only the alignment is needed, most of the polygon need not be reconstructed. We skip an unexplored arc if an upper bound on the modulus over the arc shows that it cannot improve on the best vertex found so far. The bound uses the directions in which the two endpoints of the arc were found. When these directions span a counterclockwise angle in $(0,\pi)$, the arc lies in the triangle formed by its endpoints and the intersection of their supporting lines; if one endpoint lies on the supporting line of the other, the arc lies on the chord. Since the modulus is convex, it is bounded on the arc by its largest value at the corners. The two initial queries use opposite directions, for which this bound is not available. We refine an arc whenever its bound is unavailable or permits an improvement. With exact comparisons, pruning preserves the optimum and the worst-case query bound, although it may return only part of the polygon.

%% file: sections/spatial.tex
\section{Spatial Registration Bounds}\label{sec:spatial}

\input{sections/gravity_framework.tex}

The planar bound of $n(n-1)$ vertices has no known analogue for three-dimensional or partial alignment, so there we bound regions of pose space directly and search them by branch-and-bound. In each region, or \emph{cell}, a few assignment queries at shared directions bound the residual of every matching from below, and fitting the returned matchings gives feasible poses. We discard the cells that cannot improve on the incumbent and tighten or subdivide the rest. With a known gravity direction, the bound comes from the concavity of the smallest residual in suitable yaw coordinates. For unrestricted rotations, it comes from the convexity of the largest correlation score in a lifted quaternion variable, which bounds the score from above and hence the residual from below.

The constructions differ in the pose variables that remain to be searched (\Cref{tab:spatial-domains}). A known gravity direction leaves only the yaw (\Cref{fig:gravity-framework}), whereas unrestricted rotation requires cells of unit quaternions. Complete matching removes the translation by centering, whereas partial matching keeps it as a supplied quantity or as a further search variable. Throughout, we minimize the sum of squared distances over one-to-one correspondences under the pose convention $x_i\approx Ry_j+t$; for complete matching, a tolerance on the mean residual corresponds to $n$ times that tolerance on the sum. The lemmas, propositions and theorems of this section assume exact arithmetic and exact assignment oracles, and \Cref{app:spatial-algorithms} specifies the search procedures.

\begin{table*}[t]
\centering
\caption{\textbf{Scope of the spatial constructions.} ``Global'' refers to the stated objective and to a search in exact arithmetic that is run to a positive tolerance. Partial matching means either a prescribed cardinality or a penalty on unmatched points.}\label{tab:spatial-domains}
\begin{tabular}{@{}>{\raggedright\arraybackslash}p{0.22\textwidth}>{\raggedright\arraybackslash}p{0.20\textwidth}>{\raggedright\arraybackslash}p{0.23\textwidth}>{\raggedright\arraybackslash}p{0.26\textwidth}@{}}
\toprule
Input and correspondences & Pose variables & Assignment oracle & Scope of the construction \\
\midrule
Known gravity; complete bijection & Yaw; translation by centering & Square linear assignment & Global yaw, translation, and matching \\
Known gravity; partial matching & Yaw at supplied translation & Prescribed-cardinality or penalized assignment & Global conditional on that translation \\
Known gravity; partial matching & Yaw and translation & Partial assignment with a forced anchor & Global joint search over a finite anchor cover \\
Unrestricted rotation; complete bijection & Proper rotation; translation by centering & Square linear assignment & Global complete SE(3) registration \\
Unrestricted rotation; partial matching & Proper rotation at supplied translation & Prescribed-cardinality or penalized assignment & Global conditional on that translation \\
\bottomrule
\end{tabular}
\end{table*}

\subsection{Gravity Alignment and Partial Matching}\label{sec:gravity}

Suppose that each point set comes with a unit gravity direction, $g_x$ and $g_y$, given with the same sign convention. We first \emph{level} both sets to a common vertical: we choose proper rotations $A_x,A_y$ with $A_xg_x=A_yg_y=e_3$ and form $\widetilde x_i=A_xx_i$ and $\widetilde y_j=A_yy_j$. Let $R_z(\theta)$ denote the rotation by $\theta$ about $e_3$. Every rotation that satisfies $Rg_y=g_x$ has the form
\begin{equation}\label{eq:leveled-pose-recovery}
 R=A_x^\top R_z(\theta)A_y,\qquad t=A_x^\top\widetilde t,
\end{equation}
and a supplied translation $t$ becomes $\widetilde t=A_xt$. We work with the leveled quantities and drop the tildes below. Leveling does not discard the vertical coordinate: height differences remain in the matching objective.

\paragraph{Partial matching.} To allow unequal cardinalities and points without a counterpart, we use partial matchings $M\subseteq\{1,\ldots,n\}\times\{1,\ldots,m\}$ in which no row and no column is repeated. Let $\mathcal M_k$ contain the matchings with $k$ pairs, where $0\le k\le\min(n,m)$. We consider two models. The \emph{prescribed-cardinality} model fixes $k$ and minimizes the summed residual over $\mathcal M_k$. The \emph{penalized} model lets the cardinality vary and adds a penalty $\lambda>0$ for every unmatched point:
\begin{equation}\label{eq:partial-spatial}
 E(M,R,t)=\sum_{(i,j)\in M}\|x_i-Ry_j-t\|^2
       +\lambda(n+m-2|M|).
\end{equation}
Penalized matching admits every $M\in\bigcup_k\mathcal M_k$, including the empty matching; prescribed cardinality restricts $M$ to $\mathcal M_k$ and omits the penalty. We choose neither $k$ nor $\lambda$ from reference correspondences. In both models the optimal translation depends on the centroids of the selected subsets, so centering the full sets cannot eliminate it.

\paragraph{The yaw objective.} We first fix the translation $t$ and put $x_i^t=x_i-t$; \Cref{sec:joint-gravity} treats unknown translation. We call a candidate pair $(i,j)$, which matches $x_i$ to $y_j$, an \emph{edge}. Expanding the squared distance, the cost of an edge at yaw $\theta$ is
\begin{equation}\label{eq:gravity-edge}
 c_{ij}^0-2\big(a_{ij}\cos\theta+b_{ij}\sin\theta\big),
\end{equation}
where $a_{ij}=y_{jx}x_{ix}^t+y_{jy}x_{iy}^t$, $b_{ij}=y_{jx}x_{iy}^t-y_{jy}x_{ix}^t$, and $c_{ij}^0=\|y_{j,xy}\|^2+\|x_{i,xy}^t\|^2+(y_{jz}-x_{iz}^t)^2$ is a constant that contains the height difference. For a matching $M$, let $\alpha_M=\sum_Ma_{ij}$ and $\beta_M=\sum_Mb_{ij}$, and let $c_M$ be the sum of the constants plus any penalty for unmatched points. In the \emph{yaw coordinates} $u(\theta)=(\cos\theta,\sin\theta)$, the cost of $M$ is the affine function $g_M(u)=c_M-2(\alpha_Mu_1+\beta_Mu_2)$, which we extend to all $u\in\mathbb R^2$. We also define $\boldsymbol w_M=(\alpha_M,\beta_M)$ and $\kappa_\theta=2\max_M\|\boldsymbol w_M\|$, the largest slope of these affine costs. Their lower envelope over the chosen family $\mathcal M$ of matchings is
\begin{equation}\label{eq:gravity-affine-envelope}
 F(u)=\min_{M\in\mathcal M}\{c_M-2(\alpha_Mu_1+\beta_Mu_2)\}.
\end{equation}
On the circle $u(\theta)$, $F$ is the smallest physical residual over all matchings; off the circle it has no physical meaning and may be negative. As a minimum of finitely many affine functions, $F$ is concave. One assignment problem, ordinary or augmented with dummy points for the partial models, evaluates $F(u)$ and returns a matching that attains it, with the heights retained in $c_M$. Such a query uses one yaw-coordinate vector for every edge. Minimizing the edges separately over a yaw interval, as Lipschitz rotation search does point by point~\cite{li2007registration}, instead allows different edges to use different yaws.

\paragraph{The arc bound.} We enclose a yaw arc $\theta\in I=[\mu-h_\theta,\mu+h_\theta]$ of half-width $0<h_\theta<\pi/2$ in the triangle $\mathcal T_I$ whose vertices are the two endpoints $u(\mu-h_\theta)$ and $u(\mu+h_\theta)$ and the point $\sec(h_\theta)u(\mu)$ where the tangents at the endpoints intersect (\Cref{fig:gravity-framework}b). This third vertex lies outside the circle; the exterior query is necessary because the chord between the endpoints does not enclose the arc. Querying $F$ at the three vertices gives an affine interpolant $\ell$ of the three values, and the bound for the arc is the minimum of $\ell$ over the arc. Because $\ell(u(\theta))$ is a sinusoid in $\theta$, this minimum has a closed form.

\begin{lemma}[Yaw interpolation and common support]\label{lem:arc-interpolation}
Let $F=\min_{M\in\mathcal M}g_M$ be the affine extension over a nonempty finite matching family. For $0<h_\theta<\pi/2$ and $I=[\mu-h_\theta,\mu+h_\theta]$, let $\mathcal T_I$ be the triangle enclosing $\mathcal U_I=\{u(\theta):\theta\in I\}$, and let $\ell$ interpolate its three oracle values. Then
\begin{equation}\label{eq:arc-bound-chain}
 \min_{v\in\operatorname{vert}\mathcal T_I}F(v)
 \le\min_{u\in\mathcal U_I}\ell(u)
 \le\min_{u\in\mathcal U_I}F(u).
\end{equation}
If one matching attains all three oracle values, then $F=\ell$ on $\mathcal T_I$,
including when other matchings tie.
\end{lemma}
The second inequality holds because $F$ is concave and therefore lies above its interpolant throughout the triangle. \Cref{app:arc-interpolation} gives the interpolation formula and proves the first inequality and the statement on common support.

\paragraph{The yaw search.} \Cref{alg:gravity-yaw} applies the arc bound in a best-first search that starts from three arcs of width $2\pi/3$, which cover the circle, and bisects arcs as needed. Every matching $M$ returned by a query also gives a feasible update: its cost at yaw $\theta$ is the sinusoid $g_M(u(\theta))$, and the operation $\textsc{YawMin}(g_M,I)$ minimizes it over the arc $I$ by testing the endpoints and any stationary minimum in $I$; for a constant sinusoid, either endpoint suffices. We also fit the matchings returned by exterior queries in this way, because the value of $F$ at an exterior vertex is not a physical residual.

The algorithm assumes nonempty point sets and, for prescribed cardinality, $k>0$. The remaining cases need no search: for $k=0$, and for empty point sets under complete matching, the value is zero, and if either set is empty, penalized matching returns the empty matching with value $\lambda(n+m)$. Ties in the assignment, in the sinusoid minimization and in the queue may be resolved arbitrarily. Every arc starts with the trivial lower bound $L_0$, which is zero for complete or prescribed-cardinality matching and $\lambda|n-m|$ for penalized matching.

\begin{algorithm*}[t]
\caption{Gravity yaw by three assignment queries per arc}\label{alg:gravity-yaw}
\begin{algorithmic}[1]
\REQUIRE Leveled point sets; complete matching ($n=m$), or a partial model and supplied $t$; tolerance $\varepsilon>0$; optional budget.
\ENSURE Feasible $(M_*,R_z(\theta_*),t_*)$ and residual interval $[L,U]$.
\STATE For complete matching, save the means and center both point sets; set working $t=0$.
\STATE Form the coefficients in \eqref{eq:gravity-edge} and the assignment oracle $F$.
\STATE Query $F(u(0))$; store its matching at $\theta_*=0$ and its physical cost $U$.
\STATE Put three covering arcs in a min-heap, each with lower bound $L_0$ and marked unevaluated.
\WHILE{the heap is nonempty}
\STATE $L\leftarrow\min\{U,\min_{I\text{ in heap}}L_I\}$.
\STATE Stop if $U-L\le\varepsilon$ or the budget is exhausted.
\STATE Select a least-bound arc $I=[\mu-h_\theta,\mu+h_\theta]$; keep it in the cover during evaluation.
\IF{$I$ is already evaluated}
\STATE Replace $I$ atomically by its two half-arcs, unevaluated and inheriting $L_I$; continue.
\ENDIF
\STATE $u_-\leftarrow u(\mu-h_\theta)$, $u_+\leftarrow u(\mu+h_\theta)$, $u_A\leftarrow\sec(h_\theta)u(\mu)$.
\FOR{$v\in\{u_-,u_+,u_A\}$}
\STATE Query $(f_v,M_v)\leftarrow F(v)$.
\STATE $(\theta_v,U_v)\leftarrow\textsc{YawMin}(g_{M_v},I)$.
\STATE If $U_v<U$, store this feasible matching and yaw, and set $U\leftarrow U_v$.
\ENDFOR
\STATE Let $\ell$ interpolate $(v,f_v)$ at the three vertices.
\STATE $L_I\leftarrow\max\{L_I,\min_{\theta\in I}\ell(u(\theta))\}$.
\STATE Remove $I$ if $L_I\ge U$; otherwise update its heap key and mark it evaluated.
\ENDWHILE
\STATE Recompute $L=\min\{U,\min_{I\text{ in heap}}L_I\}$, taking the empty-heap minimum as $+\infty$.
\STATE Restore $t_*=m_x-R_z(\theta_*)m_y$ for complete matching; retain the supplied $t$ for partial matching.
\STATE Return the stored pose and matching with $[L,U]$.
\end{algorithmic}
\end{algorithm*}

The search can be interrupted at any time: an interrupted evaluation leaves the previous bound of its arc in place, and the feasible updates that were completed are kept (\Cref{thm:spatial-search}). Undoing the leveling and restoring the centroids recovers the pose in the original frame. For partial matching the result remains conditional on the supplied translation, also when the search is alternated with local updates of the translation. The number of queries depends on the requested gap as follows.

\begin{proposition}[Conditional-yaw accuracy]\label{prop:gravity-accuracy}
For complete matching after centering, or partial matching at a supplied
translation, put $\boldsymbol w_M=(\alpha_M,\beta_M)$ and $\kappa_\theta=2\max_M\|\boldsymbol w_M\|$.
Without an external budget, \Cref{alg:gravity-yaw}, initialized by three equal
arcs, reaches summed-objective gap $\varepsilon>0$ with
$\mathcal O(1+\sqrt{\kappa_\theta/\varepsilon})$ assignment queries.
\end{proposition}
\begin{proof}[Proof sketch]
Moving the tangent apex to the midpoint of the arc on the circle changes the cost of a matching by at most $\kappa_\theta(\sec h_\theta-1)\le\kappa_\theta h_\theta^2$ for $h_\theta\le\pi/3$. Bisection therefore reaches the required local gap within $\mathcal O(1+\sqrt{\kappa_\theta/\varepsilon})$ cells. If $\kappa_\theta=0$, the objective is constant in the yaw. \Cref{app:gravity-accuracy} proves the gap and the query count.
\end{proof}

Yaw coordinates and angles in radians are dimensionless, so $\kappa_\theta$ and the tolerance $\varepsilon$ on the summed gap both have units of squared distance. The local gap of the arc bound is quadratic in the half-width of the arc. The next proposition contrasts this with bounds that treat the edges independently.

\begin{proposition}[Second-order coupled bounds]\label{prop:bound-order}
Fix the translation, or eliminate it by centering. Let $A$ be a yaw arc of
half-width $h_\theta\in(0,\pi/3]$, let $F^*_A=\min_{\theta\in A}F(u(\theta))$, and let $L_A$ denote the lower bound on $A$ under consideration.
\begin{enumerate}[label=(\roman*),leftmargin=*,nosep]
\item The Rubix interpolant bound satisfies $F^*_A-L_A\le \kappa_\theta h_\theta^2$, with $\kappa_\theta$ as in
\Cref{prop:gravity-accuracy}, and $F^*_A=L_A$ whenever one matching attains all
three triangle values.
\item The per-correspondence radius bound of Yang et al.~\cite{yang2016goicp}, specialized to yaw,
\[
\min_M\sum_{(i,j)\in M}(\|R_\mu x_i-y_j\|-2\sin(h_\theta/2)\|x_{i,xy}\|)_+^2
\]
and the independent-edge bound
\[
\min_M\sum_{(i,j)\in M}\min_{\theta\in A}\|R_\theta x_i-y_j\|^2
\]
are first order: there are instances with $F^*_A-L_A\ge c h_\theta$ on the arcs
centered at the minimizer, for a constant $c>0$ independent of $h_\theta$.
\item On the radius-bound instance of~(ii), best-first bisection using that bound,
a zero floor, and inherited maxima expands $\Omega(\varepsilon^{-1/2})$ arcs to
certify an absolute gap $\varepsilon<1$, while the Rubix bound certifies the
optimum after evaluating the three root arcs. This comparison excludes
additional global bounds.
\end{enumerate}
\end{proposition}
\begin{proof}[Proof sketch]
(i) is the local gap in the proof of \Cref{prop:gravity-accuracy}; a common
supporting matching makes the interpolant exact by \Cref{lem:arc-interpolation}.
For~(ii), note that the two displayed comparison bounds rotate the source set. One source point $(1,0,0)$ and one target $(2,0,0)$ give
$F=5-4\cos\theta$ and a radius-bound gap of $4\sin(h_\theta/2)-4\sin^2(h_\theta/2)\ge2h_\theta/\pi$; two
points whose edges prefer the yaws $\pm a$, with $\pi/3\le a<\pi/2$, give an independent-edge gap of
$8\sin(a-h_\theta/2)\sin(h_\theta/2)\ge(8/\pi)\sin(a/2)\,h_\theta$. For~(iii), we credit each inherited bound to the ancestor
at which it was evaluated and contract its descendant subtree there. For sufficiently small $\varepsilon$, every terminal
descendant of the central root that meets
$[-\sqrt\varepsilon,\sqrt\varepsilon]$ has half-width below $100\varepsilon$.
Covering this interval therefore requires at least $1/(100\sqrt\varepsilon)$ leaves
and $1/(200\sqrt\varepsilon)$ expansions. The objective for a single correspondence is
affine in the lifted yaw coordinates, so the Rubix bound is exact on all three
root arcs. The full proof and the first-order constructions are in \Cref{app:bound-order}.
\end{proof}
The comparison illustrates the cluster problem studied by Du and Kearfott~\cite{du1994cluster}: near a minimizer with positive
curvature, bounds of order below two can leave a number of unresolved cells that
grows without bound as the tolerance shrinks, whereas bounds of order two leave a
bounded number. Sch\"obel and Scholz~\cite{schobel2010rate} define this rate of
convergence for geometric branch-and-bound.

\begin{remark}[Unbalanced transport]\label{rem:unbalanced}
The penalty on unmatched points has a transport interpretation. Give every point unit mass. At a fixed pose, unbalanced transport with a total-variation penalty on the marginals minimizes $\sum_{ij}P_{ij}c_{ij}+\lambda(\|\mathbf 1-P\mathbf 1\|_1+\|\mathbf 1-P^\top\mathbf 1\|_1)$ over nonnegative $P$, where $c_{ij}=\|x_i-Ry_j-t\|^2$. Reducing the excess mass in a row decreases the penalty of that row by at least as much as it increases the penalties of the columns, and it does not increase the transport cost. The same argument applies to columns, so some optimal $P$ satisfies $P\mathbf 1\le\mathbf 1$ and $P^\top\mathbf 1\le\mathbf 1$. The objective is then $\sum_{ij}P_{ij}(c_{ij}-2\lambda)+\lambda(n+m)$, which is linear on the bipartite matching polytope, whose vertices are the partial matchings. The penalized objective~\eqref{eq:partial-spatial} is therefore rigid unbalanced transport with this penalty. With zero heights, \Cref{alg:gravity-yaw} computes it over planar rotations at a supplied translation and reaches a gap of $\varepsilon$ with $\mathcal O(1+\sqrt{\kappa_\theta/\varepsilon})$ assignment queries, because its arc bounds are second order (\Cref{prop:gravity-accuracy,prop:bound-order}). Unbalanced GW~\cite{sejourne2021} instead compares the internal distances of the two sets under marginals relaxed by a Kullback--Leibler penalty and is solved by local alternation; \Cref{sec:exp-unbalanced} compares the two.
\end{remark}

\subsection{Joint Gravity Alignment with Unmatched Points}\label{sec:joint-gravity}

When the translation of a partial matching is unknown, it must be searched together with the yaw, and the search needs a bounded translation domain before the matched subsets are known. \emph{Anchors} provide one. Suppose that a feasible solution of cost $U$ is known and that the matching has cardinality $k>0$. Every matching that improves on $U$ then has an edge whose residual has norm at most $\sqrt{U/k}$. For each candidate edge $(i,j)$, the anchor, we translate the observations by $x_i$ and $y_j$, force that edge into the matching, and write $t=x_i-R_z(\theta)y_j+e$, so that $-e$ is the residual of the anchor. This gives $nm$ forced-anchor domains $e\in[-\sqrt{U/k},\sqrt{U/k}]^3$, one for each anchor, which may overlap.

A cell of this search is the product of a yaw arc and a translation box. The cost $E_M(\theta,e)$ of a matching is no longer affine, because it is quadratic in the translation; but at a fixed cardinality every matching has the same quadratic term. For a box centered at $c_B$, we therefore subtract $k\|e-c_B\|^2$ from each cost. The resulting function $h_M$ is affine in the yaw coordinates and affine in each translation coordinate, so it equals its interpolant from the vertices of the product of the yaw triangle and the translation box. This product has $3\times8=24$ vertices $v=(u,e)\in\mathbb R^2\times\mathbb R^3$. We solve one assignment problem at each of them and retain a feasible matching $W$ as a \emph{witness}, whose slack is its largest excess over the minimum at the vertices,
\begin{equation}\label{eq:gravity-joint-slack}
 \Delta_W=\max_v\big\{h_W(v)-\min_M h_M(v)\big\}.
\end{equation}

\begin{lemma}[Joint-gravity cover and witness bound]\label{lem:joint-cover}
For prescribed cardinality $k>0$ and a feasible cost $U$, the $nm$ forced-anchor
domains above, with radius $\sqrt{U/k}$, cover every
improving pose and matching. For penalties, the same holds after partitioning
by $k$, subtracting $p_k=\lambda(n+m-2k)$ from $U$, retaining all branches
with $p_k\le U$, and testing the empty matching separately. Within one branch,
let $W$ be any feasible matching of size $k$. On every yaw-triangle/translation-box
cell, the slack in \eqref{eq:gravity-joint-slack} satisfies
\begin{equation}\label{eq:joint-witness-guarantee}
 \min_M E_M(\theta,e)\ge E_W(\theta,e)-\Delta_W.
\end{equation}
If $W$ supports all 24 product vertices, equality holds throughout the cell.
\end{lemma}
The witness bound holds because the nonnegative interpolation weights extend the excess at the vertices throughout the cell, after which the common quadratic term is restored; \Cref{app:joint-cover} proves both parts of the lemma. Minimizing the right-hand side of~\eqref{eq:joint-witness-guarantee} over the cell gives a lower bound for the cell and, at the same time, a feasible pose for the witness; when $\Delta_W=0$, this fit solves the cell. At a fixed yaw, the optimal $e$ is the difference between the centroids of the matched points after rotation, clipped to the translation box. Between changes of the clipping pattern, the cost is a trigonometric polynomial of degree two in the yaw, whose stationary points are the roots of a quartic. \Cref{app:joint-witness-fit} derives this fit, which \Cref{alg:joint-gravity-search} uses, including clipping, degenerate quartics and ties, and \Cref{app:geometry-identities} reduces the cost of its construction.

\paragraph{A common-translation bound.} A simpler bound needs no witness. It evaluates the yaw interpolant of \Cref{sec:gravity} at each corner of the translation box and subtracts $k$ times the squared half-diagonal of the box, which again keeps one translation common to all edges. It does not require one-to-one matchings: it also holds when a fixed number of nearest-neighbor pairs is selected and targets may repeat, which is the trimmed objective of GLORES and allows the bound to be inserted there (\Cref{sec:native-planar}). To match the convention of that solver, the next two statements rotate $X$ toward $Y$.
For $u=(u_1,u_2)$, define $\widehat R(u)$ to have horizontal block
$\left(\begin{smallmatrix}u_1&-u_2\\u_2&u_1\end{smallmatrix}\right)$ and,
in three dimensions, to leave height unchanged.

\begin{lemma}[Common translation]\label{lem:common-translation}
Use the pose convention $R_z(\theta)x_i+t-y_j$, reversing the point-set names in \eqref{eq:pw}.
Let $\mathcal A_{\rm sel}$ be a nonempty finite, pose-independent family of selections of
exactly $k\ge1$ pairs from planar or three-dimensional point sets. Define
\begin{equation}\label{eq:corrected-gravity-objective}
\begin{split}
 F(u,t)=\min_{A_{\rm sel}\in\mathcal A_{\rm sel}}\sum_{(i,j)\in A_{\rm sel}}\big[&
 \|\widehat R(u)x_i+t-y_j\|^2\\
 &+(1-\|u\|^2)\|x_{i,xy}\|^2\big].
\end{split}
\end{equation}
For planar inputs, $x_{i,xy}=x_i$. Let $B$ be an axis-aligned translation box
with center $c_B$, nonnegative halfwidths $\boldsymbol h_t$, and corners $t_\nu$, retaining duplicate
corners when necessary. Let a nondegenerate triangle enclose the physical yaw arc $\mathcal U_I$
for an interval $I$, and let $\ell_\nu$ interpolate $F(\cdot,t_\nu)$ at its three vertices.
Then
\begin{equation}\label{eq:common-translation-guarantee}
 \min_\nu\min_{u\in\mathcal U_I}\ell_\nu(u)-k\|\boldsymbol h_t\|^2
 \le\min_{u\in\mathcal U_I,\,t\in B}F(u,t).
\end{equation}
Selections may repeat targets, as in fixed-count trimmed nearest neighbors.
\end{lemma}
\begin{proof}[Proof sketch]
The correction term for points off the circle keeps the cost of each selection affine in $u$. Every selection has $k$ pairs and hence the same quadratic term in the translation, so $F_{\rm red}=F-k\|t-c_B\|^2$ is concave in $u$ and concave in $t$ separately. Interpolation over the box and over the triangle then gives the bound, since $\|t_\nu-c_B\|^2=\|\boldsymbol h_t\|^2$ at every corner of the box and $F\ge F_{\rm red}$ (\Cref{app:common-translation-proof}).
\end{proof}

On a rational grid, the accuracy of this bound can be quantified.

\begin{proposition}[Fixed-cardinality gravity approximation]\label{prop:joint-gravity-accuracy}
Let $\mathcal E$ be a finite set of pairs of rational three-dimensional points and let
$\mathcal A_{\rm sel}$ be a nonempty, globally fixed family of selections of exactly
$k\ge1$ pairs from $\mathcal E$. Assume an exact polynomial-time rational oracle for
affine minimization over $\mathcal A_{\rm sel}$. For unit scale, yaw, and translation in
a rational box $B$ of side lengths $\ell_j$, choose a rational bound
\begin{equation}\label{eq:gravity-grid-slope}
 \kappa_B\ge2k\max_{(i,j)\in \mathcal E}\sup_{t\in B}
 \|x_{i,xy}\|\,\|t_{xy}-y_{j,xy}\|.
\end{equation}
For rational $\varepsilon>0$, a feasible pose and an enclosing summed-objective
interval of width at most $\varepsilon$ can be obtained with at most
$128N_\theta\prod_{j=1}^3N_{t,j}$ oracle calls, where $N_\theta=\max\{1,\lceil2\sqrt{\kappa_B/\varepsilon}\rceil\}$ and
\begin{equation}\label{eq:gravity-grid-counts}
 N_{t,j}=\max\{1,\lceil\sqrt{3k\ell_j^2/(2\varepsilon)}\rceil\}.
\end{equation}
The count is $\mathcal O(\varepsilon^{-2})$ at fixed numerical scales.
\end{proposition}
\begin{proof}[Proof sketch]
A rational parameterization of the yaw gives rational arc endpoints, tangent intersections and nearby physical samples. Combining the Lipschitz error of these samples with the bound on the translation term gives $U-L\le k\|\boldsymbol h_t\|^2+2\kappa_B/N_\theta^2\le\varepsilon$. \Cref{app:joint-gravity-accuracy} constructs the grid and proves the count.
\end{proof}

A rational $\kappa_B$ is obtained from horizontal $\ell_1$ bounds. An unrestricted optimum
lies in $[-r_t,r_t]^3$, where $r_t=\max_i\|x_i\|_1+\max_j\|y_j\|_1$, by the
centroid formula for the selected pairs. This is a guarantee of additive accuracy: its count depends on the numerical scale and on the inverse
tolerance, not only on their bit lengths. It does not apply to pose-dependent
screening, to variable cardinality, or to the number of queries made by the adaptive search of GLORES.
For mean residuals, we use the summed tolerance $k\varepsilon_{\rm mean}$.

\subsection{From the Polygon to a Spatial Polytope}\label{sec:se3}

We now drop the gravity constraint and consider complete matching of equally weighted sets in $\mathbb R^3$, for which centering removes the translation. We write $a_i=x_i^c=x_i-m_x$, $b_j=y_j^c=y_j-m_y$, and $S=\sum_i\|a_i\|^2+\sum_j\|b_j\|^2$. Each planar correlation is replaced by a matrix, and the polygon by a polytope:
\begin{equation}\label{eq:covariance-polytope}
 C_\sigma=\sum_i a_i b_{\sigma(i)}^\top,\qquad
 \mathcal C=\conv\{C_\sigma:\sigma\in S_n\}.
\end{equation}
As in the plane, one assignment problem with scores $a_i^\top D b_j$ maximizes $\langle D,C\rangle_F$ over $\mathcal C$ for any matrix $D$, and returns a permutation, its correlation matrix and the support value.

For a fixed matching, the best rotation is found by Horn's quaternion method~\cite{horn1987}. A unit quaternion $q\in\mathbb R^4$, written with its scalar coordinate first, represents a proper rotation $R(q)$. Let $\mathbb S^4$ denote the symmetric $4\times4$ matrices. Horn's linear map $K:\mathbb R^{3\times3}\to\mathbb S^4$ gives a symmetric matrix $K(C)$ of trace zero that satisfies
\begin{equation}\label{eq:quaternion-score}
 \langle R(q),C\rangle_F=q^\top K(C)q,
 \qquad \|q\|=1.
\end{equation}
Antipodal quaternions describe the same rotation, so the parameter space is $\mathbb{RP}^3$. The best rotation score of a fixed matching is thus the largest eigenvalue of $K(C_\sigma)$. Combining this quaternion lift and convex representations of the rotations~\cite{sanyal2011orbitopes,saunderson2015rotations} with the polytope gives
\begin{equation}\label{eq:se3-polytope}
 E^*_{\mathrm{SE}(3)}=S-2\max_{C\in\mathcal C}\lambda_{\max}(K(C)).
\end{equation}

\begin{theorem}[Spatial covariance reduction]\label{thm:spatial-covariance}
For equally weighted complete point sets of the same positive size in
$\mathbb R^3$, the optimal proper rigid residual is \eqref{eq:se3-polytope}.
Some vertex of $\mathcal C$ attains its maximum, and any unit top eigenvector
of its quaternion matrix gives an optimal rotation; translation is
$m_x-Rm_y$. If the centered point sets have ranks $r_X,r_Y$, then
$\dim\mathcal C=r_Xr_Y\le9$. Repeated points, rank deficiency, and eigenvalue
ties are allowed. Rational inputs admit a finite exact polytope reduction,
without a claimed polynomial vertex or facet count.
\end{theorem}
A vertex attains the maximum because the largest eigenvalue is convex, just as the planar optimum is attained at a farthest vertex of the polygon, and the reduction yields proper rotations, in agreement with the determinant correction of the Procrustes solution. \Cref{app:covariance} proves the theorem, together with the rank formula and the finite exact construction.

The exact reduction is finite but comes without a polynomial bound. For every fixed dimension $d\ge2$, and in particular for $d=3$, we conjecture that the number of vertices of the matrix-correlation polytope grows superpolynomially in the worst case, with no uniform polynomial bound in $n$ (\Cref{app:matrix-conjecture}), and no reconstruction guarantee like the $2V$ queries of \Cref{prop:es} is known. We therefore use the reduction only to fit the matchings that are returned during a search over quaternion regions, in which each assignment query again shares one direction across all edges.

\subsection{Quaternion Region Bounds}\label{sec:quaternion-bound}

\begin{figure}[t]
\centering
\includegraphics[width=\linewidth]{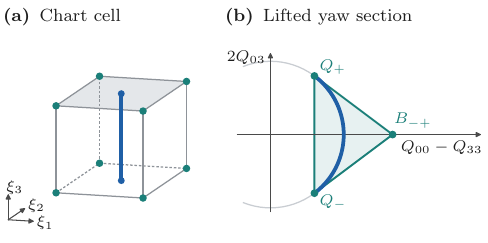}
\caption{\textbf{Bounding rotations through quaternion lifts.} (a)~An illustrative chart cell $\tilde q=(1,\xi_1,\xi_2,\xi_3)$ with $\xi_1,\xi_2,\xi_3\in[-\tfrac12,\tfrac12]$. Its eight raw corners (teal) define the cell, and its central segment $\xi_1=\xi_2=0$ (blue) contains the pure yaw rotations. (b)~On this yaw section, in the coordinates $(Q_{00}-Q_{33},2Q_{03})$ of $Q=qq^\top$ and not of the raw quaternion, the endpoint lifts $Q_\pm$ and the exterior cross lift $B_{-+}$ enclose the arc of physical rotations (blue). Assignment values at these query matrices bound the correlation score throughout the region. The full three-dimensional construction uses at most 36 queries per box before reuse; the illustrated section uses three.}\label{fig:quaternion-lift}
\end{figure}

The score~\eqref{eq:quaternion-score} is quadratic in $q$ but linear in the matrix $Q=qq^\top\in\mathbb S^4$, which we call the \emph{lift} of $q$. This linearity lets assignment queries bound whole regions of rotations. We cover the rotations by four charts. In each chart, we choose a coordinate of the quaternion that has the largest magnitude, make it positive, and scale it to one, so that the other three coordinates lie in $[-1,1]^3$. Splitting each of them at zero gives eight boxes per chart, hence 32 boxes that cover every proper rotation. The boxes overlap on their boundaries, which is harmless. Within each box, any two corners have a dot product of at least one.

Let $\tilde q$ range over one such box, or over a box obtained from it by subdivision, with raw corners $\tilde q_1,\ldots,\tilde q_{N_{\rm raw}}$, where $N_{\rm raw}=8$. The raw vector $\tilde q$ is not normalized; its unit representative is $q=\tilde q/\|\tilde q\|$. From the corners we form the 36 symmetric matrices
\begin{equation}\label{eq:outer-quaternion}
 B_{ab}=\frac{\tilde q_a\tilde q_b^\top+\tilde q_b\tilde q_a^\top}{2\tilde q_a^\top \tilde q_b},\qquad a\le b.
\end{equation}
The matrices $B_{aa}$ are the lifts of the normalized corners. The cross terms $B_{ab}$ with $a<b$ play the role of the exterior vertex of the yaw triangle: all $B_{ab}$ have trace one, but the cross terms need not be positive semidefinite or represent physical poses.

To bound all matchings over a box, we define the \emph{assignment support} of a matrix $Q\in\mathbb S^4$,
\begin{equation}\label{eq:quaternion-support}
 H(Q)=\max_{\sigma\in S_n}\langle K(C_\sigma),Q\rangle_F.
\end{equation}
To evaluate $H$, we use the adjoint linear map $K^*:\mathbb S^4\to\mathbb R^{3\times3}$ and write $D(Q)=K^*(Q)$, which is defined by $\langle D(Q),C\rangle_F=\langle Q,K(C)\rangle_F$. Then
\begin{equation}\label{eq:lifted-assignment-support}
 H(Q)=\max_{\sigma\in S_n}\sum_i a_i^\top D(Q)b_{\sigma(i)}.
\end{equation}
Thus one assignment problem evaluates $H(B_{ab})$ and returns a matching that attains this support value. For the lift of a unit quaternion, $D(qq^\top)=R(q)$ is the rotation itself; for the exterior matrices, $D(B_{ab})$ need not be a rotation. A witness matching can tighten the resulting bound, as in the joint gravity search. For any matching $\tau$, we define its largest deficit at the queries,
\begin{equation}\label{eq:face-slack}
 \delta_\tau=\max_{a\le b}
 \big\{H(B_{ab})-\langle K(C_\tau),B_{ab}\rangle_F\big\}.
\end{equation}

\begin{lemma}[Quaternion enclosure and witness bound]\label{lem:quaternion-enclosure}
Let $\tilde q$ lie in the convex hull of nonzero raw quaternion vertices $\tilde q_a$ satisfying
$\tilde q_a^\top \tilde q_b>0$ for every pair, and put $q=\tilde q/\|\tilde q\|$. The matrices in
\eqref{eq:outer-quaternion} enclose $qq^\top$, that is, $qq^\top\in\conv\{B_{ab}:a\le b\}$. Consequently the support
$H$ in \eqref{eq:quaternion-support} satisfies
\begin{equation}\label{eq:face-bound}
 H(qq^\top)\le\min\{\max_{a\le b}H(B_{ab}),\lambda_{\max}(K(C_\tau))+\delta_\tau\}
\end{equation}
for every feasible permutation $\tau$, with $\delta_\tau$ as in
\eqref{eq:face-slack}. A common supporting permutation has $\delta_\tau=0$,
regardless of ties.
\end{lemma}
\begin{proof}[Proof sketch]
Write $\tilde q=\sum_a\alpha_a\tilde q_a$ with convex weights. Expanding the normalized lift gives a convex combination of the corner and cross-corner lifts, because their dot products are positive. Convexity of $H$, which is a maximum of linear functions, gives the first bound. Subtracting the linear score of the witness before interpolating, and then applying the Rayleigh inequality, gives the second. \Cref{app:quaternion-lifts,app:quaternion-cover} prove the enclosure, the coverage by charts and the convergence under refinement.
\end{proof}

The enclosure is an instance of the convex-hull property of rational B\'ezier representations with positive weights~\cite{piegl1997nurbs}, and restricted to pure yaw rotations it recovers the triangle of \Cref{sec:gravity} (\Cref{fig:quaternion-lift,app:quaternion-yaw-section}). If $\tau$ attains the support value at every query, the bound equals the score of the feasible fit of $\tau$, so the cell is closed even when its residual is not zero. We minimize the witness bound over the available witnesses and combine it with two cheaper \emph{rotation-cap bounds}, which depend only on the largest angular distance of the cell from its center (\Cref{app:rotation-cap}).

For a cell $\Omega$, we denote by $b_\Omega$ the smallest available upper bound on the score: the largest exterior support value, the witness bounds, and the rotation-cap bounds. Since the residual decreases as the correlation score increases, the lower bound on the residual of the cell is
\begin{equation}\label{eq:spatial-residual-bound}
 L_\Omega=\max\{0,S-2b_\Omega\}.
\end{equation}
Each returned permutation $\sigma$ gives the feasible residual $S-2\lambda_{\max}(K(C_\sigma))$ after its rotation is fitted; the fitted pose may lie outside the cell and can still improve the incumbent. \Cref{alg:se3-search} uses at most 36 exterior queries per box before reuse, bisects the longest raw coordinate, and inherits bounds through \Cref{alg:spatial-search}. In our experiments, the search with the rotation-cap bounds alone is the baseline, and Rubix adds the exterior queries and witness bounds to it; both share the search and its feasible updates. The stronger of the two rotation-cap bounds already couples the edges, because it bounds the gain of any matching over a retained witness.

At a fixed translation, the construction also permits prescribed-cardinality and penalized matching. We subtract $t$ from the fixed set, replace $C_\sigma$ by the correlation of the selected pairs, and absorb the constant of each matching, which now depends on the selected pairs, into a scalar multiple of the identity in $K$ (\Cref{app:partial-translation}). For unknown translation, an optimum is still attained at a vertex of a finite hull, now of matching moments.

\begin{proposition}[Finite moment reduction for partial rigid alignment]\label{prop:partial-moments}
For finite point sets and either prescribed cardinality
$0\le k\le\min(n,m)$ or a positive unmatched penalty, freely translated partial
proper-rigid alignment attains an optimum at
a vertex of the convex hull of matching statistics
\begin{equation}\label{eq:partial-moment-statistics}
 \Bigl(|M|,\sum_Mx_i,\sum_My_j,\sum_M(\|x_i\|^2+\|y_j\|^2),\sum_Mx_i y_j^\top\Bigr).
\end{equation}
The vector coordinates $\boldsymbol s_{X,M}=\sum_Mx_i$ and $\boldsymbol s_{Y,M}=\sum_My_j$ are coordinate sums, not centroids; write $S_M=\sum_M(\|x_i\|^2+\|y_j\|^2)$ and $C_M=\sum_Mx_i y_j^\top$ for the scalar norm sum and matrix correlation. This hull has dimension at most 17 and a linear support oracle given by partial
assignment. A nonempty vertex matching is evaluated by subset centering and
proper Procrustes alignment~\cite{horn1987}; the empty matching is evaluated separately.
\end{proposition}
\begin{proof}[Proof sketch]
At fixed $R,t$, the objective is affine in these statistics. Its infimum over poses is concave and attains its minimum at a vertex of the hull; subset centering gives each nonempty matching a pose optimum that is attained (\Cref{app:partial-moments}).
\end{proof}
The complexity of reconstructing this hull is open, and we do not use it algorithmically: our partial quaternion search fixes the translation. In floating point, the assignments and eigenvalues of all these searches give numerical gaps, which are not the exact certificates of \Cref{prop:optimality}.

\subsection{Reusing Assignment Certificates}\label{sec:bound-cost}

We apply a cheap valid bound first and reserve the exterior queries for the cells that remain unresolved. Tightening a cell updates its priority in the queue without subdividing it, and neighboring cells reuse geometry, witnesses and support values. Two devices reduce the number of assignment problems.

The first is a \emph{matching-difference bound}, which replaces the separate assignments at the query matrices by one assignment. For the edge scores $A_{ij}(Q)=\langle K(a_i b_j^\top),Q\rangle_F$ and a known permutation $\tau$, we take for each edge its largest advantage over the edge that $\tau$ assigns to the same row,
\begin{equation}\label{eq:lifted-difference}
 d_{ij}^\tau=\max_{a\le b}\{A_{ij}(B_{ab})-A_{i,\tau(i)}(B_{ab})\},
\end{equation}
and form $\eta_\tau=\max_\sigma\sum_i d_{i,\sigma(i)}^\tau$ by one assignment. The second device replaces a support query by a check against dual prices. For any column prices $g=(g_j)_{j=1}^n\in\mathbb R^n$, we set
\begin{equation}\label{eq:reused-dual}
\begin{aligned}
 f_i&=\max_j\{A_{ij}(Q)-g_j\},\\
 H_{\rm dual}(Q;g)&=\sum_i f_i+\sum_jg_j,
\end{aligned}
\end{equation}
so that $f_i+g_j\ge A_{ij}(Q)$. Prices obtained at one query can thus be reused at another, where only the row maxima must be recomputed.

\begin{lemma}[Cheaper bounds from matching differences and duals]\label{lem:spatial-bound-reuse}
For a feasible square-assignment witness $\tau$, the quantities in
\eqref{eq:lifted-difference} satisfy $\eta_\tau\ge0$ and
\begin{equation}\label{eq:difference-bound-guarantee}
 H(Q)\le\langle K(C_\tau),Q\rangle_F+\eta_\tau\quad\text{for }Q\in\conv\{B_{ab}\}.
\end{equation}
For $N_{\rm query}$ vertices, construction and one assignment cost
$\mathcal O(N_{\rm query}n^2+n^3)$ operations and $\mathcal O(n^2)$ streamed storage.
Any prices in \eqref{eq:reused-dual} also give a valid support upper bound, $H(Q)\le H_{\rm dual}(Q;g)$;
a feasible matching attaining it is optimal. Row-complete rectangular matching
requires nonnegative column prices, and equality requires zero prices on unused
columns.
\end{lemma}
The difference bound follows from the linearity of the scores and the convex enclosure, and the dual bound from weak duality (\Cref{app:spatial-bound-reuse}). By the first, \eqref{eq:face-bound} holds with $\eta_\tau$ in place of $\delta_\tau$, at a cost that compares with $\mathcal O(N_{\rm query}n^3)$ for $N_{\rm query}$ assignment queries. The difference bound can be weaker than the bound from the full set of queries, but when $\eta_\tau=0$ it closes a region on which the witness remains optimal. Substituting dual upper bounds into \eqref{eq:face-slack} likewise preserves \eqref{eq:face-bound}. A dual bound that is attained by a feasible matching recovers the exact support value without another assignment solve; otherwise it remains a valid but conservative upper bound that later queries can tighten. The remaining budget and the measured costs can guide which bound to evaluate. The joint gravity search admits an analogous bound from minimum differences, after the common quadratic term in the translation is subtracted (\Cref{app:gravity-difference}).

\subsection{Guarantees of the Search}\label{sec:spatial-guarantees}

The regional bounds yield a global interval only while the queue covers every solution that can improve the incumbent. The following theorem states this invariant, which also covers initial cells that are still pending and evaluations that are interrupted, and the refinement conditions needed for termination. If a budget is exhausted first, the remaining gap may exceed $\varepsilon$.

\begin{theorem}[Anytime interval and positive-tolerance termination]\label{thm:spatial-search}
Suppose the root domains cover every solution improving the initial feasible
incumbent, and every root is retained until validly pruned. Give each live region
$\Omega$ a valid residual lower bound
$L_\Omega$, replace it only by a complete finite child cover, and discard it
only when $L_\Omega\ge U$. Then every interruption preserves
\begin{equation}\label{eq:anytime-interval}
 \min\{U,\min_{\Omega\text{ live}}L_\Omega\}\le E^*\le U,
\end{equation}
with the empty minimum interpreted as $+\infty$. Without an external budget,
the specified conditional-yaw, joint-gravity, and quaternion searches terminate
at every $\varepsilon>0$, provided tightening eventually yields to subdivision
and every active coordinate shrinks along refined paths.
\end{theorem}
\begin{proof}[Proof sketch]
Feasible updates only decrease $U$, so pruned regions remain excluded; complete child covers and inherited bounds retain every other candidate. On the finite compact root cover, the bounds converge to the physical values as the cells shrink. An infinite refinement would contain an infinite path along which the gap tends to zero, which contradicts the stopping rule for a least-bound cell. \Cref{app:search-invariants} proves this argument, using the local convergence results of \Cref{app:gravity-accuracy,app:joint-convergence,app:quaternion-cover}; \Cref{app:spatial-queue} handles interruptions.
\end{proof}

%% file: sections/gravity_framework.tex
\begin{figure*}[t]
\centering
\includegraphics[width=\textwidth]{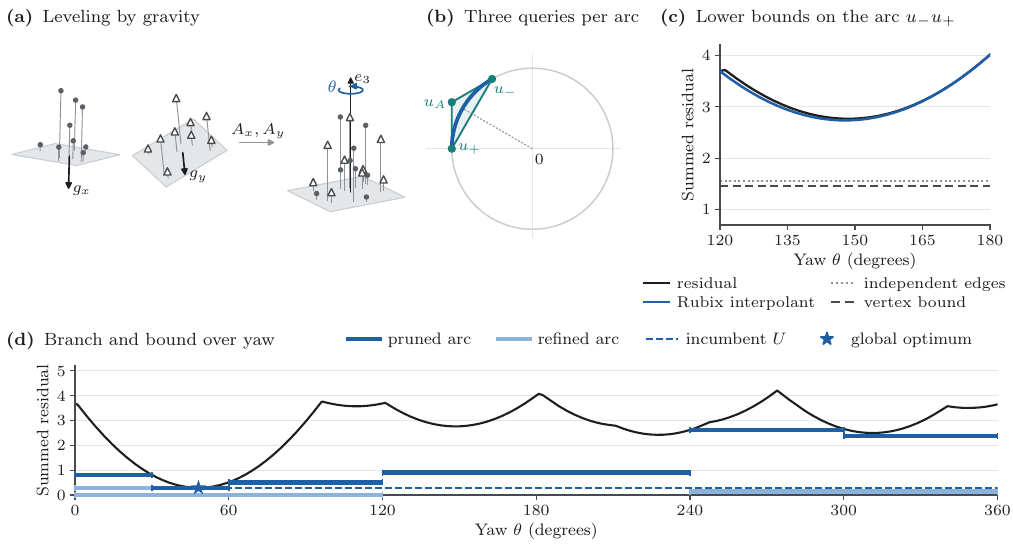}
\caption{\textbf{Gravity reduces 3D alignment to a bounded yaw search.} A nine-point complete-matching example. (a)~Gravity directions $g_x,g_y$ give leveling rotations $A_x,A_y$, leaving yaw $\theta$ about $e_3$; heights (stems) remain in every edge cost. (b)~Rubix bounds the blue arc from $u_-$ to $u_+$ on $u(\theta)=(\cos\theta,\sin\theta)$ by querying its endpoints and tangent intersection $u_A$. (c)~The affine interpolant stays below the residual, nearly reaching its minimum ($2.734$ versus $2.765$); the minimum vertex value ($1.458$) and independent-edge bound ($1.547$) are looser. (d)~\Cref{alg:gravity-yaw} prunes arcs whose bound reaches incumbent $U$ and refines the rest. Nine evaluations close the gap at $48.2^\circ$. Residuals are summed squared errors. Partial matching fixes translation or uses \Cref{sec:joint-gravity}.}\label{fig:gravity-framework}
\end{figure*}

%% file: sections/experiments.tex
\section{Experiments}\label{sec:experiments}

We evaluate planar alignment, shape retrieval and material classification, alignment under partial overlap, and registration search (\Cref{sec:exp-planar,sec:where,sec:exp-unbalanced,sec:exp-registration}).

\paragraph{Experimental design.} Solver comparisons fix the inputs, objective and evaluation; retrieval and classification fix the procedure and change only the distance. Registration comparisons change only the bound or matching subsolver within a fixed search procedure, the \emph{host}. We report adaptations to the baselines' recommended settings and sweep parameters that trade accuracy for cost.

Timings include initialization, bound evaluation and feasible updates; paired comparisons use identical inputs and one processor core. Statistics include every test pair, with paired McNemar and Wilcoxon tests for retrieval and classification. Floating-point gaps are numerical; exact certificates concern rational planar inputs. \Cref{app:experiments,app:unbalanced,app:eth-protocol,app:registration-experiments} give the protocols, \Cref{app:computing} describes the hardware, and \Cref{app:additional} reports contextual results. We will release all code and dependencies.

\input{sections/planar_experiments.tex}
\input{sections/unbalanced_experiments.tex}
\input{sections/native_with_without.tex}
\input{sections/native_planar.tex}

\input{sections/eth_experiments.tex}

\input{sections/registration_experiments.tex}

%% file: sections/planar_experiments.tex
\subsection{Global Planar Alignment}\label{sec:exp-planar}

Double-precision Rubix computes the numerical \emph{reference value} $v_{\rm ref}$ with SciPy's assignment solver~\cite{crouse2016,virtanen2020}. Costs are squared PW distances. A \emph{numerical failure} occurs when a cost exceeds $v_{\rm ref}$ by more than $10^{-9}\max(v_{\rm ref},10^{-12})$; more than $1\%$ \emph{excess} means a cost above $1.01v_{\rm ref}$. Exact-arithmetic certificates are tested separately (\Cref{app:planar-conventions,app:optimality}).

All methods use identical centered inputs, eliminating translation exactly, and optimize over $G=\mathrm{SO}(2)$; additional synthetic tests permit $G=\mathrm{O}(2)$. AM starts from covariance eigenvectors (PCA), uniform random rotations or a GW matching~\cite{alvarezmelis2018gw}. The grid baseline solves assignments at equally spaced angles and optimizes each returned matching's rotation. All minimize the same PW objective (\Cref{app:solvers}).

\paragraph{Polygon size.} Tight instances require exact arithmetic: at $n=20$, Rubix recovers all $380$ vertices with $760=2V$ assignments, against $56$ vertices in double precision (\Cref{app:optimality}). On Gaussian sets, MNIST digits and noisy copies, the median grows roughly as $n^{1.24}$, reaching $1{,}350$ at $n=200$ against the bound of $39{,}800$ (\Cref{fig:scaling}). Full reconstruction uses exactly $2V$ assignments in all 135 tests. Pruned Rubix needs a median of 12--36 assignments for $n=10$--$400$ and agrees with full reconstruction wherever both run. At $n=400$, it takes $1.50$\,s, against $13.6$\,s for a 360-angle grid and $0.51$\,s for AM (PCA), which has more than $1\%$ excess on 43 of the study's 144 pairs (\Cref{app:geometry}).

\begin{figure}[t]
\centering
\includegraphics[width=\linewidth]{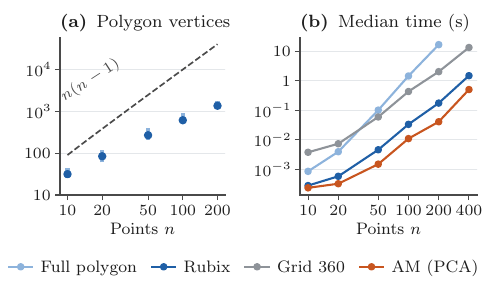}
\caption{\textbf{Observed polygon sizes remain below the worst-case bound.} Gaussian, MNIST-digit and noisy-copy pairs, timed serially on one thread. (a)~Median vertex count (dots) and range across the three families, with the sharp bound $n(n-1)$ (dashed). (b)~Median time per pair for full polygon reconstruction, the pruned Rubix search, a 360-angle grid and AM (PCA). Full reconstruction stops at $n=200$.}\label{fig:scaling}
\end{figure}

\paragraph{Accuracy and cost.}\label{sec:exp-fail} On 1,000 pairs each of MPEG-7 silhouettes ($n=64$), Gaussian sets ($n=50$) and noisy lattice neighborhoods ($n=24$), AM (PCA) has more than $1\%$ excess on 33--77\% of pairs (\Cref{tab:failures-full}). Ten random starts use five to nine times as many assignments yet retain 35--48.5\% excess rates; GW initialization fails more often than PCA. A 360-angle grid stays within $1\%$ throughout using about 12 times Rubix's assignments, but up to $5\%$ of its costs exceed the numerical threshold. Lipschitz branch-and-bound also needs far more assignments (\Cref{app:solvers}).

Among tested grids, the smallest with at most $1\%$ numerical failures needs 72--720 angles and 1.8--25 times pruned Rubix's runtime on the first 200 pairs per family. These grids are selected on the timed pairs themselves. Matching every MPEG-7 reference requires 1,440 angles and 50 times Rubix's runtime; 100-start AM takes 10--24 times as long and still fails on 2.5--11.5\% (\Cref{tab:tta,fig:where}). Single-run AM remains cheaper: mean times on MPEG-7 are 3.3\,ms for AM (PCA) and 11.8\,ms for Rubix (\Cref{app:frontier}).

\input{sections/published_planar.tex}

\paragraph{Globally optimal solvers.} APM~\cite{lian2017} subdivides rectangles of matching statistics; RPM-PA~\cite{lian2021polyhedral} expands their convex hull to a prescribed tolerance. Both eliminate the transformation, but neither bounds the assignment count. For complete planar matching they search the correlation plane, where Rubix reconstructs the polygon in $2V$ assignments with $V\le n(n-1)$ (\Cref{thm:main,prop:es}).

We run the authors' APM code and our paper-based implementation of RPM-PA; the latter shares SciPy's assignment solver with Python Rubix. All also optimize translation and nonnegative scale (\Cref{app:published-planar}). APM matches Rubix's reference with 551--562 median assignments versus 25--31. RPM-PA has 45.5--66.0\% excess rates at its published tolerance $\varepsilon_{\mathrm{RPM}}=0.3$ and 4.5--13.5\% at $0.01$. Zero tolerance removes excess but needs 50--60 median assignments, takes 9--37 times as long as Rubix and exceeds 120\,s on 70 of 600 pairs (\Cref{tab:global-solvers}).

\input{tables/global_solvers.tex}

\subsection{Retrieval and Material Classification}\label{sec:where}

We fix each retrieval or classification procedure and vary only the distance: PW computed by Rubix or AM, or an alignment-free alternative. \Cref{app:where} gives settings and diagnoses AM (PCA) failures.

\paragraph{Shape retrieval.} MPEG-7~\cite{latecki2000} has 1,400 silhouettes in 70 classes of 20. All methods use the same 64 boundary samples, centered and scaled to unit root-mean-square radius, for all 979,300 pairwise $\mathrm{SO}(2)$ distances. Bullseye measures how many of the 20 same-class shapes occur among the nearest 40, including the query; leave-one-out 1-NN and mean average precision (mAP) exclude it.

AM (PCA) has more than $1\%$ excess on $37\%$ of pairs, with a maximum cost ratio of 51. Rubix improves all three retrieval measures; 246 queries gain same-class neighbors and 89 lose some (\Cref{tab:mpeg7}). Across 14 settings, the alignment-free SE(2)-invariant distance SEINT~\cite{lin2026seint} reaches at most $41.6\%$ bullseye and $67.4\%$ 1-NN, versus $67.4\%$ and $92.5\%$ for Rubix (McNemar $p<10^{-79}$). SEINT is faster, at $0.34$ versus $5.35$\,ms per distance, but no setting matches Rubix's accuracy (\Cref{app:seint}).

\begin{table}[t]
\caption{\textbf{Retrieval on all 1,400 MPEG-7 silhouettes} ($n=64$), measured by bullseye score, leave-one-out 1-NN accuracy and mean average precision (\%). The first three rows are PW distances computed by the named solver. GW is the Gromov--Wasserstein distance, SEINT~\cite{lin2026seint} is an SE(2)-invariant transport metric, shown in its original classification setting and in the best of its 14 tested settings, and $W_2$ is the 2-Wasserstein distance without rotation. Bold marks the highest score in each column.}\label{tab:mpeg7}
\begin{center}
\input{tables/mpeg7.tex}
\end{center}
\end{table}

\paragraph{Local structure of 2D materials.} We classify atomic neighborhoods up to rotation~\cite{larsen2016ptm,britton2024ptm} in square, triangular, honeycomb and kagome lattices with unit nearest-neighbor spacing and Gaussian displacements of standard deviation $\sigma_{\mathrm{noise}}$. Each neighborhood contains a site and its 23 nearest neighbors, is centered on the site, normalized and randomly rotated. Five independent draws per noise level each supply ten labels and 50 tests per class for 1-NN classification.

AM (PCA) has more than $1\%$ excess in 74--80\% of comparisons. At $\sigma_{\mathrm{noise}}=0.12$, Rubix correctly classifies 889 of 1,000 tests, versus 770 for AM (PCA) and 835 for ten-start AM (\Cref{tab:materials}). A 36-angle grid ties Rubix at noise $0.08$ and $0.12$ and correctly classifies one more at $0.16$: dense angular search also recovers accuracy lost to initialization. Rubix's advantage over AM persists across neighborhood sizes and label counts (\Cref{tab:robust}).

On a separate 600-test comparison per noise level, alignment-free methods lead near ideal lattices (\Cref{tab:baselines}). At noise $0.08$, a bond-orientational descriptor reaches $99.2\%$ versus Rubix's $92.7\%$; at $0.04$, SEINT reaches $100.0\%$ versus $97.8\%$ (McNemar $p<0.001$). SEINT's $93.7\%$ at $0.08$ does not differ significantly from Rubix ($p=0.54$).

At noise $0.16$, Rubix instead reaches $82.3\%$, versus $57.8\%$ for the descriptor, $53.5\%$ for GW~\cite{alvarezmelis2018gw} and $44.8\%$ for SEINT. No SEINT setting matches Rubix at noise $0.12$, $0.16$ or $0.20$, although it costs $0.16$ versus $1.0$\,ms per distance (\Cref{app:seint}). Rubix also outperforms original Fiedler--W~\cite{adamo2025} and our Ping-Pong reproduction~\cite{even2024} at all five noise levels, with median paired speedups of 5.87--6.00 and 27.0--27.8 (\Cref{tab:published-materials}).

\input{tables/published_materials.tex}

\paragraph{Polycrystals and clustering.} Three synthetic polycrystals each contain eight randomly oriented grains of the four lattices. On their 3,267 shared interior test atoms, Rubix reduces AM's 797 errors to 467 ($41.4\%$); SEINT makes 1,040 errors at its best setting and 1,622 at its original setting (\Cref{fig:poly}).

Clustering 200 neighborhoods into four groups by k-medoids gives Rubix a higher adjusted Rand index (ARI) than AM at every noise level; average-linkage gains depend on the agreement measure (\Cref{tab:clustering,app:polycrystals-clustering}). Best-setting SEINT leads k-medoids at noise $0.08$ (ARI $0.70$ versus $0.55$), but trails Rubix at $0.12$ and $0.16$ and under average linkage throughout.

%% file: sections/published_planar.tex
\paragraph{Published solvers.}
On 200 pairs per family, original Fiedler--W~\cite{adamo2025} has 63.0--87.5\% excess rates and takes 1.54--6.02 times as long as Rubix. Our reproduction is faster on MPEG-7 and Gaussian pairs but has similar errors; both implementations agree on 592/600 objectives and matchings. Our Ping-Pong reproduction~\cite{even2024} and stochastic Wasserstein--Procrustes~\cite{grave2019} have higher error rates and take over 20 times as long (\Cref{tab:published-planar}).

For the common $\mathrm{SO}(2)$ objective, we restrict Fiedler--W to proper rotations, initialize Ping-Pong with Frank--Wolfe optimization on Gram matrices, and use stochastic Wasserstein--Procrustes' assignment variant, original schedule and capped batch size (\Cref{app:published-planar}).

\input{tables/published_planar.tex}

%% file: tables/published_planar.tex
\begin{table*}[t]
\centering
\caption{\textbf{Published planar solvers on 600 fixed pairs.} Each family has 200 pairs. Excess is the percentage of costs more than $1\%$ above Rubix's double-precision reference. Ratios are paired median solver/Rubix times including initialization; bold marks minimum excess and ratios above one. Both Fiedler--W implementations use identical inputs and processor, timed separately against Rubix. Every output is valid; \Cref{app:published-planar} gives adaptations.}\label{tab:published-planar}
\begin{tabular}{@{}lrrrrrr@{}}
\toprule
& \multicolumn{2}{c}{MPEG-7 ($n=64$)} & \multicolumn{2}{c}{Gaussian ($n=50$)} & \multicolumn{2}{c}{Lattice ($n=24$)}\\
\cmidrule(lr){2-3}\cmidrule(lr){4-5}\cmidrule(l){6-7}
Method & Excess (\%) & \shortstack{Time\\ratio} & Excess (\%) & \shortstack{Time\\ratio} & Excess (\%) & \shortstack{Time\\ratio}\\
\midrule
AM (PCA)~\cite{adamo2025} & 29.5 & $0.33\times$ & 58.0 & $0.32\times$ & 81.0 & $0.48\times$\\
Fiedler--W, original~\cite{adamo2025} & 63.0 & $\mathbf{1.54}\times$ & 81.5 & $\mathbf{2.04}\times$ & 87.5 & $\mathbf{6.02}\times$\\
Fiedler--W, reproduced~\cite{adamo2025} & 63.0 & $0.75\times$ & 81.5 & $0.85\times$ & 88.0 & $\mathbf{1.96}\times$\\
Ping-Pong, reproduced~\cite{even2024} & 77.0 & $\mathbf{33.6}\times$ & 89.5 & $\mathbf{27.0}\times$ & 88.5 & $\mathbf{22.3}\times$\\
Stochastic WP, LAP variant~\cite{grave2019} & 79.5 & $\mathbf{289}\times$ & 91.0 & $\mathbf{254}\times$ & 94.0 & $\mathbf{301}\times$\\
Rubix & \textbf{0.0} & $1.00\times$ & \textbf{0.0} & $1.00\times$ & \textbf{0.0} & $1.00\times$\\
\bottomrule
\end{tabular}
\end{table*}

%% file: tables/global_solvers.tex
\begin{table*}[t]
\centering
\caption{\textbf{Global solvers on 600 planar pairs.} Complete similarity matching on \Cref{tab:published-planar}'s inputs. Excess is the percentage of costs more than $1\%$ above the reference; assignments and milliseconds are single-core medians. RPM-PA is our implementation, sharing Python and SciPy assignment with Rubix; original APM runs under Octave with C++ assignment (\Cref{app:published-planar}). RPM-PA's published tolerance is $\varepsilon_{\mathrm{RPM}}=0.3$; zero tolerance exceeds 120\,s on 70 pairs. Bold marks the fewest assignments and lowest time among zero-excess methods.}\label{tab:global-solvers}
\begin{tabular}{@{}lrrrrrrrrr@{}}
\toprule
& \multicolumn{3}{c}{MPEG-7 ($n=64$)} & \multicolumn{3}{c}{Gaussian ($n=50$)} & \multicolumn{3}{c}{Lattice ($n=24$)}\\
\cmidrule(lr){2-4}\cmidrule(lr){5-7}\cmidrule(l){8-10}
Method & Excess & LAPs & ms & Excess & LAPs & ms & Excess & LAPs & ms\\
\midrule
APM~\cite{lian2017} & 0.0 & 551 & 245.1 & 0.0 & 562 & 179.2 & 0.0 & 556 & 133.8\\
RPM-PA~\cite{lian2021polyhedral}, $\varepsilon_{\mathrm{RPM}}=0.3$ & 45.5 & 10 & 10.6 & 66.0 & 10 & 9.3 & 53.0 & 10 & 8.2\\
RPM-PA~\cite{lian2021polyhedral}, $\varepsilon_{\mathrm{RPM}}=0.01$ & 4.5 & 26 & 27.9 & 13.5 & 36 & 33.1 & 5.0 & 38 & 32.1\\
RPM-PA~\cite{lian2021polyhedral}, $\varepsilon_{\mathrm{RPM}}=0$ & 0.0 & 50 & 59.7 & 0.0 & 60 & 57.8 & 0.0 & 54 & 46.5\\
Rubix & 0.0 & \textbf{25} & \textbf{5.8} & 0.0 & \textbf{30} & \textbf{4.2} & 0.0 & \textbf{31} & \textbf{1.2}\\
\bottomrule
\end{tabular}
\end{table*}

%% file: tables/mpeg7.tex
\begin{tabular}{@{}lrrr@{}}
\toprule
Distance & \shortstack{Bullseye\\(\%)} & \shortstack{1-NN\\(\%)} & \shortstack{mAP\\(\%)}\\
\midrule
Rubix & \textbf{67.4} & \textbf{92.5} & \textbf{59.2}\\
AM (PCA)~\cite{adamo2025} & 65.5 & 91.6 & 57.1\\
AM (GW)~\cite{alvarezmelis2018gw} & 56.4 & 88.6 & 48.1\\
GW~\cite{alvarezmelis2018gw,flamary2021} & 56.3 & 88.6 & 47.3\\
SEINT~\cite{lin2026seint} & 36.2 & 63.1 & 24.1\\
SEINT, best setting~\cite{lin2026seint} & 41.6 & 67.4 & 29.4\\
$W_2$ (no rotation) & 21.0 & 52.1 & 11.2\\
\bottomrule
\end{tabular}

%% file: tables/published_materials.tex
\begin{table*}[t]
\centering
\caption{\textbf{Published solvers as distance subroutines for lattice classification.} Accuracy is the mean $\pm$ sample standard deviation (\%) across three draws, which total 600 tests per noise level. Bold marks a higher mean accuracy than both other solvers. These draws are shared with the descriptor and GW comparisons and differ from the five draws of \Cref{tab:materials} (\Cref{app:published-planar}).}\label{tab:published-materials}
\begin{tabular}{@{}lrrrrr@{}}
\toprule
Method & $\sigma_{\mathrm{noise}}=0.04$ & $0.08$ & $0.12$ & $0.16$ & $0.20$\\
\midrule
Fiedler--W, original~\cite{adamo2025} & $85.0\pm2.2$ & $77.7\pm3.6$ & $67.0\pm2.5$ & $55.3\pm5.0$ & $41.5\pm3.5$\\
Ping-Pong, reproduced~\cite{even2024} & $85.0\pm6.4$ & $77.7\pm4.5$ & $66.8\pm5.0$ & $57.3\pm3.5$ & $42.5\pm4.8$\\
Rubix & $\mathbf{97.8}\pm0.8$ & $\mathbf{92.7}\pm2.8$ & $\mathbf{88.7}\pm1.9$ & $\mathbf{82.3}\pm2.5$ & $\mathbf{64.0}\pm6.1$\\
\bottomrule
\end{tabular}
\end{table*}

%% file: sections/unbalanced_experiments.tex
\subsection{Unbalanced Alignment}\label{sec:exp-unbalanced}

A bijection forces unrelated points to match under partial overlap. The penalized model \eqref{eq:partial-spatial} charges $\lambda$ per unmatched point and globally optimizes matching and rotation. Its equivalence to rigid TV-unbalanced transport (\Cref{rem:unbalanced}) motivates comparisons with unbalanced GW (UGW)~\cite{sejourne2021}, fused UGW~\cite{thual2022fugw} and partial GW~\cite{chapel2020partial}. Synthetic and lattice tests supply the rotation center as an anchor; fused UGW receives distances to it as features. Every method uses validation-selected settings. Scans have no validation split, so GW pipelines receive their best test settings (\Cref{app:unbalanced}).

\paragraph{Penalized matching and global search.} On 40-point planar sets, Rubix recovers rotations within $5^\circ$ in at least 99\% of tests at every shared fraction, down to 30\% (\Cref{tab:unbalanced}). Our diagnostic calls a miss a \emph{solver failure} if a true-correspondence restart improves a GW objective and recovers the rotation, or Rubix closes AM's gap at the same penalty and recovers it. Other misses are \emph{objective failures}; for GW this classification does not establish global optimality (\Cref{app:unbalanced-diagnostic}). Balanced Rubix closes every gap but misses 403 of 800 rotations. Sixteen-start AM misses 29, including 26 solver failures; penalized Rubix misses five. GW methods show both categories (\Cref{tab:unbalanced-diagnosis}).

Selecting settings by validation correspondence F1 also makes Rubix the most accurate matcher, at 10--18\,ms per solve versus 32--88\,ms for UGW. At 160 points per set, Rubix recovers every rotation with a closed gap in median time 0.43\,s, versus 46\% recovery in 1.4\,s for UGW (\Cref{app:unbalanced-planar-results}).

\paragraph{Full 3D rotation.} Quaternion search (\Cref{sec:quaternion-bound}) closes all gaps and recovers all rotations in 200 tests with 24-point sets, at a median 0.5\,s versus 0.002--0.05\,s for GW methods. At 30\% overlap, fused UGW recovers 38\%, 60-start AM 34\%, and UGW and partial GW at most 2\% (\Cref{tab:unbalanced-3d,app:unbalanced-3d}).

\paragraph{Defective lattices.} Fixed-size neighborhoods fill vacancies with atoms from the next shell; a fixed radius instead preserves variable cardinality. At thermal noise $0.08$, penalized Rubix on radius-$2.6$ neighborhoods exceeds balanced Rubix on 24-atom neighborhoods by 18.5--30.5 percentage points as defect rates rise from $0.1$ to $0.3$ (McNemar $p\le1.5\times10^{-10}$; \Cref{tab:unbalanced}). It also exceeds UGW and fused UGW ($p\le2\times10^{-3}$), partial GW and SEINT~\cite{lin2026seint}. With no defects, Rubix and UGW reach 100\%; at every nonzero defect rate, 16-start AM ties Rubix. Here the improvement comes from variable-size neighborhoods and penalized matching, without further benefit from global search.

\paragraph{Real scans.} All identified ETH scan pairs of Li et al.~\cite{li2024screw} yield gravity-aligned problems on representative subsets, or \emph{coresets}, of 64 source and 80 target points. With reference translation, Rubix recovers 27 of 31 yaws within $5^\circ$ and closes every gap in median time 36\,ms; offsetting translation by 0.5\,m gives 25 recoveries. Best-setting fused UGW recovers 23 and 19, and 16-start AM 22 and 20 (\Cref{tab:unbalanced-eth,app:unbalanced-eth}). Gains over UGW and partial GW are significant ($p\le1.5\times10^{-4}$), but those over fused UGW and AM are not ($p\ge0.12$). With translations from Li et al.~\cite{li2024screw} or RAP~\cite{pan2026rap}, Rubix recovers 20 and 23 yaws, leading the matchers but trailing each front end's 25 native yaws.

Despite 44\% median voxel overlap, independently sampled coresets often occupy different surfaces: the reference penalized matching contains a median of eight source points. Three of Rubix's four misses have at most three such points (\Cref{fig:eth-overlap}). Among 17 pairs with fewer than ten, Rubix recovers 14 yaws, fused UGW 13, AM 10, partial GW 3 and UGW 1. Controlled synthetic overlap shows a similar decline for GW methods.

When translation is unknown, GLORES~\cite{glores2020} with Rubix bounds registers 190 of 212 laser-scan configurations, versus 169 for UGW plus trimmed ICP, 151 for partial GW plus trimmed ICP and 139 for ICP alone ($p\le6\times10^{-6}$; \Cref{tab:unbalanced-scans}). Native GLORES registers 189, so both global searches share the pose advantage. GW needs neither supplied translation nor a common frame; Rubix needs an anchor or a host that searches translation. \Cref{app:unbalanced-eth} also evaluates our Python port of Li et al.'s MATLAB solver~\cite{li2023gravity}, which receives no translation.

\begin{figure*}[t]
\centering
\includegraphics[width=\textwidth]{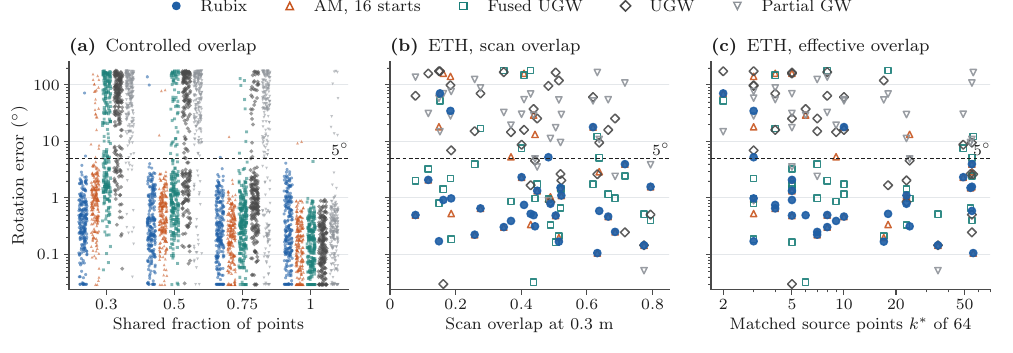}
\caption{\textbf{Rotation error against overlap, controlled and real.} (a)~Every test instance of the anchored planar sets that share the given fraction of points (200 per fraction, validation-selected settings), offset horizontally by method. (b,c)~All identified ETH scan pairs of Li et al. with the reference translation supplied, each GW method at its best setting, against (b)~the fraction of source voxels with a target voxel within $0.3$\,m at the reference pose and (c)~the number of source points that the penalized model matches at the reference pose. The dashed line marks the $5^\circ$ success threshold; errors below $0.03^\circ$ are drawn at $0.03^\circ$.}\label{fig:eth-overlap}
\end{figure*}

\begin{table*}[t]
\centering
\caption{\textbf{Penalized matching with global search.} Left: $5^\circ$ rotation recovery on anchored synthetic sets (200 tests per overlap), and solver/objective failures over all 800. Right: lattice 1-NN accuracy at noise $0.08$, vacancy rate as shown and half as many interstitials, against defect-free references (200 queries per column). Settings are validation-selected; \Cref{tab:unbalanced-full,tab:unbalanced-full-lattice} retain every cell, best test baselines and SEINT. Bold marks the best value.}\label{tab:unbalanced}
\input{tables/unbalanced.tex}
\end{table*}

\begin{figure*}[t]
\centering
\includegraphics[width=\textwidth]{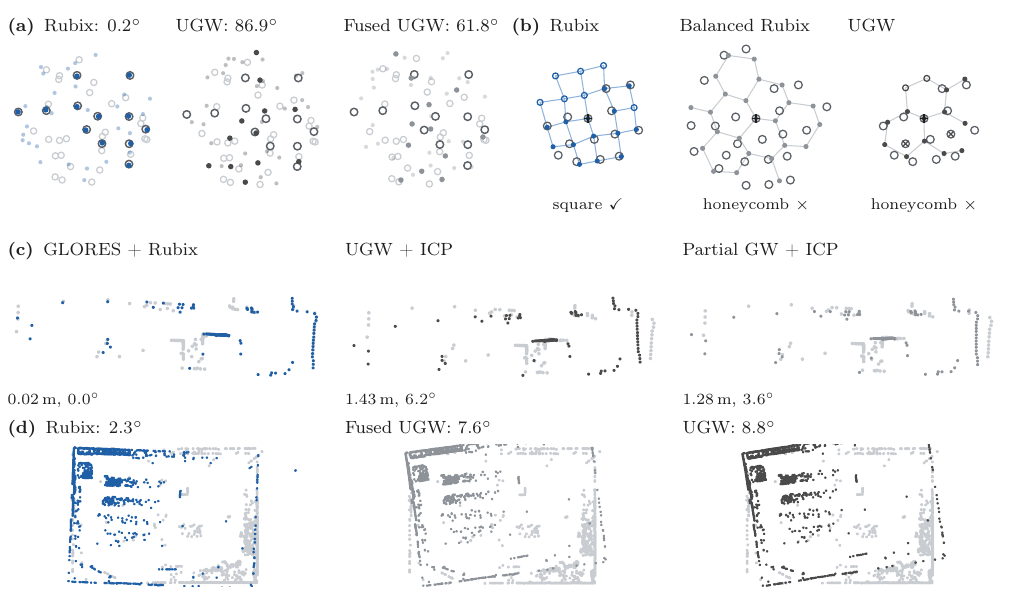}
\caption{\textbf{Partial alignments found by Rubix and missed by unbalanced GW.} (a)~40-point sets with 30\% overlap: targets are rings (shared targets dark); rotated sources are solid when shared. Titles give rotation errors. (b)~A noisy square neighborhood ($\sigma_{\mathrm{noise}}=0.08$): observed atoms are rings and the site a cross-hair. Each nearest reference is aligned and bonded, solid when matched, hollow otherwise; unmatched observations are crossed. Predicted classes appear below. Balanced Rubix uses 24 atoms. (c)~Planar scans, with gray targets and reported pose errors. (d)~ETH Office from above, floor/ceiling removed, with gray targets, supplied reference translation and reported yaw errors. Each panel uses the first dataset-ordered case where Rubix succeeds and displayed baselines fail (\Cref{app:unbalanced}).}\label{fig:unbalanced-examples}
\end{figure*}

%% file: tables/unbalanced.tex
\begin{tabular}{@{}lrrrrrrrr@{}}
\toprule
& \multicolumn{4}{c}{Partial overlap} & \multicolumn{4}{c}{Defective lattices}\\
\cmidrule(lr){2-5}\cmidrule(l){6-9}
& \multicolumn{3}{c}{Rotations (\%), shared} & Failures & \multicolumn{4}{c}{Accuracy (\%), defect rate}\\
Method & 0.75 & 0.5 & 0.3 & solver/obj. & 0 & 0.1 & 0.2 & 0.3\\
\midrule
Rubix, unbalanced & \textbf{99.5} & \textbf{99.0} & \textbf{99.0} & \textbf{0/5} & \textbf{100.0} & \textbf{99.0} & \textbf{94.5} & \textbf{90.5}\\
Rubix, balanced & 59.5 & 28.5 & 10.5 & 0/403 & 98.0 & 80.5 & 66.5 & 60.0\\
AM, 16 starts & \textbf{99.5} & 98.5 & 88.5 & 26/3 & 95.5 & \textbf{99.0} & \textbf{94.5} & \textbf{90.5}\\
UGW~\cite{sejourne2021} & 83.5 & 43.0 & 20.0 & 207/100 & \textbf{100.0} & 94.0 & 78.0 & 69.0\\
Fused UGW~\cite{thual2022fugw} & 96.0 & 74.0 & 44.5 & 63/108 & 99.5 & 90.5 & 79.5 & 70.0\\
Partial GW~\cite{chapel2020partial} & 32.5 & 13.5 & 7.5 & 501/1 & 63.5 & 60.0 & 56.5 & 45.5\\
\bottomrule
\end{tabular}

%% file: sections/native_with_without.tex
\subsection{Registration with Rubix Bounds}\label{sec:exp-registration}\label{sec:native-with-without}

We insert Rubix bounds into native GLORES~\cite{glores2020}, real gravity-aligned matching, and synthetic yaw, joint yaw--translation and full-rotation searches. Each comparison fixes the host objective, search policy, feasible updates and stopping rule. A \emph{native} method uses its authors' implementation and objective. The \emph{independent-edge bound} minimizes each edge separately over a cell before assignment, allowing different poses per edge. The \emph{vertex bound} uses Rubix's assignments at the yaw-triangle vertices but takes their smallest value, isolating the benefit of minimizing the interpolant on the arc.

%% file: sections/native_planar.tex
\subsubsection{A Rubix Bound inside GLORES}\label{sec:native-planar}

GLORES~\cite{glores2020} searches planar rotation and translation to minimize the $k=\lceil\rho_{\mathrm{keep}} n\rceil$ smallest squared nearest-neighbor distances from source to target. This trimmed, directed objective allows unequal sizes, repeated targets and unmatched sources. We replace each yaw-interval--translation-box bound by the larger of its native and Rubix bounds; all added queries are timed.

We test two Rubix bounds. The first applies the arc interpolation of \Cref{sec:gravity} to the yaw interval and relaxes the translation independently for each point. The second, stronger bound keeps one translation common to all selected pairs. Following \Cref{lem:common-translation}, it subtracts $k$ times the squared distance to the center of the box, which makes the objective concave separately in the yaw coordinates and in the translation, and then combines the interpolants at the four corners of the box (\Cref{app:common-translation}).

Sixty Hokuyo scan pairs from three IILABS sequences~\cite{iilabs2025}, acquired at different times, yield 240 configurations with two beam counts and two retained fractions (\Cref{app:native-planar}). All four bounds reach the native $1\%$ tolerance throughout. Arc interpolation gives a median paired speedup of $1.067$; common translation raises it to $1.331$ and reduces nodes by a factor of $1.844$ (\Cref{tab:native-planar}). The common-translation bound is faster throughout, with a scan-pair-stratified bootstrap interval of $[1.296,1.376]$. Pose success remains essentially unchanged, at 190 versus 189 of 212 reference configurations.

\input{tables/native_planar.tex}

%% file: tables/native_planar.tex
\begin{table*}[t]
\centering
\caption{\textbf{Replacing the bound inside native GLORES~\cite{glores2020}.} All four variants complete the 240 configurations at the same numerical tolerance of $1\%$, with the same trimmed objective and numerical safeguards (\Cref{app:native-planar}). Times and nodes are medians; the speedup is the median paired ratio of the native time to the time of the variant, with preprocessing and bound evaluation included. Bold marks improvements over native GLORES. Pose success requires a translation error of at most $0.25$\,m and a yaw error of at most $2^\circ$ on the 212 configurations with motion-capture references; the 28 configurations without references count toward the runtime only. Native PLICP~\cite{censi2008}, which has its own point-to-line objective and no global bound, registers 146 configurations in a median of 0.0118\,s from the zero pose and 185 in 0.2286\,s from 16 initial yaws.}\label{tab:native-planar}
\begin{tabular}{@{}lrrrr@{}}
\toprule
Method & Time (s) & Nodes & Paired speedup & Pose success\\
\midrule
GLORES~\cite{glores2020} & 2.330 & 2,374.5 & $1.000\times$ & 189/212\\
GLORES + vertex bound~\cite{glores2020} & 2.315 & 2,263 & $0.996\times$ & 189/212\\
GLORES + Rubix arc bound~\cite{glores2020} & \textbf{2.103} & \textbf{2,035} & $\mathbf{1.067}\times$ & 189/212\\
GLORES + Rubix, common translation~\cite{glores2020} & \textbf{1.608} & \textbf{1,352} & $\mathbf{1.331}\times$ & \textbf{190/212}\\
\bottomrule
\end{tabular}
\end{table*}

%% file: sections/eth_experiments.tex
\subsubsection{Matching Subproblems from Real 3D Scans}\label{sec:exp-eth}

We test six adjacent Office and Courtyard scan pairs from ETH~\cite{theiler2015}. Pairs share scans and are not independent. Without measured IMU calibration, we use the scanner's nominal $+z$ axis as gravity; reference poses serve only evaluation.

Coresets of 64 source and 80 target points yield fixed 32-point subsets selected by partial assignment at each initial yaw. The solvers optimize yaw and bijection with heights retained and translation eliminated by centering. A common pipeline ranks poses by penalized matching, refines partial matches and applies native gravity-constrained ICP~\cite{kubelka2022}. Global optimality concerns the fixed subsets, not their selection or refinement (\Cref{app:eth-protocol}).

AM, the independent-edge bound, the vertex bound and Rubix use 16 equally spaced initial yaws. Global solvers target a numerical mean-squared gap of $10^{-5}$\,m$^2$ within five seconds; complete-pipeline results appear in \Cref{app:eth-pipelines}.

\input{tables/eth_subsolver_results.tex}

\paragraph{Global matching and search cost.} All global solvers close every gap and reach the same reference on 96/96 subproblems, versus 22/96 for AM (\Cref{tab:eth-blocks}). Rubix takes a median $1.46$\,ms, with paired speedups of $1.63$ over the vertex bound and $85.9$ over the independent-edge bound, reducing assignments from 5,544 and 224,508 to 2,226.

\input{sections/gravity_published_insertion.tex}

%% file: tables/eth_subsolver_results.tex
\begin{table*}[t]
\centering
\caption{\textbf{Gravity matching on ETH subsets.} Six pairs and 16 yaws give 96 fixed 32-point problems per solver. Global solvers share the host and target mean-squared gap $10^{-5}$\,m$^2$ within five seconds on one core; AM supplies no lower bound. Assignments and total time sum all problems; bold marks reductions from the vertex bound (\Cref{app:eth-protocol}).}\label{tab:eth-blocks}
\begin{tabular}{lrrrrr}
\toprule
Subsolver & Reference attained & Gap closed & Assignments & Total (s) & Median (ms)\\
\midrule
AM & 22/96 & 0/96 & 338 & 0.062 & 0.58\\
Independent-edge bound & 96/96 & 96/96 & 224,508 & 13.910 & 138.78\\
Vertex bound & 96/96 & 96/96 & 5,544 & 0.232 & 2.44\\
Rubix interpolant & 96/96 & 96/96 & \textbf{2,226} & \textbf{0.144} & \textbf{1.46}\\
\bottomrule
\end{tabular}
\end{table*}

%% file: sections/gravity_published_insertion.tex
\paragraph{A matching stage after published registration methods.}\label{sec:published-gravity-insertion}
We insert this stage after the original implementations of Quatro~\cite{lim2022quatro}, RAP~\cite{pan2026rap} and Li et~al.~\cite{li2024screw} on all identified ETH pairs, with feature-pipeline adaptations in \Cref{app:published-gravity}. The native yaw and 16 equally spaced yaws select 17 pairs of 32-point subsets. Only the subset solver changes; assignment, selection, ranking and final gravity ICP remain fixed.

All three global solvers close all 2,108 subset gaps and return identical final poses (\Cref{tab:gravity-published-insertion}). Rubix's median stage time is 10.4\,ms, with paired speedups of $45.8$--$46.9$ over the independent-edge bound and $1.46$--$1.48$ over the vertex bound. \Cref{app:native-with-without} includes front-end time and pose success.

%% file: sections/registration_experiments.tex
\subsubsection{Yaw Search across Matching Models}\label{sec:exp-gravity}

The yaw search compares Rubix's interpolant, the independent-edge bound and the vertex bound for complete, prescribed-cardinality and penalized matching (\Cref{alg:gravity-yaw}). The 72 scenes vary $n\in\{32,64,128\}$, cardinality ratio, unmatched fraction and vertical spread. Complete matching removes translation by centering; partial models fix a shared translation. Three models, three summed-squared gap targets ($10^{-2}$, $10^{-4}$, $10^{-6}$) and three bounds give 1,944 otherwise identical solves (\Cref{app:gravity-subsolver-protocol}).

\begin{figure*}[t]
\centering
\includegraphics[width=\textwidth]{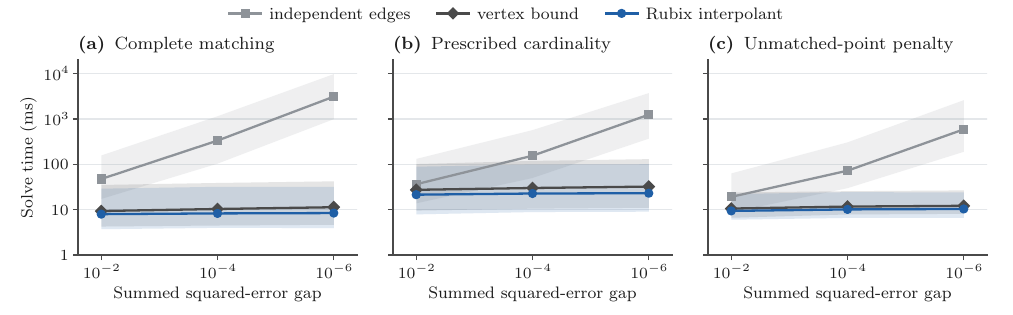}
\caption{\textbf{Faster yaw search under three matching models.} Median times and interquartile ranges over 72 scenes, including unresolved runs: (a)~complete matching after centering; (b)~prescribed-cardinality and (c)~penalized matching at fixed translation. Rubix and the vertex bound close every gap. From coarse to fine targets, the independent-edge bound closes 72, 72, 50 gaps in (a), 72, 72, 67 in (b), and all 72 in (c). At $10^{-6}$, paired median vertex/Rubix time ratios are $1.38$, $1.35$, $1.21$.}\label{fig:gravity-bounds}
\end{figure*}

Rubix has the lowest median time in each model and, like the vertex bound, closes all 648 gaps; the independent-edge bound closes 621 (\Cref{fig:gravity-bounds}). At $10^{-6}$, Rubix closes all 216 gaps in median times of 8.41, 23.05 and 10.36\,ms, while the vertex bound needs $1.68$, $1.50$ and $1.48$ times as many assignments.

For problems both bounds solve, median speedups over the independent-edge bound are $274$, $36.9$ and $52.2$; the figure includes its 27 unresolved searches at the limits. Common-yaw matching provides most of the improvement, with a further benefit from arc interpolation. These searches do not inherit the polynomial guarantee for complete planar matching.

\subsubsection{Joint Search over Yaw and Translation}\label{sec:exp-joint-gravity}

\begin{figure*}[t]
\centering
\includegraphics[width=\textwidth]{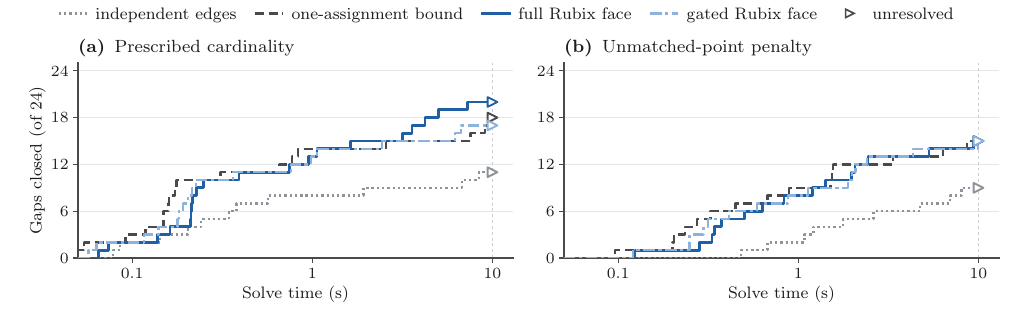}
\caption{\textbf{Witness bounds close more joint yaw--translation gaps.} Each curve counts solved scenes out of 24 under (a)~prescribed-cardinality and (b)~penalized matching, at a summed-squared gap of $10^{-4}$. Triangles mark unresolved runs' last observed times; the dotted line marks 10\,s. Full-face closes 20 and 15 gaps, one-assignment 18 and 15, the independent-edge bound 11 and 9, and gated bounds 17 and 15. All 192 runs are included.}\label{fig:joint-gravity}
\end{figure*}

Using the anchor cover of \Cref{lem:joint-cover}, we jointly search yaw, translation and partial matching on 24 scenes with $n\in\{8,16,32\}$. The independent-edge bound is compared with the \emph{full-face} witness bound~\eqref{eq:joint-witness-guarantee}, using all 24 cell vertices; its \emph{one-assignment} difference variant (\Cref{sec:bound-cost,app:gravity-difference}); and a \emph{gated} combination that spends no more assignments on full-face evaluations than cheaper bounds have used. Initial cells, cardinality partitions and initialization are fixed. The two partial models give 192 solves, targeting a summed-squared gap of $10^{-4}$ within 10\,s (\Cref{app:joint-gravity-protocol}).

Full-face closes 20/24 prescribed-cardinality and 15/24 penalized gaps (\Cref{fig:joint-gravity}). It takes longer than one-assignment on problems both solve, but adds two prescribed-cardinality successes and loses none. On the nine penalized problems solved by the independent-edge bound, one-assignment and full-face are $6.52$ and $5.20$ times faster and close six more gaps. At $n=32$, full-face closes four of eight prescribed-cardinality gaps and one penalized gap; the independent-edge bound closes none. Bound cost therefore matters alongside tightness.

\begin{figure*}[t]
\centering
\includegraphics[width=\textwidth]{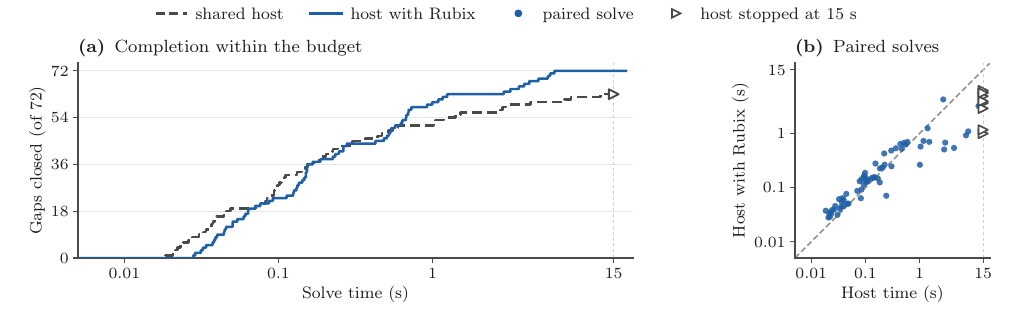}
\caption{\textbf{Rubix closes every spatial-alignment gap.} Complete SE(3) alignment with a 15\,s limit and mean-squared gap target $10^{-5}$. (a)~Solved scenes against time. (b)~Paired times; points below the diagonal favor Rubix, and triangles mark nine host timeouts whose completion times are unknown. Both variants return identical poses and matchings throughout; total CPU time includes initialization, assignments, bounds and feasible updates.}\label{fig:se3-shared}
\end{figure*}

\subsubsection{Full Spatial Rotation}\label{sec:exp-se3}

Complete SE(3) alignment uses 72 synthetic scenes with $n\in\{8,16,32,64\}$, three point distributions, two noise levels and centered inputs. Rubix adds exterior support queries and witness bounds to the quaternion host with rotation-cap bounds, keeping its search policy and assignment solver fixed (\Cref{alg:se3-search,app:se3-protocol}).

At a numerical mean-squared gap of $10^{-5}$ and a 15\,s limit, Rubix closes 72/72 gaps versus 63/72 and reduces total CPU time from 189.2 to 59.1\,s, including every search and bound evaluation (\Cref{fig:se3-shared}). The nine host timeouts close in 1.00--6.18\,s with Rubix, giving speedup lower bounds of $2.43$--$15.06$. \Cref{app:spatial-ablations} compares five bound-selection policies.

\paragraph{Subsets of real scans.} On RAP-selected 32-point subsets of all identified ETH pairs~\cite{li2024screw,theiler2015,pan2026rap}, both variants close every gap with identical poses and matchings. Rubix reduces search but increases runtime; \Cref{app:se3-real-protocol,app:se3-real-results} report this comparison and a control without RAP.

%% file: sections/discussion.tex
\section{Conclusion}\label{sec:discussion}

Rubix solves planar PW alignment through the permutation polygon. Its sharp $n(n-1)$ vertex bound answers Rote's question~\cite{mitchell2009} and yields an exact $\mathcal O(n^5)$ algorithm, with output-sensitive cost $\mathcal O(n^3V)$ for $V$ vertices. Under the stated initialization and tie rules, AM visits matchings with vertex correlations; a visited matching with nonzero correlation is fixed exactly when its vertex is stationary. Rubix selects a farthest vertex. The resulting distances improve retrieval, noisy crystal classification and k-medoids clustering over AM, while penalized partial matching handles incomplete sets. Spatial assignment bounds reduce registration search cost by retaining a common pose across residuals; pose accuracy remains a separate criterion.

\paragraph{Limitations.} The quadratic bound requires equal planar weights. Rational weights with common denominator $N\ge2$ need at most $N(N-1)$ representatives by duplication, but three unequally weighted points can already give $8>6$ vertices (\Cref{cor:weights,app:weighted-counterexample}). Exactness requires exact assignment and support comparisons. Partial matching with free translation needs a registration host or our joint search, which currently reaches tens of points. Spatial guarantees depend on numerical scale and tolerance and do not establish exact polynomial complexity in the input bit length (\Cref{prop:gravity-accuracy,prop:joint-gravity-accuracy}). A closed numerical gap need not recover the pose: gravity bias can leave errors invisible to the objective (\Cref{sec:exp-gravity-bias,app:native-extensions}). Shared scans and reused scan pairs also make the ETH and planar samples dependent.

\paragraph{Open directions.} We conjecture superpolynomial worst-case vertex growth in $n$ for the full matrix-correlation polytope in each fixed dimension $d\ge2$ (\Cref{app:matrix-conjecture}); the planar scalar-correlation polygon proved here is distinct. The polygon's average vertex count, the complexity under general linear maps, and polynomial enumeration for penalized partial rotation matching remain open. In the partial case, matching-dependent constants prevent the $n(n-1)$ argument from applying. Further directions include reusing assignment prices as in Zikan's sweep~\cite{zikan1991}, adaptive-precision certificates~\cite{shewchuk1997}, uncertain gravity and practical moment-polytope representations for partial registration (\Cref{app:partial-moments}).

%% file: appendix/content.tex
\clearpage
\appendix
\setcounter{equation}{0}
\setcounter{figure}{0}
\setcounter{table}{0}
\setcounter{theorem}{0}
\setcounter{algorithm}{0}
\renewcommand{\thetheorem}{A.\arabic{theorem}}
\renewcommand{\thealgorithm}{A.\arabic{algorithm}}
\renewcommand{\theHtheorem}{appendix.\arabic{theorem}}
\providecommand{\theHalgorithm}{}
\renewcommand{\theHalgorithm}{appendix.\arabic{algorithm}}
\renewcommand{\theequation}{A.\arabic{equation}}
\renewcommand{\thefigure}{A.\arabic{figure}}
\renewcommand{\thetable}{A.\arabic{table}}
\renewcommand{\theHequation}{appendix.\arabic{equation}}
\renewcommand{\theHfigure}{appendix.\arabic{figure}}
\renewcommand{\theHtable}{appendix.\arabic{table}}
\crefalias{section}{appendix}\crefalias{subsection}{appendix}\crefalias{subsubsection}{appendix}
\section*{Appendices}
\input{appendix/overview.tex}

\input{appendix/proofs.tex}
\input{appendix/landscape.tex}

\input{appendix/optimality.tex}
\input{appendix/spatial.tex}
\input{appendix/spatial_algorithms.tex}
\input{appendix/experiments.tex}
\input{appendix/unbalanced.tex}
\input{appendix/registration_experiments.tex}
\input{appendix/additional_results.tex}

%% file: appendix/overview.tex
\paragraph{Guide to the appendices.}
\Cref{app:proofs,app:landscape,app:optimality} give the planar proofs, extensions and exact examples; \Cref{app:spatial,app:spatial-algorithms} give the spatial proofs and algorithms. Experimental details, partial matching, registration and contextual results follow in \Cref{app:experiments,app:unbalanced,app:registration-experiments,app:additional}; hardware is listed in \Cref{app:computing}.

%% file: appendix/proofs.tex
\section{Proofs of the Planar Results}\label{app:proofs}

This appendix proves the planar reduction, the vertex bound and the guarantees of the algorithm. The proof of the bound has three ingredients. Assignment duality constrains the cycle of reassigned points at a transition; a count of continuous angle changes determines the phase jump of the cycle; and genericity, with a perturbation argument and an accounting over a full turn, extends the bound to all inputs. The subsections follow the order in which the results depend on each other.

Throughout, let $n\ge2$ and $a,b\in\mathbb C^n$, and let $V$ be the number of vertices of $\mathcal P=\mathcal P(a,b)$ in~\eqref{eq:polygon}. We write $z\times w=\Im(\bar z w)$ and $\mathbb T=\{u\in\mathbb C:|u|=1\}$. The clockwise quarter-turn $R_-(z)=-\mathrm i z$ satisfies $R_-(e)\times e=|e|^2$ and $\langle R_-(e),w\rangle=-(e\times w)$. For $u\in\mathbb T$, we set $\mathcal M_u=\argmax_\sigma\Re(\bar uz_\sigma)$, and $P_\sigma$ denotes the permutation matrix of $\sigma$.

\subsection{The PW Reduction}\label{app:elementary-reductions}\label{app:reduction}
\begin{proof}[Proof of \Cref{lem:reduction}]
For a fixed matching, minimizing \eqref{eq:rotation-cost} over the rotation replaces $\Re(e^{\mathrm i\theta}z_\sigma)$ by $|z_\sigma|$. The maximum of this convex function over the hull of the finitely many correlations is attained at a vertex, and every vertex is the correlation of a matching. For a nonzero maximum, $e^{\mathrm i\theta}z_\sigma=|z_\sigma|$ gives $\theta=-\Arg z_\sigma$; a zero maximum makes every rotation optimal. A reflection followed by a rotation replaces $y$ by $\bar y$, so comparing the two polygons covers $\mathrm O(2)$. The argument permits several optimal matchings and rotations.
\end{proof}

\subsection{Genericity and Perturbation}\label{app:generic}
The counting argument of \Cref{lem:cycle,lem:jump,lem:counting} assumes generic inputs. We define genericity, show that it holds outside a set of measure zero (\Cref{lem:generic}), and prove \Cref{lem:semicont}, which extends the bound of \Cref{thm:main} to arbitrary inputs.

\subsubsection{Generic Pairs}\label{app:genericity}
\begin{definition}\label{def:generic}
The pair $(a,b)$ is \emph{generic} if (G1) the map $\sigma\mapsto z_\sigma$ is injective and (G2) no three distinct points $z_\sigma$ are collinear.
\end{definition}

\begin{lemma}\label{lem:generic}
The non-generic pairs are contained in a proper real-algebraic subset of $\mathbb C^{2n}\cong\mathbb R^{4n}$. In particular, generic pairs are dense and have full Lebesgue measure.
\end{lemma}

\begin{proof}[Proof of \Cref{lem:generic}]
\emph{Distinct correlations.} For $\sigma\ne\tau$, the polynomial $z_\sigma-z_\tau=a^\top(P_\sigma-P_\tau)b$ is not identically zero: the coordinate vectors $a=e_i$, $b=e_j$ at a nonzero entry $(i,j)$ of $P_\sigma-P_\tau$ give a nonzero value. Thus (G1) fails only on a proper algebraic set.

\emph{Noncollinear triples.} For distinct $\sigma_1,\sigma_2,\sigma_3$, define $D_1=P_{\sigma_2}-P_{\sigma_1}$, $D_2=P_{\sigma_3}-P_{\sigma_1}$ and
$\Phi(a,b)=\Im\big(\overline{a^\top D_1b}\;a^\top D_2b\big)$. Collinearity is equivalent to $\Phi=0$, so we must show that $\Phi$ is not identically zero.
The matrices $D_1,D_2$ are nonzero and linearly independent. Their entries lie in $\{-1,0,1\}$, so any factor of proportionality would be $\pm1$; the factor $+1$ would identify $\sigma_2$ and $\sigma_3$, and the factor $-1$ would make $P_{\sigma_3}=2P_{\sigma_1}-P_{\sigma_2}$ contain an entry $2$.
Suppose that $\Phi\equiv0$. Taking $b=\beta_{\mathrm{test}}\in\mathbb R^n$ and $a=\alpha_{\mathrm{test}}'+\mathrm i\alpha_{\mathrm{test}}''$ with $\alpha_{\mathrm{test}}',\alpha_{\mathrm{test}}''\in\mathbb R^n$ gives
\begin{equation}
\begin{aligned}
\Phi={}&(\alpha_{\mathrm{test}}'^\top D_1\beta_{\mathrm{test}})(\alpha_{\mathrm{test}}''^\top D_2\beta_{\mathrm{test}})\\
&-(\alpha_{\mathrm{test}}''^\top D_1\beta_{\mathrm{test}})(\alpha_{\mathrm{test}}'^\top D_2\beta_{\mathrm{test}}).
\end{aligned}
\end{equation}
Since this vanishes for all $\alpha_{\mathrm{test}}'$ and $\alpha_{\mathrm{test}}''$, the vectors $D_1\beta_{\mathrm{test}}$ and $D_2\beta_{\mathrm{test}}$ are linearly dependent for every real $\beta_{\mathrm{test}}$. Taking $a$ real and $b$ complex shows in the same way that $\alpha_{\mathrm{test}}^\top D_1$ and $\alpha_{\mathrm{test}}^\top D_2$ are dependent for every real $\alpha_{\mathrm{test}}$.

\emph{Rank.} On the dense open set $\mathcal U_{\mathrm{lin}}=\{\beta_{\mathrm{test}}:D_2\beta_{\mathrm{test}}\ne0\}$, write $D_1\beta_{\mathrm{test}}=\eta_{\mathrm{lin}}(\beta_{\mathrm{test}})D_2\beta_{\mathrm{test}}$. If $D_2\beta_{\mathrm{test},1}$ and $D_2\beta_{\mathrm{test},2}$ are independent, comparing coefficients in
\begin{equation}
\begin{aligned}
&\eta_{\mathrm{lin}}(\beta_{\mathrm{test},1}+t\beta_{\mathrm{test},2})(D_2\beta_{\mathrm{test},1}+tD_2\beta_{\mathrm{test},2})\\
&\quad=\eta_{\mathrm{lin}}(\beta_{\mathrm{test},1})D_2\beta_{\mathrm{test},1}\\
&\qquad+t\eta_{\mathrm{lin}}(\beta_{\mathrm{test},2})D_2\beta_{\mathrm{test},2}
\end{aligned}
\end{equation}
gives $\eta_{\mathrm{lin}}(\beta_{\mathrm{test},1})=\eta_{\mathrm{lin}}(\beta_{\mathrm{test},2})$.
If $\operatorname{rank}D_2\ge2$, any two points of $\mathcal U_{\mathrm{lin}}$ can be compared through a third point whose image lies on neither of their image lines. Hence $\eta_{\mathrm{lin}}$ is constant on $\mathcal U_{\mathrm{lin}}$ and, by density, $D_1=\eta_{\mathrm{lin}} D_2$, which contradicts independence. Exchanging the roles of $D_1$ and $D_2$ excludes $\operatorname{rank}D_1\ge2$ in the same way, so the only remaining possibility is that both ranks are one. Write $D_l=c_lr_l^\top$ with nonzero factors. Choosing $\beta_{\mathrm{test}}$ outside both row kernels forces $c_1\parallel c_2$, and choosing $\alpha_{\mathrm{test}}$ nonorthogonal to both columns forces $r_1\parallel r_2$, which again contradicts independence. Thus $\Phi\not\equiv0$. The non-generic pairs lie in the finite union of the zero sets of these polynomials, which is a proper real-algebraic set and has measure zero.
\end{proof}

\subsubsection{Persistence Under Perturbation}\label{app:semicontinuity}
\begin{proof}[Proof of \Cref{lem:semicont}]
Let $v_1,\dots,v_k$ be the vertices of $\mathcal P(a,b)$. For $k=1$ the claim holds because every polygon has at least one vertex. Otherwise, for each $v_l$ we choose a unit direction $u_l$ that exposes $v_l$ alone, and we choose $\delta>0$ below every positive score gap between $v_l$ and another correlation in direction $u_l$, and below half of every distance between two distinct vertices. By \Cref{lem:generic}, there is a generic pair within distance $\delta_{\mathrm{pert}}$ that moves each correlation by less than $\delta/3$. Every maximizer in direction $u_l$ then comes from a correlation that originally equaled $v_l$, so a vertex of the exposed face lies within $\delta/3$ of $v_l$, also when the face is an edge. These neighborhoods are disjoint, which yields at least $k$ distinct vertices.
\end{proof}

\subsection{Cycle Transitions}\label{app:cycle-transitions}
\begin{proof}[Proof of \Cref{lem:cycle}]
By (G2), a supporting line contains at most two distinct correlations, and by (G1) these belong to at most two maximizing permutations. If $\sigma$ and $\tau$ both maximize, we decompose $\sigma^{-1}\tau$ into disjoint nontrivial cycles. Switching one cycle cannot raise the projection $\Re(\bar uz)$, because $\sigma$ is optimal. The cycles move disjoint rows, so these changes sum to the change from $\sigma$ to $\tau$, which is zero, and each of them is zero. If there were two cycles, switching only one of them would give a third maximizer, which contradicts genericity. Thus the relative permutation is a single nontrivial cycle.
\end{proof}

\subsection{The Phase Function}\label{app:phase}
\Cref{lem:exchange} confines the phase~\eqref{eq:phase} to one branch of the argument, and \Cref{prop:total} gives the total of the jumps over a full turn. Combined with \Cref{lem:jump}, this proves \Cref{lem:counting} in \Cref{app:counting-bound}.

\subsubsection{Swap Optimality and the Argument Branch}\label{app:phase-exchange}
\begin{lemma}[exchange]\label{lem:exchange}
For $i<k$, $z_\sigma-z_{\sigma\circ(i\,k)}=\Delta z_{ik}(\sigma):=(a_i-a_k)(b_{\sigma(i)}-b_{\sigma(k)})$. If $\sigma\in \mathcal{M}_u$ then $\Re(\bar u\,\Delta z_{ik}(\sigma))\ge0$ for all $i<k$.
\end{lemma}

\begin{proof}[Proof of \Cref{lem:exchange}]
The two matchings differ only in rows $i$ and $k$, which gives the identity for the swap gain~\eqref{eq:swap-gain}. Optimality of $\sigma$ against the swapped matching gives the inequality.
\end{proof}

Condition (G1) forces the entries of $a$ to be distinct and the entries of $b$ to be distinct, because a repeated entry makes two swapped correlations equal. Thus $\Delta z_{ik}(\sigma)\ne0$. By \Cref{lem:exchange}, optimality confines each term of the phase to $[-\pi/2,\pi/2]$ and $\Psi(\sigma,u)$ to $[-\tfrac\pi2\tbinom n2,\tfrac\pi2\tbinom n2]$.

\subsubsection{Total Jump Over One Turn}\label{app:phase-total}
For generic inputs with $n\ge3$, we index the vertices $z_{\sigma_m}$ counterclockwise modulo $V$ and let $u_m$ be the outward unit normal of the edge $[z_{\sigma_m},z_{\sigma_{m+1}}]$. By \Cref{lem:cycle}, $\mathcal M_{u_m}=\{\sigma_m,\sigma_{m+1}\}$, while $\sigma_m$ is the unique maximizer inside the arc from $u_{m-1}$ to $u_m$. The relative permutation of $\sigma_m$ and $\sigma_{m+1}$ is one cycle of length $L_m\ge2$. For $n=2$, the polygon is a segment with two vertices. We define the jump at the $m$th edge by $\Delta\Psi_m=\Psi(\sigma_{m+1},u_m)-\Psi(\sigma_m,u_m)$.

\begin{proposition}[total jump]\label{prop:total}
For generic $(a,b)$ with $n\ge3$, the jumps satisfy
\begin{equation}
\sum_{m=1}^{V}\Delta\Psi_m=\pi n(n-1).
\end{equation}
\end{proposition}

\begin{proof}[Proof of \Cref{prop:total}]
We unwrap the angles of the normals as $\varphi_0<\cdots<\varphi_V=\varphi_0+2\pi$, with $u_m=e^{\mathrm i\varphi_m}$. On $[\varphi_{m-1},\varphi_m]$, optimality keeps each $e^{-\mathrm i\varphi}\Delta z_{ik}(\sigma_m)$ in the closed right half-plane without the origin, where its principal argument decreases at unit speed. Hence $\Psi(\sigma_m,u_m)-\Psi(\sigma_m,u_{m-1})=-\binom n2(\varphi_m-\varphi_{m-1})$. Summing over $m$ and adding the jumps gives
\begin{equation}
\begin{aligned}
&\sum_m\big[\Psi(\sigma_{m+1},u_m)-\Psi(\sigma_m,u_{m-1})\big]\\
&\qquad=\sum_m\Delta\Psi_m-2\pi\binom n2.
\end{aligned}
\end{equation}
The left side telescopes to $0$, because the final normal direction and matching agree with the initial ones. The jumps therefore balance the continuous decrease over the full turn.
\end{proof}

\input{appendix/transition_geometry.tex}

\subsection{The Jump Lemma}\label{app:jump}

We prove \Cref{lem:jump} from the cycle geometry of \Cref{sec:transition-geometry} and the results on orientation and angle change in \Cref{lem:orient,lem:braid}.

In the notation of \Cref{app:phase-total}, we fix $m$ and abbreviate $\sigma=\sigma_m$, $\tau=\sigma_{m+1}$, $u=u_m$. We write the cycle as $i_1\to\cdots\to i_L\to i_1$, with $j_r=\sigma(i_r)$, $\tau(i_r)=j_{r+1}$, and $\tau=\sigma$ off $\mathcal I_{\mathrm{cyc}}=\{i_1,\dots,i_L\}$; indices are taken modulo $L$. Let $\Gamma$ follow $b_{j_1},\dots,b_{j_L}$, and set $I=\sum_{j_0\notin\{j_r\}}\operatorname{wind}(\Gamma,b_{j_0})$ for $L\ge3$ and $I=0$ for $L=2$. We prove that $\Gamma$ is simple and counterclockwise and that $\Delta\Psi_m=\pi(L-1+2I)$; thus $I$ counts the off-cycle sites inside $\Gamma$. The proof converts the phase jump into angle changes along the paths of the reassigned points and then computes those changes from the geometry of the cycle. The case of a swap requires a separate calculation, because its two paths collide.

\paragraph{Step 1: Interpolation.}
We interpolate $B(\zeta)=(1-\zeta)(b_{\sigma(i)})_i+\zeta(b_{\tau(i)})_i$ and set $\omega_{ik}(\zeta)=\bar u(a_i-a_k)(B_i(\zeta)-B_k(\zeta))$. Optimality gives nonnegative real parts at both endpoints (\Cref{lem:exchange}), and the affine interpolation preserves this inequality along the path. If $B_i-B_k$ never vanishes, the path stays in the closed right half-plane without the origin, so its continuous change of argument equals the difference of its principal arguments. Multiplication by the fixed nonzero factor $\bar u(a_i-a_k)$ does not alter the continuous angle change of $B_i-B_k$. Thus, in the absence of collisions,
\begin{equation}\label{eq:J-as-braid}
\Delta\Psi_m=\sum_{i<k}\Delta_{\zeta\in[0,1]}\arg\big(B_i(\zeta)-B_k(\zeta)\big),
\end{equation}
where $\Delta\arg$ denotes the continuous angle change.

\paragraph{Step 2: A Power Diagram.}
We set $c_{ij}=\Re(\bar ua_ib_j)$. The optimal face of the primal assignment problem is the segment $[P_\sigma,P_\tau]$. Strict complementarity~\cite{goldman1956} supplies a feasible optimal dual $(f,g)$ whose tight entries are exactly the union of the supports of these two matchings. We write $\xi_i^{\mathrm{row}}=\overline{\bar ua_i}$ and $y_j=b_j$, so that $c_{ij}=\langle\xi_i^{\mathrm{row}},y_j\rangle$, and define
\begin{equation}
\mathcal{R}_j=\{\xi:\langle \xi,y_j\rangle-g_j\ge\langle \xi,y_k\rangle-g_k\ \forall k\}.
\end{equation}
These are the power cells of~\eqref{eq:power-distance}. At row $i$, the maximizing sites correspond exactly to the tight dual entries, so membership in a cell identifies the possible optimal partners of the row. Unchanged rows lie in the interiors of cells; changed rows lie in the relative interiors of shared edges, where precisely two sites tie. Each tied site is the unique maximizer nearby on one side of the edge. Since every site is used by an optimal matching, every cell has interior.
The diagram is dual to the regular subdivision~\cite{aurenhammer1987,deloera2010} $\mathcal T$ obtained from the lower lifting $(y_j,g_j)$. Every site is a vertex of $\mathcal T$, and shared cell edges give noncrossing edges $[y_j,y_l]$ that contain no other site. The reassignment cycle follows these edges. Thus the distinct sites of the cycle form a simple polygon for $L\ge3$, with no off-cycle site on its boundary. For $L=2$, no other site lies on the segment of the swap. Collinear sites give a one-dimensional subdivision and cannot support a longer cycle.

\paragraph{Step 3: $L=2$.}
The swap gives $\Delta z_{i_1i_2}(\tau)=-\Delta z_{i_1i_2}(\sigma)$. Set $\Delta z_e=z_\tau-z_\sigma=-\Delta z_{i_1i_2}(\sigma)$; counterclockwise traversal gives $u=R_-(\Delta z_e)/|\Delta z_e|$ and $\bar u\Delta z_{i_1i_2}(\sigma)=-\mathrm i|\Delta z_e|$. The argument of this pair therefore changes from $-\pi/2$ to $+\pi/2$, which contributes $\pi$. The two interpolants collide at $\zeta=1/2$, so this direct computation at the endpoints is needed in place of~\eqref{eq:J-as-braid}. Relative to each off-cycle site, the two moving points traverse the same segment in opposite directions without meeting that site (Step~2), so their angle changes cancel. Every remaining pair is constant. Hence $\Delta\Psi_m=\pi$.

\paragraph{Step 4: Orientation ($L\ge3$).}
We traverse the edge $[z_\sigma,z_\tau]$ counterclockwise, so $\bar u(z_\tau-z_\sigma)=\mathrm i|z_\tau-z_\sigma|$. With $e_r=y_{j_{r+1}}-y_{j_r}$, this gives
\begin{align}
0<\Im\big(\bar u(z_\tau-z_\sigma)\big)
&=\sum_{r=1}^L\Im\big(\bar ua_{i_r}(b_{j_{r+1}}-b_{j_r})\big)\\
&=\sum_{r=1}^L\xi_{i_r}^{\mathrm{row}}\times e_r=:\chi_\Gamma.
\end{align}
\begin{lemma}\label{lem:orient}
In the situation of Step~2 with $L\ge3$, $\chi_\Gamma>0$ if $\Gamma$ is counterclockwise and $\chi_\Gamma<0$ if it is clockwise. Hence $\Gamma$ is counterclockwise.
\end{lemma}
\begin{proof}[Proof of \Cref{lem:orient}]
The region $D$ bounded by $\Gamma$ is a union of faces $Q$ of the subdivision, each dual to a vertex $v_Q$ of the diagram at which all sites of $Q$ are maximizers. We relate the orientation of each edge of the subdivision to that of its dual edge, and then sum over the enclosed faces to compare the contributions of the interior edges and of the boundary.

(i) \emph{Reciprocity.} Orient $e=y_l-y_j$ with $Q$ on its left. Its dual edge leaves $v_Q$ in the direction $R_-(e)$: the two tied scores keep $\langle\xi,e\rangle$ constant, while every other site of $Q$ loses maximality because $\langle R_-(e),y_k-y_j\rangle=-e\times(y_k-y_j)<0$. The opposite direction fails this inequality. If $Q'$ lies on the right of $e$, then $v_{Q'}=v_Q+\eta_e R_-(e)$ for some $\eta_e\ge0$.

(ii) \emph{Cancellation across faces.} For counterclockwise $\Gamma$, sum $\sum_{e\in\partial Q}v_Q\times e=0$ over the faces in $D$, with each boundary traversed counterclockwise. An interior edge contributes $(v_Q-v_{Q'})\times e=-\eta_e|e|^2$. Thus the outer boundary satisfies $\sum_rv_{Q_r}\times e_r=\sum_{\text{interior }e}\eta_e|e|^2\ge0$, where $Q_r$ lies to the left of $e_r$.

(iii) \emph{Strict positivity.} By (i), the point $\xi_{i_r}^{\mathrm{row}}$, which lies in the relative interior of the diagram edge dual to $e_r$, equals $v_{Q_r}+\eta_rR_-(e_r)$ with $\eta_r>0$. Therefore $\chi_\Gamma=\sum_rv_{Q_r}\times e_r+\sum_r\eta_r|e_r|^2>0$.
For clockwise $\Gamma$, reversing its edges to $-e_r$ gives $-\chi_\Gamma>0$ by the same computation.
\end{proof}

\paragraph{Step 5: Winding ($L\ge3$).}
For each fixed off-cycle site $y_{j_0}$, the moving points together trace $\Gamma$ and contribute $2\pi\operatorname{wind}(\Gamma,y_{j_0})$. Within the cycle, the points $\gamma_r(\zeta)=(1-\zeta)y_{j_r}+\zeta y_{j_{r+1}}$ never collide: nonadjacent edges are disjoint, and a shared endpoint is reached at $\zeta=1$ by one point and left at $\zeta=0$ by the other. Thus~\eqref{eq:J-as-braid} gives
\begin{equation}
\begin{gathered}
\Delta\Psi_m=\Theta(\Gamma)+2\pi I,\qquad\\
\Theta(\Gamma):=\sum_{r<s}\Delta_{\zeta\in[0,1]}\arg\big(\gamma_r(\zeta)-\gamma_s(\zeta)\big).
\end{gathered}
\end{equation}
\Cref{lem:braid} gives $\Theta(\Gamma)=\pi(L-1)$, and since $\Gamma$ is simple and counterclockwise, $I$ is the number of enclosed off-cycle sites. This proves \Cref{lem:jump}.
\qed

\begin{lemma}[braid lemma]\label{lem:braid}
Let $\Gamma=(\gamma_1,\dots,\gamma_L)$, $L\ge3$, be a simple polygon with indices modulo $L$, and let every vertex slide linearly to the next one, $\gamma_r(\zeta)=(1-\zeta)\gamma_r+\zeta\gamma_{r+1}$ for $0\le \zeta\le1$. Then $\Theta(\Gamma)=\pi(L-1)$ if $\Gamma$ is counterclockwise and $-\pi(L-1)$ if clockwise.
\end{lemma}

\begin{proof}[Proof of \Cref{lem:braid}]
As in Step~5, points that slide along distinct edges of a simple polygon do not collide. Each summand of $\Theta$ is the principal angle from $\gamma_r-\gamma_s$ to $\gamma_{r+1}-\gamma_{s+1}$, because the segment that connects them avoids $0$. Thus $\Theta$ is continuous on the space $\mathcal S_L^+$ of simple counterclockwise $L$-gons with labeled vertices.
The cyclic shift of the vertices also fixes the value of $\Theta$ modulo a full turn:
\begin{equation}
\begin{aligned}
e^{\mathrm i\Theta}
&=\prod_{r<s}\frac{\gamma_{r+1}-\gamma_{s+1}}{\gamma_r-\gamma_s}
\Big/\Big|\prod_{r<s}\frac{\gamma_{r+1}-\gamma_{s+1}}{\gamma_r-\gamma_s}\Big|
\\
&=\operatorname{sgn}(c)=(-1)^{L-1},
\end{aligned}
\end{equation}
because permuting the points by the cyclic shift $c$ multiplies $\prod_{r<s}(\gamma_r-\gamma_s)$ by $\operatorname{sgn}(c)$. Thus $\Theta\in\pi(L-1)+2\pi\mathbb Z$, and $\Theta$ is locally constant on $\mathcal S_L^+$.
It remains to show that $\mathcal S_L^+$ is path-connected. By the convexification theorem of Connelly et~al.~\cite{connelly2003}, each polygon can be convexified continuously through simple polygons, with its orientation preserved. We translate the resulting convex polygon so that the origin lies in its interior. Its vertices are then $\rho_re^{\mathrm i\alpha_r}$ with increasing angles and consecutive gaps in $(0,\pi)$. Interpolating the radii to $1$ and the gaps to $2\pi/L$ keeps the gaps in $(0,\pi)$, so the polygon remains star-shaped about the origin, simple and counterclockwise. Finally we rotate the resulting regular polygon to a fixed reference. Continuity and local constancy then make the angle sum the same for every polygon in this space.
For the regular polygon $\gamma_r=e^{2\pi\mathrm ir/L}$ we have $\gamma_r(\zeta)-\gamma_s(\zeta)=(\gamma_r-\gamma_s)\big((1-\zeta)+\zeta e^{2\pi\mathrm i/L}\big)$, so each pair turns by $2\pi/L$ and $\Theta=\binom L2\frac{2\pi}L=\pi(L-1)$. Reversing the orientation negates $\Theta$.
\end{proof}

\subsection{Vertex Bound and Tightness}\label{app:main}
We combine \Cref{lem:jump} (\Cref{app:jump}) and \Cref{prop:total} (\Cref{app:phase-total}) to prove \Cref{lem:counting}. \Cref{lem:semicont} (\Cref{app:semicontinuity}) extends the resulting bound to every input, and \Cref{prop:tight} establishes the sharpness claimed in \Cref{thm:main}.

\subsubsection{Counting Identity and Vertex Bound}\label{app:counting-bound}
\begin{proof}[Proof of \Cref{thm:main,lem:counting}]
For generic inputs with $n\ge3$, substituting~\eqref{eq:jump} into \Cref{prop:total} and dividing by $\pi$ gives~\eqref{eq:counting}. Each term is at least one, so $V\le n(n-1)$. \Cref{lem:semicont} extends the bound to all inputs, and for $n=2$ there are at most two correlations. One representative matching per vertex covers all directions. Under genericity, the maximizing matching is unique between edge normals and changes once at each normal.
\end{proof}

\subsubsection{Tightness}\label{app:tightness}
\begin{proof}[Proof of \Cref{prop:tight}]
We relabel so that $a_1<\cdots<a_n$. Abel's identity \eqref{eq:abel-summation} expresses each correlation through the positive increments $a_k-a_{k-1}$ and nested column sums. In a direction $u$, assigning the largest row weights to the largest projections $\Re(\bar u b_j)$ maximizes every sum simultaneously. Equality of the support functions shows that $\mathcal P$ is a translate of the Minkowski sum of positive multiples of $Q_h=\operatorname{conv}\{\sum_{j\in S}b_j:|S|=h\}$, $1\le h<n$.
It remains to count the changes in the order of the projections. For $n\ge3$, each pair exchanges its order twice, at the directions with $\Re(\bar u(b_j-b_l))=0$. Because the differences are non-parallel, all $2\binom n2$ transition directions are distinct. The tied columns are adjacent in the sorted order, so exactly one maximizing top-$h$ subset changes, which exposes one nondegenerate edge of one $Q_h$. Every other summand has a unique supporting vertex. The positive coefficient preserves this edge in the Minkowski sum, and no further transitions occur. Thus the polygon has exactly $n(n-1)$ edges and vertices. For $n=2$, the correlations differ by $(a_1-a_2)(b_1-b_2)\ne0$, which gives a segment with two vertices.
\end{proof}

\subsection{Correctness of the Rubix Algorithm}\label{app:algorithms}
We prove the invariant on boundary arcs and the query count of \Cref{alg:rubix,prop:es}, and combine them with \Cref{lem:reduction,thm:main} to prove \Cref{cor:pw}.

\subsubsection{Polygon Reconstruction}\label{app:reconstruction}
\begin{proof}[Proof of \Cref{prop:es}]
If the two initial outputs have the same correlation, both the real extrema and the opposite secondary imaginary extrema coincide. Thus $\mathcal P$ is a point and two calls suffice. Otherwise, the two initial chords represent complementary boundary arcs. A positive support gap splits its arc at a vertex strictly between the endpoints, and a zero gap confirms an edge. Pending arcs have disjoint interiors and exclude their endpoints, so no vertex is discovered twice: each query adds a vertex or confirms an edge. Each initial arc generates a binary tree with the splits as internal nodes and the confirmed edges as leaves. Together these trees have $V-2$ internal nodes and $V$ leaves, which gives $2+(V-2)+V=2V$ calls, including the initialization. For a segment, the two directed copies of the chord give four calls. \Cref{thm:main} bounds $V$, regardless of how many permutations represent a vertex.
\end{proof}

\subsubsection{Exact Alignment and Complexity}\label{app:pw-complexity}
\begin{proof}[Proof of \Cref{cor:pw}]
\Cref{prop:es} gives exactly $2V$ oracle calls, for points and segments as well. Lexicographic assignment uses ordered pairs of costs and cubic componentwise arithmetic (\Cref{sec:polygon-reconstruction}). We apply this to the polygons of \Cref{lem:reduction} for $\mathrm{SO}(2)$ or $\mathrm O(2)$. By \Cref{thm:main}, the total is $\mathcal O(n^5)$ arithmetic operations; bit costs are separate. For input coordinates with rational real and imaginary parts, the oracle costs and the edge normals are rational, which permits exact verification of the returned polygon (\Cref{prop:optimality,app:verification}).
\end{proof}

\input{appendix/verification.tex}

\subsection{Fixed Points of Alternation}\label{app:am}

\paragraph{Update conventions.}
The claims of \Cref{prop:am} on fixed points and finite convergence use the tie convention of \Cref{sec:am-comparison}. We initialize $\mathrm{SO}(2)$ AM~\cite{adamo2025} with an assignment that returns a vertex. Subsequent assignments retain the current matching if it is optimal and otherwise return a vertex maximizer. For nonzero correlation, $\sigma_{t+1}\in\mathcal M_{u_t}$ with $u_t=z_{\sigma_t}/|z_{\sigma_t}|$.

\begin{proof}[Proof of \Cref{prop:am}]
We write $z=z_{\sigma_t}$, $z'=z_{\sigma_{t+1}}$, and $u=z/|z|$ for $z\ne0$. Optimality of the assignment gives
\begin{equation}
|z'|\ge\Re(\bar u z')\ge\Re(\bar u z)=|z|.
\end{equation}
Retention makes the second inequality strict whenever the matching changes, even if several matchings share a vertex. A matching is retained exactly when $\Re(\bar z(w-z))\le0$ for all $w\in\mathcal P$. A vertex of largest modulus satisfies this condition because $\Re(\bar zw)\le|z||w|\le|z|^2$. A nonzero initial vertex excludes zero thereafter and allows at most $V-1\le n(n-1)-1$ strict increases. An unchanged matching retains its correlation and direction and therefore remains fixed. Including the final check, there are at most $V$ assignment updates after the initialization.
\end{proof}

\paragraph{Ties and zero correlations.}
An arbitrary tie rule need not fix a stationary vertex: on the segment $[1,1+\mathrm i]$, both endpoints maximize the real part but their moduli differ. The example $a=(1,0)$, $b=(1,1+\mathrm i)$ satisfies (G1) and (G2), so genericity does not replace the retention rule. Numerical implementations may use other tie rules.

At zero correlation, every angle solves the Procrustes problem~\cite{schonemann1966}, and a direction must be chosen explicitly. The modulus remains nondecreasing. A different vertex has positive modulus; a retained zero vertex is fixed for the chosen direction, which stationarity alone does not determine. If $\mathcal P=\{0\}$, every rotation has the same cost and retention fixes the matching.

%% file: appendix/transition_geometry.tex
\subsection{Geometry of a Transition}\label{sec:transition-geometry}

To prove \Cref{lem:jump}, we place the reassignment cycle of \Cref{lem:cycle} on a planar subdivision, on which a cycle of length at least three is a simple polygon. This subsection collects the facts on duality and on angles that Steps~2, 4 and~5 of \Cref{app:jump} use.

\paragraph{Assignment duality.} The dual of the assignment problem minimizes $\sum_i f_i+\sum_j g_j$ subject to $f_i+g_j\ge c_{ij}$. These inequalities bound the score of every matching, and a matching that consists of \emph{tight} pairs, for which equality holds, attains the bound. Strict complementarity~\cite{goldman1956} supplies an optimal dual whose tight pairs are exactly the pairs used by at least one optimal matching. A single dual solution therefore identifies all exchanges that can occur at a transition.

\paragraph{Power diagrams and their duals.} For distinct sites $b_j\in\mathbb C$ and real numbers $g_j$, the \emph{power cell}~\cite{aurenhammer1987} of $b_j$ consists of the points at which this site minimizes the squared distance minus its weight $|b_j|^2-2g_j$; a point at which several sites tie belongs to all of their cells. Expanding the square gives
\begin{equation}\label{eq:power-distance}
|\xi-b_j|^2-(|b_j|^2-2g_j)=|\xi|^2-2(\langle\xi,b_j\rangle-g_j).
\end{equation}
Thus a power cell is the region in which its site maximizes an affine score. If every cell has interior, lifting each $b_j$ to height $g_j$ and projecting the lower faces of the convex hull gives a \emph{regular subdivision}~\cite{aurenhammer1987,deloera2010} of the hull of the sites, in which every site is a vertex. The two planar structures are dual: two cells that share an edge give an edge of the subdivision that joins their sites. These edges neither cross nor contain another site in their interiors. Collinear sites give segments.

\paragraph{Angles along curves.} A \emph{simple} polygon has no self-intersections on its boundary, and traversing it counterclockwise keeps its interior on the left. For a path $\gamma(\zeta)$ that avoids the origin, $\Delta\arg\gamma$ is the continuous change of its angle, complete turns included. For a closed curve $\Gamma$ that avoids a point $z$, $\operatorname{wind}(\Gamma,z)=\Delta\arg(\Gamma-z)/(2\pi)$. A simple counterclockwise polygon has winding number one about the points inside it and zero about the points outside.

%% file: appendix/verification.tex
\subsection{Exact Verification}\label{app:verification}

We prove \Cref{prop:optimality} by verifying the output of \Cref{alg:rubix} with integer assignment certificates and exact orientation tests, as outlined in \Cref{sec:optimality}.

\paragraph{Lexicographic assignment.}
The two objectives of the oracle can be encoded in one integer assignment problem without perturbing the primary order. For a nonzero rational direction $u$, we clear denominators in $\Re(\bar u a_i b_j)$ and in $\Re(\overline{\mathrm i u}a_i b_j)$ separately, which gives integer matrices $D^{\mathrm{pri}}$ and $D^{\mathrm{sec}}$; positive scaling preserves both orders. The secondary assignment sums lie in $[-nD_{\max},nD_{\max}]$, where $D_{\max}=\max_{i,j}|D^{\mathrm{sec}}_{ij}|$, and distinct primary sums differ by at least one. Maximizing $D^{\mathrm{lex}}_{ij}=N_{\mathrm{lex}}D^{\mathrm{pri}}_{ij}+D^{\mathrm{sec}}_{ij}$ with $N_{\mathrm{lex}}=2nD_{\max}+1$ therefore gives the lexicographic optimum, because every primary gain outweighs any secondary loss. To certify the result, we check that $f_i+g_j\ge D^{\mathrm{lex}}_{ij}$ for all $i,j$, with equality on the matching $\sigma$; equal primal and dual objectives prove optimality. To obtain $h_{\mathcal P}(u)$ for the checks of supporting lines, we evaluate that matching with the original primary scores. The integer dual thus certifies lexicographic optimality, and the evaluation with the primary scores supplies the support value.

\begin{proof}[Proof of \Cref{prop:optimality}]
Recomputing the correlations of the returned matchings gives a hull $\widehat{\mathcal P}\subseteq\mathcal P(a,b)$. Exact signs of $(z'-z)\times(w-z)$ determine its distinct extreme points and verify their reported order up to a cyclic shift. For each edge of a two-dimensional $\widehat{\mathcal P}$, we check that the assignment optimum in the outward normal direction equals the support value of the edge. The corresponding half-planes contain $\mathcal P(a,b)$ and intersect in $\widehat{\mathcal P}$, which proves the reverse inclusion.
The lower-dimensional cases require checks along the hull as well as across it. For a segment, the two opposite normal directions confine $\mathcal P(a,b)$ to its line, and the initial lexicographic queries bound its endpoints, through the secondary extrema if the line is vertical. If the initial points coincide, equality of the primary extrema gives a vertical line, and equality of the opposite secondary extrema gives a single point. Thus $\widehat{\mathcal P}=\mathcal P(a,b)$ in all cases, and convexity reduces the maximum to an exact comparison of the squared moduli of the vertices.
\end{proof}

\paragraph{Exact PW output.}
When translation is allowed, we center the coordinates in rational arithmetic and construct the polygons of \Cref{lem:reduction}. Their maximum squared modulus $r_\star^2$, with $r_\star\ge0$, gives the exact PW value through~\eqref{eq:pw-polygon}. For a maximizing $z\ne0$, the rotation is multiplication by $\bar z/r_\star$, composed with conjugation if the reflected polygon attains the maximum. A zero maximum permits every rotation. Bit costs depend on the size of the input.

%% file: appendix/landscape.tex
\section{Extensions and Limits}\label{app:landscape}

\input{appendix/weights.tex}

\subsection{Comparison of Transformation Classes}\label{app:transformation-classes}

\Cref{tab:landscape} summarizes what is known for different classes of transformations and weights. It combines the equal-weight bounds of \Cref{thm:main,cor:pw} with the extension to rational weights in \Cref{cor:weights}, in which the replication argument supplies representative transport plans by aggregating matchings of the copied points. The resulting algorithm is \emph{pseudo-polynomial} in the common denominator $N$, whose value can be exponential in its bit length. Uniform weights on $n$ and $m$ points use $N=\operatorname{lcm}(n,m)\le nm$, which is polynomial in the sizes of the point sets. \Cref{app:matrix-correlations} derives the formulations with matrix correlations for rotations in higher dimension and for unrestricted linear maps.

\begin{table*}[tp]
\caption{\textbf{Bounds on representative matchings or transport plans} and arithmetic complexity of global alignment.}
\label{tab:landscape}

\begin{center}
\input{tables/landscape.tex}

\end{center}
\end{table*}

\begin{remark}[diagonal weights]\label{rem:special}
The bound uses a common phase $e^{\mathrm i\theta}$ across all products $a_ib_j$. For diagonal weights, let $x_i=(x_i^{(1)},x_i^{(2)})$, $y_j=(y_j^{(1)},y_j^{(2)})$ and $D_{\eta_{\mathrm{diag}}}=\operatorname{diag}(\cos\eta_{\mathrm{diag}},\sin\eta_{\mathrm{diag}})$. Assignment with scores $x_i^\top D_{\eta_{\mathrm{diag}}} y_j$ is encoded by the convex hull of $\big(\sum_ix_i^{(1)}y_{\sigma(i)}^{(1)},\sum_ix_i^{(2)}y_{\sigma(i)}^{(2)}\big)$, where $x^{(1)},y^{(1)},x^{(2)},y^{(2)}$ are real vectors. The integer example in \Cref{app:diagonal-counterexample} has $n=4$ and $16>n(n-1)=12$ vertices, exceeding the quadratic bound.
\end{remark}

\input{appendix/linear_extensions.tex}

%% file: appendix/weights.tex
\subsection{Rational Weights}\label{app:weights}
We prove \Cref{cor:weights} by showing that replication preserves the set of feasible correlations, so that \Cref{thm:main} applies.

\begin{proof}[Proof of \Cref{cor:weights}]
For $N=1$ the set is a single point. Otherwise we replace each atom $x_i$ of positive weight by $N\alpha_i$ copies and each $y_j$ by $N\beta_j$ copies, which gives $\tilde x,\tilde y\in\mathbb C^N$. Aggregating a doubly stochastic matrix over the blocks of copies gives $N\pi$ for some $\pi\in\Pi(\alpha,\beta)$. Conversely, take any such transport plan and set $P_{IJ}=\pi_{ij}/(N\alpha_i\beta_j)$ for the copies $I$ of $x_i$ and $J$ of $y_j$. Its row sums are
\begin{equation}
\sum_J P_{IJ}=\sum_j\frac{N\beta_j\pi_{ij}}{N\alpha_i\beta_j}
=\frac{\sum_j\pi_{ij}}{\alpha_i}=1,
\end{equation}
and its column sums equal one in the same way. Thus the lifted matrix is doubly stochastic. The block $(i,j)$ sums to $N\pi_{ij}$, so
\begin{equation}
\frac1N\sum_{I,J}P_{IJ}\overline{\tilde x_I}\tilde y_J
=\sum_{i,j}\pi_{ij}\bar x_i y_j.
\end{equation}
The weighted polygon is consequently $\frac1N\mathcal P(\overline{\tilde x},\tilde y)$. Scaling does not change the number of vertices, and \Cref{thm:main} applies despite the repeated points.
\end{proof}

%% file: tables/landscape.tex
\begin{tabular}{@{}p{4cm}p{5.7cm}p{6.4cm}@{}}
\toprule
\raggedright Setting & \raggedright Representative matchings or couplings & \raggedright Algorithm and reference\tabularnewline
\midrule
\raggedright $d=1$, reflections $\{\pm1\}$ & \raggedright Two: sorted and reversed & \raggedright $\mathcal{O}(n\log n)$ by sorting\tabularnewline
\raggedright $d=2$, rotations, uniform weights, $n\ge2$ & \raggedright At most $n(n-1)$; tight & \raggedright $\mathcal{O}(n^5)$; at most $2n(n-1)$ assignments (Theorem~\ref{thm:main})\tabularnewline
\raggedright $d=2$, rotations and reflections, $n\ge2$ & \raggedright At most $2n(n-1)$ & \raggedright $\mathcal{O}(n^5)$ (Corollary~\ref{cor:pw})\tabularnewline
\raggedright $d=2$, rotations, uniform weights, $n$ versus $m$ points & \raggedright At most $\max\{1,N_{\mathrm{uni}}(N_{\mathrm{uni}}-1)\}$, where $N_{\mathrm{uni}}=\operatorname{lcm}(n,m)$ & \raggedright Polynomial by duplication (Corollary~\ref{cor:weights})\tabularnewline
\raggedright $d=2$, rotations, rational weights with denominator $N$ & \raggedright At most $\max\{1,N(N-1)\}$; can exceed $n(n-1)$ & \raggedright Pseudo-polynomial by duplication (Corollary~\ref{cor:weights})\tabularnewline
\raggedright $d=2$, linear maps; planar GW with inner-product costs & \raggedright Vertices of a polytope of dimension at most four; can exceed $n(n-1)$ & \raggedright Conjectured superpolynomial vertex count; no polynomial bound (\Cref{app:matrix-conjecture})\tabularnewline
\raggedright Fixed $d\ge3$, rotations & \raggedright Vertices of the matrix correlation polytope give candidates & \raggedright Conjectured superpolynomial vertex count; no polynomial bound (\Cref{app:matrix-conjecture})\tabularnewline
\raggedright Affine one-parameter assignment costs & \raggedright $n^{\Omega(\log n)}$ & \raggedright Hrube\v{s} and Yehudayoff~\cite{hrubes2023}\tabularnewline
\bottomrule
\end{tabular}

%% file: appendix/linear_extensions.tex
\subsection{Matrix-Correlation Formulations}\label{app:matrix-correlations}
The entries of \Cref{tab:landscape} for higher dimensions and for linear maps use a matrix correlation in place of the complex scalar of \Cref{lem:reduction}. We derive the corresponding convex scores over a hull of correlations, and then the related reduction for GW with inner-product costs.
For equally weighted point sets, write $X=[x_1\ \cdots\ x_n]$, $Y=[y_1\ \cdots\ y_n]\in\mathbb R^{d\times n}$, $\mathcal C=\conv\{C_\sigma:\sigma\in S_n\}$ and
\begin{equation}
C_\sigma=\sum_i x_i y_{\sigma(i)}^\top,\qquad \Sigma_Y=YY^\top.
\end{equation}
For $G=\mathrm O(d)$ or $\mathrm{SO}(d)$,
\begin{equation}
\begin{aligned}
&\min_{R\in G}\sum_i\|x_i-Ry_{\sigma(i)}\|^2\\
&\quad=\sum_i\|x_i\|^2+\sum_i\|y_i\|^2
-2\max_{R\in G}\operatorname{tr}(R^\top C_\sigma).
\end{aligned}
\end{equation}
The maximized correlation score is convex in $C_\sigma$, so a vertex of $\mathcal C$ attains its maximum and supplies a representative matching in \Cref{tab:landscape}. The reduction to vertices therefore applies in higher dimensions as well, but bounding the number of these vertices for $d\ge3$ remains open (\Cref{sec:discussion}).

Eliminating an unrestricted linear map $A\in\mathbb R^{d\times d}$ gives a different convex score on the same hull:
\begin{equation}
\min_A\sum_i\|x_i-Ay_{\sigma(i)}\|^2
=\sum_i\|x_i\|^2-\operatorname{tr}(C_\sigma \Sigma_Y^\dagger C_\sigma^\top).
\end{equation}
The Moore--Penrose inverse $\Sigma_Y^\dagger$ replaces each positive eigenvalue by its reciprocal and leaves the zero eigenvalues unchanged. If $v\in\ker \Sigma_Y$, then $0=\sum_j(y_j^\top v)^2$, hence $C_\sigma v=0$ and $C_\sigma=C_\sigma \Sigma_Y^\dagger \Sigma_Y$. Setting $A_\star=C_\sigma \Sigma_Y^\dagger$ and completing the square gives
\begin{equation}
\begin{aligned}
&\sum_i\|x_i-Ay_{\sigma(i)}\|^2-\sum_i\|x_i-A_\star y_{\sigma(i)}\|^2\\
&\quad=\operatorname{tr}\big((A-A_\star)\Sigma_Y(A-A_\star)^\top\big)\ge0.
\end{aligned}
\end{equation}
Thus the minimum is attained also when $\Sigma_Y$ is singular. The subtracted score $\|C_\sigma(\Sigma_Y^\dagger)^{1/2}\|_F^2$, with the positive semidefinite square root, is convex and is maximized at a vertex.

A related matrix reduction holds for GW~\cite{alvarezmelis2018gw} with inner-product costs and squared loss. A nonnegative coupling $\pi$ with row and column sums $1/n$ gives
\begin{equation}
\sum_{i,k,j,l}\big(\langle x_i,x_k\rangle-\langle y_j,y_l\rangle\big)^2
\pi_{ij}\pi_{kl}
=c_{\mathrm{GW}}-2\|X\pi Y^\top\|_F^2,
\end{equation}
where $c_{\mathrm{GW}}=n^{-2}(\|X^\top X\|_F^2+\|Y^\top Y\|_F^2)$ is independent of $\pi$.
By the Birkhoff decomposition~\cite{brualdi1977}, $X\pi Y^\top$ ranges over $\mathcal C/n$ as $\pi$ varies. Convexity therefore gives an optimal permutation coupling at a vertex, with objective $c_{\mathrm{GW}}-2\|C_\sigma\|_F^2/n^2$.

For the planar example of \Cref{rem:special}, retaining the diagonal entries maps $\mathcal C$ onto a polygon with $16$ vertices. Hence $\mathcal C$ has at least $16>n(n-1)$ vertices. The planar full hull $\mathcal C\subset\mathbb R^{2\times2}$, of dimension at most four, also falls within the conjecture below. This calculation uses inner-product costs, whereas the classification experiment uses squared-distance costs.

\subsubsection{Conjectured Vertex Growth at Fixed Dimension}\label{app:matrix-conjecture}
The scalar polygon of \Cref{thm:main} retains only the two correlation coordinates that planar rotations need. The full matrix hull $\mathcal C$ retains all $d^2$ entries and can have vertices that are lost under that planar projection. Let $V_d(n)$ denote the largest number of vertices of $\mathcal C$ over $X,Y\in\mathbb R^{d\times n}$. We conjecture that, for each fixed $d\ge2$, $V_d(n)$ grows superpolynomially: $V_d(n)$ is not $\mathcal O(n^k)$ for any fixed $k$. In other words, no polynomial in $n$ bounds the number of vertices of the full hull uniformly over the inputs, and this would hold in particular at each fixed $d\ge3$, because padding planar inputs with zero coordinates preserves their matrix-correlation hull and its number of vertices. The conjecture concerns the full hull; the number of vertices that are needed for physical rotation queries alone is a separate question. The $16$-vertex example refutes only the particular bound $n(n-1)$ and does not establish the conjectured asymptotic growth.

%% file: appendix/optimality.tex
\section{Exact Arithmetic Examples}\label{app:optimality}\label{app:exact-examples}

\subsection{Exact Reconstruction and Verification}\label{app:arithmetic-examples}
The following examples apply the construction of \Cref{prop:tight} and the verification procedure of \Cref{prop:optimality}, whose proofs are in \Cref{app:tightness,app:verification}. They show that a hull computed in floating point can miss vertices that exact arithmetic distinguishes. The tight instance uses $a_j=j-1$ and $b_j=2^{j-1}+\mathrm i4^{j-1}$ for $j=1,\ldots,20$. Rubix in exact arithmetic computes all $380$ vertices with $760=2V$ assignments, whereas double precision with relative tolerance $10^{-10}$ gives 56 vertices. For the thin rational instance $a=(0,1,3)$ and $b=(0,1,2+\mathrm i10^{-60})$, exact computation gives six vertices, against two in double precision.

\subsection{Limits of the Planar Bound}\label{app:planar-counterexamples}
The following two instances show that both the structure of rotations and the equal weights are needed in \Cref{thm:main}. We compute their numbers of vertices by exhaustive enumeration and integer-arithmetic hulls.

\subsubsection{Diagonal Weighting}\label{app:diagonal-counterexample}
Set $n=4$. The first coordinates are $x^{(1)}=(38,-28,24,-2)$ and $y^{(1)}=(9,17,-52,77)$, and the second coordinates are $x^{(2)}=(14,-24,-24,-18)$ and $y^{(2)}=(-46,16,-13,-10)$. The $24$ points $(\sum_ix_i^{(1)}y_{\sigma(i)}^{(1)},\sum_ix_i^{(2)}y_{\sigma(i)}^{(2)})$ have 16 hull vertices, more than $n(n-1)=12$. This is the instance used in \Cref{rem:special} to distinguish diagonal scores from rotation correlations.

\subsubsection{Unequal Weights}\label{app:weighted-counterexample}
For $n=m=3$, we take $a=(-1-6\mathrm i,-13+12\mathrm i,16+31\mathrm i)$ and $b=(4-\mathrm i,-1,-\mathrm i)$, with weights $\alpha=(1,2,4)/7$ and $\beta=(2,2,3)/7$. The polygon $\{\sum_{ij}\pi_{ij}a_i b_j:\pi\in\Pi(\alpha,\beta)\}$ has 8 vertices, which exceeds $n(n-1)=6$; it corresponds to the PW point sets $x=\bar a$ and $y=b$. Duplication (\Cref{cor:weights}) turns this example into an equal-weight instance with $7!$ matchings, and the applicable bound uses the common denominator, as proved in \Cref{app:weights}.

%% file: appendix/spatial.tex
\section{Proofs for the Spatial Bounds}\label{app:spatial}

This appendix proves the results of \Cref{sec:spatial} in exact arithmetic; \Cref{app:search-invariants} combines the regional bounds into the guarantees of the search. Computations of assignments, eigenvalues, roots and trigonometric functions in floating point require separate certification. \Cref{app:spatial-algorithms} gives the search procedures and their interruption rule.

\subsection{Conditional Gravity Alignment}\label{app:conditional-gravity}

Leveling changes the coordinate frame and preserves every residual. We normalize each nonzero signed gravity direction $g$, choose a unit vector $b_g\perp g$, and form the matrix $A$ with rows $b_g^\top,(g\mathbin\times b_g)^\top,g^\top$. Then $A\in\mathrm{SO}(3)$ and $Ag=e_3$; a different choice of $b_g$ only shifts the origin of the yaw. Orthogonal invariance gives
\begin{equation}\label{eq:leveling-residual-identity}
 \|x_i-Ry_j-t\|=\|A_xx_i-R_z(\theta)A_yy_j-A_xt\|
\end{equation}
for $R=A_x^\top R_z(\theta)A_y$. Conversely, $Rg_y=g_x$ makes $A_xRA_y^\top$ fix $e_3$, so it is a yaw rotation. Leveling therefore preserves the constrained problem and the physical translation.

At a fixed leveled translation, we put $x_i^t=x_i-t$ and expand the squared distance to obtain \eqref{eq:gravity-edge}. For a feasible matching, define $\boldsymbol w_M=(\sum_Ma_{ij},\sum_Mb_{ij})$ and let $c_M$ be $\sum_Mc_{ij}^0$ plus any penalty for unmatched points. Then $g_M(u)=c_M-2\langle \boldsymbol w_M,u\rangle$, and $F=\min_Mg_M$ is concave on $\mathbb R^2$ for all three matching models, including the constant empty matching. To evaluate exterior queries, we keep the original squared norms of the points in this affine expansion. Squaring the residual of a nonorthogonal map instead would introduce quadratic terms that the expansion does not contain.

The partial models use standard assignment solvers after dummy points are added. For prescribed $k$, we augment the problem to order $n+m-k$, with the real costs in the upper left block, zero costs in the off-diagonal dummy blocks, and forbidden entries in the bottom right block. Every finite perfect assignment then uses exactly $k$ real edges. Penalized matching adds $\lambda(n+m)$, subtracts $2\lambda$ per real edge, and gives each point of the smaller set its own zero-cost dummy, which uses $\min(n,m)$ rows and $n+m$ columns. Both reductions allow the negative costs of exterior queries.

\subsubsection{Yaw interpolation}\label{app:arc-interpolation}

\begin{proof}[Proof of \Cref{lem:arc-interpolation}]
In coordinates centered on the arc, the vertices of the triangle are $(c_h,-s_h)$, $(c_h,s_h)$ and $(1/c_h,0)$, with $c_h=\cos h_\theta>0$ and $s_h=\sin h_\theta>0$. Its chord and the two tangents at the endpoints enclose the arc. For the oracle values $f_-,f_+,f_A$ and $\bar f=(f_-+f_+)/2$, the interpolant is
\begin{equation}\label{eq:arc-interpolant-coefficients}
 \begin{gathered}
 \ell(u(\mu+\eta))=\alpha\cos\eta+\beta\sin\eta+\gamma,\\
 \alpha=\frac{f_A-\bar f}{1/c_h-c_h},\qquad
 \beta=\frac{f_+-f_-}{2s_h},\\
 \gamma=\bar f-\alpha c_h.
 \end{gathered}
\end{equation}
To minimize this sinusoid, we evaluate both endpoints and the angle $\operatorname{atan2}(\beta,\alpha)+\pi$ modulo $2\pi$ when it lies in $[-h_\theta,h_\theta]$; if $\alpha=\beta=0$, the value is the constant $\gamma$. The same minimization fits each returned matching. Concavity gives $F\ge\ell$ throughout the triangle, and the convex interpolation weights give $\ell\ge\min\{f_-,f_+,f_A\}$. These two inequalities prove \eqref{eq:arc-bound-chain}. If one matching $M$ attains all three queries, then $g_M=\ell$ and $F\le g_M$, hence equality. The oracle may return different optimizers under ties even when a common supporting matching exists; testing a retained witness at all three vertices verifies common support without requiring uniqueness.
\end{proof}

For comparison, taking the minimum of each edge sinusoid separately before solving the assignment allows a different yaw for each edge. This bound and the interpolant do not dominate one another uniformly, and their maximum is valid. The smallest of the queried vertex values is never stronger than the interpolant. Coupled assignment queries and concave underestimators also appear in earlier registration methods~\cite{lian2014concave,lian2026dcreg}.

\subsubsection{Conditional-yaw accuracy}\label{app:gravity-accuracy}

\begin{proof}[Proof of \Cref{prop:gravity-accuracy}]
For an arc of half-width $h_\theta\le\pi/3$, let $v$ be a vertex with the smallest queried value and let $M$ attain it. The minimum $L_I$ of the interpolant over the arc is at least $F(v)$. If $v$ is an endpoint, its matching is already feasible there. Otherwise $v=\sec(h_\theta)u_0$, where $u_0$ is the midpoint of the arc on the circle, and $g_M(u_0)\le F(v)+\kappa_\theta(\sec h_\theta-1)$. Fitting the returned matchings therefore gives a feasible value $U_I$ with $U_I-L_I\le \kappa_\theta(\sec h_\theta-1)\le \kappa_\theta h_\theta^2$, since $\cos h_\theta\ge1/2$. For $\kappa_\theta>0$, the depth $d_{\rm tree}=\max\{0,\lceil\tfrac12\log_2(\kappa_\theta(\pi/3)^2/\varepsilon)\rceil\}$ suffices. A fully evaluated cell at that depth has $L_I\ge U-\varepsilon$ for the global incumbent cost $U$, so the best-bound search terminates before it subdivides the cell. The three binary trees up to depth $d_{\rm tree}$ contain $3(2^{d_{\rm tree}+1}-1)$ cells, each of which requires three queries. If $\kappa_\theta=0$, the objective is constant in the yaw and the initial evaluations suffice.
\end{proof}
If the horizontal norms of $x_i^t=x_i-t$ and $y_j$ are at most $r_{xy}$ and the matchings use at most $k_{\max}$ edges, then $\kappa_\theta\le2k_{\max}r_{xy}^2$; complete matching uses the centered points. The query count measures arithmetic work at a specified accuracy and numerical scale. It implies neither polynomial bit complexity nor termination at zero tolerance. The weaker vertex bound has the same bound on its query count. No count follows for the joint search over translation or for unrestricted rotation.

\subsubsection{First- and second-order bounds}\label{app:bound-order}

\begin{proof}[Proof of \Cref{prop:bound-order}]
(i) For an arc of half-width $h_\theta\le\pi/3$, the smallest of the three triangle values is attained at an endpoint or at the apex $\sec(h_\theta)u(\mu)$. Fitting the matching returned there over the arc gives a feasible value that exceeds the minimum of the interpolant by at most $\kappa_\theta(\sec h_\theta-1)\le \kappa_\theta h_\theta^2$, as in the proof of \Cref{prop:gravity-accuracy}; hence $F^*_A-L_A\le \kappa_\theta h_\theta^2$. If one matching $M$ attains all three values, then $F\le g_M=\ell$ on the triangle and $F\ge\ell$ by \Cref{lem:arc-interpolation}, so $L_A=\min_A g_M=F^*_A$.

(ii) We fix the translation at zero in both examples. For the radius bound, take one source point $x=(1,0,0)$ and one target $y=(2,0,0)$. Then $F(\theta)=5-4\cos\theta$, with minimum $1$ at $\theta=0$. On an arc centered at $\mu$, the radius bound is $(r_{\rm res}(\mu)-2\sin(h_\theta/2))_+^2$ with $r_{\rm res}(\mu)=(5-4\cos\mu)^{1/2}\ge1$. For the arc centered at the minimizer, $F^*_A=1$ and the gap is $1-(1-2\sin(h_\theta/2))^2=4\sin(h_\theta/2)-4\sin^2(h_\theta/2)\ge2\sin(h_\theta/2)\ge2h_\theta/\pi$, since $\sin(h_\theta/2)\le1/2$ for $h_\theta\le\pi/3$. For the independent-edge bound, take the sources $(\pm1,0,0)$ and the targets $(\pm\cos a,\sin a,0)$ with $\pi/3\le a<\pi/2$, matched by sign, so that $h_\theta\le a$ on every admissible arc. The two edges prefer the yaws $a$ and $-a$. The sign matching costs $4-4\cos a\cos\theta$ and the crossed matching $4+4\cos a\cos\theta$, so $F=4-4\cos a\,|\cos\theta|$ has its minimum $4-4\cos a$ at $\theta=0$. On $[-h_\theta,h_\theta]$, each edge is minimized at the endpoint nearer its preferred yaw and the crossed matching gives no less, so the bound is $4-4\cos(a-h_\theta)$ and the gap is $4(\cos(a-h_\theta)-\cos a)=8\sin(a-h_\theta/2)\sin(h_\theta/2)\ge(8/\pi)\sin(a/2)\,h_\theta$.

(iii) We now convert the local gap of the radius bound into a lower bound on the work of the search. Put $r=\sqrt\varepsilon$ and assume $0<\varepsilon\le1/40000$.
Since $r_{\rm res}(\mu)^2\le1+2\mu^2$, we have $r_{\rm res}(\mu)\le1+\mu^2$.
For $0\le h_\theta\le\pi/3$, the radius bound is below $1-\varepsilon$ whenever
$\mu^2<2h_\theta/\pi-\varepsilon$; this uses $\sin(h_\theta/2)\ge h_\theta/\pi$ and $\zeta^2\le\zeta$ on $[0,1]$.

For the counting argument, we retain both children of every split, including arcs that were pruned earlier. The three
roots, centered at $-2\pi/3$, $0$ and $2\pi/3$ with half-width $\pi/3$, cover
$[-\pi,\pi]$, and bisection preserves this cover. If a terminal certificate uses an
inherited bound, we contract its subtree at a certifying ancestor. The contracted
forest still covers the roots and has no more expansions. Each terminal leaf
now has a direct radius bound $L\ge1-\varepsilon$: the original feasible
incumbent satisfies $U\ge1$, even when it differs between leaves, and the zero
floor cannot certify this threshold. The central root cannot be terminal by
this criterion, so its descendants have $h_\theta\le\pi/6$.

A descendant that meets $[-r,r]$ satisfies $|\mu|\le h_\theta+r$. If $h_\theta\ge100\varepsilon$,
then $h_\theta>0$, $h_\theta<11/20$ and $r\le1/200$ give
\begin{equation}\label{eq:radius-terminal-width}
 \mu^2+\varepsilon\le\tfrac{57}{100}h_\theta+2\varepsilon
 \le\tfrac{59}{100}h_\theta<\tfrac35h_\theta<\tfrac{2h_\theta}{\pi},
\end{equation}
using $\pi<3.15$. Such an arc cannot be terminal. Thus each terminal leaf
covers at most $200\varepsilon$ of $[-r,r]$. Covered lengths add under bisection,
so the central tree has $N_{\rm leaf}\ge1/(100r)$ leaves and $N_{\rm leaf}-1\ge1/(200r)$ expanded
nodes, since $r\le1/200$. This proves the lower bound independently of
the order in which best-first search breaks ties.

With the Rubix bound, the objective for the single correspondence,
$\|\widehat R(u)x-y\|^2+(1-\|u\|^2)\|x_{xy}\|^2$, is affine in $u$, so by~(i)
each root bound equals the minimum over its arc. Fitting the returned matching gives
$U=1$ at $\theta=0$. The central root has bound $1$, and the other two have
bounds $5-4\cos(\pi/3)=3$, so the search stops after the root evaluations.
The separation concerns the displayed bounds: adding the global norm bound
$(\|y\|-\|x\|)^2=1$ also solves this particular radius-bound example immediately.
\end{proof}

\subsection{Joint Gravity Registration}\label{app:joint-gravity}

\subsubsection{Anchor cover and witness bound}\label{app:joint-cover}

\begin{proof}[Proof of \Cref{lem:joint-cover}]
The anchor construction gives a finite translation domain without knowledge of the matched subsets. For prescribed $k>0$ and a feasible cost $U$, an improving matching has an edge whose residual has norm at most $r=\sqrt{U/k}$. For every anchor $(i,j)$, we put $x'_a=x_a-x_i$, $y'_b=y_b-y_j$, and $e=t-x_i+R_z(\theta)y_j$. Then
\begin{equation}\label{eq:anchor-residual-identity}
 x_a-R_z(\theta)y_b-t=x'_a-R_z(\theta)y'_b-e.
\end{equation}
The forced anchor has residual $-e\in[-r,r]^3$. At least one selected edge of an improving matching can serve as this anchor, so the union over all anchors covers every improving solution. An optimum exists: each of the finitely many nonempty matchings allows the translation to be eliminated through its centroids, which leaves a compact yaw domain. For penalties, we use each $k$ with $U_k=U-\lambda(n+m-2k)\ge0$ and radius $r_k=\sqrt{U_k/k}$; a negative $U_k$ cannot improve $U$. The empty matching costs $\lambda(n+m)$.

For a residual-translation box with center $c_B$ and halfwidth vector $\boldsymbol h_t$, write $r_B^2=\|\boldsymbol h_t\|^2$ and $u=(\cos\theta,\sin\theta)$, and define
\begin{equation}\label{eq:joint-curvature-removal}
 h_M(u,e)=E_M(u,e)-k\|e-c_B\|^2.
\end{equation}
As in the conditional-yaw construction, the horizontal norms of the points are kept constant off the circle. Removing the common quadratic term in the translation makes $h_M$ affine in $u$ and affine in each component of $e$ separately, cross terms included, so it equals its product barycentric interpolant. At all 24 vertices, $\|e-c_B\|^2=r_B^2$, so the oracle value there is $\min_M E_M(v)-kr_B^2$.

For $f_v=\min_Mh_M(v)$ and any witness $W$ of size $k$, \eqref{eq:gravity-joint-slack} gives $h_M(v)\ge f_v\ge h_W(v)-\Delta_W$ for every competitor $M$. The product interpolation uses nonnegative weights, so it extends these vertex inequalities throughout the cell. Restoring the quadratic term gives
\begin{equation}\label{eq:joint-witness-bound-proof}
 \min_M E_M(u,e)\ge E_W(u,e)-\Delta_W.
\end{equation}
A witness that supports all vertices has $\Delta_W=0$, which gives equality and solves the cell by minimizing over its pose. The bound requires neither common support nor unique optimizers.
\end{proof}

\subsubsection{Constrained witness fit}\label{app:joint-witness-fit}

To minimize the cost of a nonempty witness over a cell in \Cref{lem:joint-cover}, we first clip $\overline{x'}-R_z(\theta)\overline{y'}$ coordinatewise to the box, where the bars denote the means of the selected points. Between the yaws at which a coordinate crosses a face of the box, the objective is
\begin{equation}\label{eq:clipped-yaw-polynomial}
 a_0+a_1\cos\theta+b_1\sin\theta+a_2\cos(2\theta)+b_2\sin(2\theta).
\end{equation}
With $\xi=\tan(\theta/2)$, its derivative has a quartic numerator with the ascending coefficients
\begin{equation}\label{eq:clipped-yaw-quartic}
 \begin{aligned}
 &(b_1+2b_2,\ -2a_1-8a_2,\ -12b_2,\\
 &\qquad -2a_1+8a_2,\ -b_1+2b_2).
 \end{aligned}
\end{equation}
The candidate minimizers are the real roots in their clipping intervals, the breakpoints, the endpoints of the arc, and the poles of the tangent half-angle substitution; constant cases and cases of lower degree are included. This covers positive residuals and minima on the boundary. Roots computed in floating point are numerical and do not certify a bound or a pose.

The cheaper independent-edge bound minimizes each edge over the cell before the constrained assignment, so each edge may use a different pose. For each edge it adds the squared horizontal distance between the arc and the rectangle to the squared vertical distance between the intervals. The horizontal distance is attained at an endpoint of the arc, at an intersection with a supporting line of the rectangle, at a coordinate tangency, or in the radial direction of a nonzero corner: a zero residual first occurs at an endpoint or at such an intersection, and elsewhere stationarity in the angle gives a tangency when one coordinate is clipped and a corner direction when both are clipped. Degenerate rectangles and circles of radius zero are included. The subsequent assignment still enforces the anchor and the cardinality, and the matching that it returns is fitted at one common pose for feasible updates.

\subsubsection{Refinement and termination}\label{app:joint-convergence}

The joint-gravity case of \Cref{thm:spatial-search} requires that the regional gap vanishes under refinement, when the yaw and every active translation coordinate shrink along refined paths. For a fixed anchor and cardinality, write $F(\theta,e)=\min_M E_M(\theta,e)$ over the admissible matchings. As nested cells shrink to $(\theta_0,e_0)$, the minimum $E_W^{\min}$ of a fixed witness $W$ over the cell tends to $E_W(\theta_0,e_0)$. The common quadratic term cancels in the differences at the vertices, so $\Delta_W$ tends to $E_W(\theta_0,e_0)-F(\theta_0,e_0)$. Thus $E_W^{\min}-\Delta_W$ tends to $F(\theta_0,e_0)$ also when the limiting slack is not zero. Because there are finitely many matchings, the convergence is uniform over changing witnesses and arbitrary ties. A matching returned by an exterior query and evaluated at a nearby physical pose in the cell tends to the same cost, and global fitting can only improve it. The gap between the regional feasible value and its lower bound therefore tends to zero. The finite-branching argument of \Cref{app:search-invariants} then gives termination at a positive tolerance without a budget.

When anchors are initialized lazily, every pending anchor must remain in the queue: zero is a valid lower bound for residuals with prescribed cardinality, and $\lambda|n-m|$ bounds the penalized costs. A search over selected anchors alone cannot certify the full problem.

\input{appendix/common_translation.tex}

\subsection{Covariance Polytopes and Proper Rotations}\label{app:covariance}

\begin{proof}[Proof of \Cref{thm:spatial-covariance}]
Centering complete point sets uses the same centroids for every permutation. This eliminates the translation as $t=m_x-Rm_y$ and gives the centered residual $S-2\langle R,C_\sigma\rangle_F$. For a unit quaternion $q=(q_0,q_v)$ with $q_v\in\mathbb R^3$, we use the convention that, for any $b_{\rm rot}\in\mathbb R^3$,
\begin{equation}\label{eq:quaternion-rotation-formula}
 R(q)b_{\rm rot}=(q_0^2-\|q_v\|^2)b_{\rm rot}+2q_v(q_v^\top b_{\rm rot})+2q_0(q_v\mathbin\times b_{\rm rot}).
\end{equation}
Horn's quaternion matrix~\cite{horn1987} encodes the covariance score as a quadratic form:
\begin{equation}\label{eq:quaternion-covariance-matrix}
 \begin{gathered}
 K(C)=\begin{bmatrix}\operatorname{tr}C&c_{\rm cross}(C)^\top\\
 c_{\rm cross}(C)&C+C^\top-(\operatorname{tr}C)I_3\end{bmatrix},\\
 c_{\rm cross}(C)=(C_{32}-C_{23},C_{13}-C_{31},C_{21}-C_{12})^\top.
 \end{gathered}
\end{equation}
Expanding the rotation formula proves \eqref{eq:quaternion-score}. Any unit eigenvector for the largest eigenvalue, also under ties, maximizes the Rayleigh quotient and gives a proper rotation. Linearity of $K$ and convexity of $\lambda_{\max}$ prove that a vertex attains the maximum in \eqref{eq:se3-polytope}. The affine dimension is $r_Xr_Y\le9$ for the centered ranks $r_X,r_Y$: the matrices with four nonzero entries and zero row and column sums span the direction space of the Birkhoff polytope~\cite{brualdi1977}, and their images span the outer products of point differences. Repeated points and rank deficiency are included.

The adjoint converts a lifted quaternion query into ordinary assignment scores. For a symmetric query matrix $Q\in\mathbb S^4$, with $b_Q\in\mathbb R^3$ and $Q_{vv}\in\mathbb R^{3\times3}$, we write
\begin{equation}\label{eq:quaternion-adjoint}
 \begin{gathered}
 Q=\begin{bmatrix}Q_{00}&b_Q^\top\\b_Q&Q_{vv}\end{bmatrix},\\
 D(Q)=(Q_{00}-\operatorname{tr}Q_{vv})I_3+2Q_{vv}+2[b_Q]_\times,
 \end{gathered}
\end{equation}
where $[b_Q]_\times b_{\rm rot}=b_Q\mathbin\times b_{\rm rot}$. The identity $\langle D(Q),C\rangle_F=\langle Q,K(C)\rangle_F$ gives the scores $a_i^\top D(Q)b_j$ without forming a $4\times4$ matrix for each edge. Unit rank-one lifts give rotations; exterior queries require neither a positive semidefinite $Q$ nor a rotation $D(Q)$.

For rational inputs, rational support queries, reconstruction of the facets and algebraic comparison of eigenvalues give a finite exact reduction; arbitrary real inputs require ideal arithmetic. No polynomial bound on the number of facets follows.
\end{proof}

For nonnegative isotropic scale, we replace the planar radius $r_\sigma$ in \eqref{eq:similarity} by $r_\sigma^{\mathrm{3D}}=\lambda_{\max}(K(C_\sigma))$, which is nonnegative because $K$ has zero trace. A vertex of the covariance polytope maximizes it. Zero spread, zero correlation and bounded scale are treated as in \Cref{sec:similarity}, with the source spread $S_Y$ as the denominator.

\subsection{Exterior Quaternion Lifts}\label{app:quaternion-lifts}

\begin{proof}[Proof of \Cref{lem:quaternion-enclosure}]
The enclosure follows by expanding the normalized lift of a convex combination. We write $\tilde q=\sum_a\alpha_a\tilde q_a$, where $\alpha_a\ge0$ and $\sum_a\alpha_a=1$, and assume that all corner dot products are strictly positive. For $a\le b$, we define
\begin{equation}\label{eq:quaternion-lift-weights}
 \beta_{aa}=\frac{\alpha_a^2\|\tilde q_a\|^2}{\|\tilde q\|^2},\quad \beta_{ab}=\frac{2\alpha_a\alpha_b\tilde q_a^\top \tilde q_b}{\|\tilde q\|^2}\quad(a<b).
\end{equation}
These coefficients are nonnegative and sum to one. Expanding gives $qq^\top=\sum_{a\le b}\beta_{ab}B_{ab}$, which proves the enclosure. Convexity gives the bound by the largest vertex value and, for any matching $\tau$,
\begin{equation}\label{eq:quaternion-slack-proof}
 \begin{aligned}
 &H(qq^\top)-\langle K(C_\tau),qq^\top\rangle_F\\
 &\quad\le\sum_{a\le b}\beta_{ab}\big(H(B_{ab})-\langle K(C_\tau),B_{ab}\rangle_F\big)\\
 &\quad\le\delta_\tau.
 \end{aligned}
\end{equation}
Bounding the quadratic score of the witness itself by its largest eigenvalue proves \eqref{eq:face-bound}. Feasibility gives $\delta_\tau\ge0$, and common support makes it zero without uniqueness.
\end{proof}

\paragraph{Yaw section of the lift.}\label{app:quaternion-yaw-section}
The yaw triangle reappears when this construction is projected onto the yaw coordinates. For the raw endpoints $\tilde q_\pm=(1,0,0,\pm\tan\alpha)$, with $0<\alpha<\pi/4$, the projection $p(Q)=(Q_{00}-Q_{33},2Q_{03})$ maps the rank-one lifts to $(\cos2\alpha,\pm\sin2\alpha)$ and their cross lift to $(\sec2\alpha,0)$. \Cref{fig:quaternion-lift} illustrates this at $\tan\alpha=1/2$, within a cell whose corner dot products are positive; this illustrative cell differs from the boxes of the initial cover.

\subsubsection{Chart cover and refinement}\label{app:quaternion-cover}

The quaternion case of \Cref{thm:spatial-search} uses this enclosure on a finite cover of the full rotation space. Scaling a coordinate of largest magnitude to positive one gives four charts that cover every quaternion. The eight orthants of each chart and their descendants have corner dot products of at least one. As nested boxes shrink to a physical quaternion $q_\infty$, all exterior lifts converge to $q_\infty q_\infty^\top$. Since $H$ is a finite maximum of linear assignment scores and hence continuous, the largest queried value tends to $H(q_\infty q_\infty^\top)$. The assignment score at the normalized raw center has the same limit, regardless of ties. Global fitting only increases this feasible score, and the retained regional upper bound never exceeds the largest queried value. Their gap is therefore bounded above by a quantity that tends to zero. Under the progress conditions of \Cref{app:search-invariants}, the search in exact arithmetic terminates for every $\varepsilon>0$; a budget may leave a gap. Gaps in the summed residual are twice the gaps in the score.

\subsection{The Rotation-Cap Bounds}\label{app:rotation-cap}

A rotation cap gives a cheaper regional bound that is applied before the exterior queries. We normalize the raw center of the chart box to $q_c$ and define the cap by
\begin{equation}\label{eq:quaternion-cap-geometry}
 \zeta=\min_a\frac{q_c^\top \tilde q_a}{\|\tilde q_a\|},\quad R_c=R(q_c),\quad \omega=2\arccos\zeta.
\end{equation}
Acute charts give $\zeta>0$ and $0\le\omega<\pi$. For $\tilde q=\sum_a\alpha_a\tilde q_a$, we have $q_c^\top \tilde q\ge\zeta\sum_a\alpha_a\|\tilde q_a\|\ge\zeta\|\tilde q\|$. Normalization therefore places the quaternion within $\arccos\zeta$ of $q_c$, and its rotation within $\omega$ of $R_c$.

For vectors $a_{\rm cap},b_{\rm cap}\in\mathbb R^3$, we put $c_{\rm cap}=a_{\rm cap}^\top R_c b_{\rm cap}$ and $\nu_{\rm cap}=\|a_{\rm cap}\|\|b_{\rm cap}\|$. The largest dot product over this rotation cap is $\phi_\omega(c_{\rm cap},\nu_{\rm cap})$, where
\begin{equation}\label{eq:rotation-cap-maximum}
 \phi_\omega(c,\nu)=
 \begin{cases}
 0,&\nu=0,\\
 \nu,&c\ge \nu\cos\omega,\\
 c\cos\omega+\sqrt{\nu^2-c^2}\sin\omega,&\text{otherwise}.
 \end{cases}
\end{equation}
The formula reduces the angle between $a_{\rm cap}$ and $R_c b_{\rm cap}$ by at most $\omega$, stopping at zero. One assignment on these pairwise maxima gives the \emph{ordinary cap bound}, which allows a separate rotation for each edge.

The \emph{stronger cap bound} limits the improvement over a retained witness matching. For a permutation $\tau$ and any rotation $R$,
\begin{equation}\label{eq:rotation-cap-difference}
 \max_\sigma\langle R,C_\sigma\rangle_F-\langle R,C_\tau\rangle_F
 =\max_\sigma\sum_i a_i^\top R\big(b_{\sigma(i)}-b_{\tau(i)}\big).
\end{equation}
We apply $\phi_\omega$ to the differences $b_j-b_{\tau(i)}$ and maximize the resulting bounds by one assignment, which gives $\eta_\tau^{\rm cap}\ge0$ because $\tau$ itself has zero difference. Adding $\lambda_{\max}(K(C_\tau))$ bounds the score over the cell. We take the minimum over the witnesses and the ordinary cap bound. This bound can close regions on which one matching remains optimal.

\subsection{One-Assignment Difference Bounds and Reusable Duals}\label{app:matching-difference}

\subsubsection{Lifted matching differences and dual reuse}\label{app:spatial-bound-reuse}

\begin{proof}[Proof of \Cref{lem:spatial-bound-reuse}]
For the square lifted assignment problem, we fix a permutation $\tau$ and define $d^\tau$ and $\eta_\tau$ by \eqref{eq:lifted-difference}. At any convex combination $Q$ of the exterior matrices, linearity gives
\begin{equation}\label{eq:lifted-edge-difference-bound}
 A_{ij}(Q)-A_{i,\tau(i)}(Q)\le d_{ij}^\tau.
\end{equation}
Summing over a matching and maximizing gives $H(Q)-\langle K(C_\tau),Q\rangle_F\le\eta_\tau$, with $\eta_\tau\ge0$ because $d_{i,\tau(i)}^\tau=0$. If $\eta_\tau=0$, the witness is optimal throughout the enclosure. Its largest rotation score plus $\eta_\tau$ bounds the cell. The edgewise maxima can be attained at different exterior matrices, so even common support need not make this relaxation zero; support queries or feasible duals can tighten it.

Computing the edgewise maxima and then one assignment gives the stated work bound, and accumulating the maxima one matrix at a time gives the stated storage bound. Summing $f_i+g_{\sigma(i)}\ge A_{i,\sigma(i)}(Q)$ over a matching proves the dual upper bound defined in \eqref{eq:reused-dual}, and a feasible matching that attains this bound is optimal. Rectangular assignments in which every row is matched require nonnegative column prices, with zero prices on the unused columns for equality. Recomputing the row maxima restores dual feasibility, which permits the reuse of prices also after the old matching has ceased to be optimal. The prices may also be linear functions of the query that are obtained from paths of difference constraints: evaluating them and repairing the rows remains feasible, even though the paths need not stay optimal.
\end{proof}

\subsubsection{Joint-gravity matching differences}\label{app:gravity-difference}

For the joint gravity search at a fixed cardinality, we apply the corresponding construction with cost differences to the square assignment representation of $h_M$ with dummy nodes, in which real edges are reduced by $\|e-c_B\|^2$ and dummy edges have zero cost. Let its finite edge costs be $B_{ij}(v)$ and let $\tau$ be a finite witness assignment. We define $\eta^-_\tau=\min_\sigma\sum_i d^-_{i,\sigma(i)}$, where
\begin{equation}\label{eq:gravity-difference-construction}
 d^-_{ij}=\min_v\{B_{ij}(v)-B_{i,\tau(i)}(v)\},
\end{equation}
and forbidden edges remain forbidden. Product interpolation gives $\eta^-_\tau\le\min_Mh_M(u,e)-h_\tau(u,e)$, and the feasibility of the witness gives $\eta^-_\tau\le0$. Restoring the common quadratic term yields $\min_ME_M\ge E_\tau+\eta^-_\tau$. The quadratic term must be subtracted before the edge differences are formed: real and dummy rows carry different quadratic terms, although every assignment has $k$ real edges. For the order $N_{\rm aug}=n+m-k$ and $N_{\rm query}=24$ vertices, this bound costs $\mathcal O(N_{\rm query}N_{\rm aug}^2+N_{\rm aug}^3)$ operations and $\mathcal O(N_{\rm aug}^2)$ storage, in addition to the geometry and the minimization for the witness.

\paragraph{Equivalent geometric calculations.}\label{app:geometry-identities}
Two geometric simplifications reduce the work of the joint gravity search. First, on a full yaw circle the independent-edge bound needs only the extreme radii of the rectangle. For a full yaw circle of radius $r$ and a translation rectangle centered at $c_B$ with halfwidth $\boldsymbol h_t$, the norms of the points of the rectangle fill the interval between
\begin{equation}\label{eq:rectangle-radius-extrema}
 r_{\min}=\|\max(|c_B|-\boldsymbol h_t,0)\|,\qquad r_{\max}=\||c_B|+\boldsymbol h_t\|.
\end{equation}
Minimizing $|\|z\|-r|$ over the connected rectangle gives the distance $\max(r_{\min}-r,r-r_{\max},0)$. We square this horizontal distance and add the squared vertical distance between the intervals. This includes zero radius and degenerate rectangles; subarcs still need the angular candidates of \Cref{app:joint-witness-fit}. Second, a reduced matching difference attains its minimum over the box in closed form. At a fixed yaw vertex, write a reduced difference as $\psi(e)=\psi(c_B)+\boldsymbol g_\psi^\top(e-c_B)$, where $\boldsymbol g_\psi\in\mathbb R^3$ is its gradient. Its minimum over the box is $\psi(c_B)-\sum_\ell h_{t,\ell}|g_{\psi,\ell}|$. Thus three evaluations in the yaw replace the 24 product vertices. The gradients of differences between real and dummy edges must retain the shift by the quadratic term.

\subsection{Search Invariants and Cost}\label{app:search-invariants}

\begin{proof}[Proof of \Cref{thm:spatial-search}]
The queue represents the part of the search domain that can still contain an improving solution. For every live region we keep $L_\Omega\le\min_{p\in\Omega,M}E(M,p)$, together with a feasible incumbent of cost $U$. Discarded regions had $L_\Omega\ge U$, and a later decrease of $U$ cannot make them improving. Hence
\begin{equation}\label{eq:live-cover-interval}
 L=\min\{U,\min_{\Omega\text{ live}}L_\Omega\}\le E^*\le U,
\end{equation}
with $L=U$ for an empty queue. Feasible updates preserve this interval, also when the fitted pose lies outside the cell that generated it, and so do tightening, changes of priority and complete child covers. Children inherit the maximum of the parent's bound and their new bounds. Keeping the old record of a region until its evaluation or replacement is complete preserves the cover under interruption.

The initial residual bounds are zero for complete or prescribed-cardinality matching and $\lambda|n-m|$ for penalties. They cover all roots before any evaluation: the three yaw arcs, all anchors of the joint gravity search, or the 32 quaternion boxes. Complete matching can also use the score bound
\begin{equation}\label{eq:initial-covariance-bound}
 b_0=\sum_{i=1}^n r_i^\uparrow s_i^\uparrow,
 \qquad r_i=\|a_i\|,\quad s_j=\|b_j\|,
\end{equation}
with both sequences in increasing order. Edgewise Cauchy--Schwarz and the rearrangement inequality bound every assignment by $b_0$, which gives the residual bound $\max(0,S-2b_0)$.

A valid interval alone does not guarantee that the search finishes. Termination also requires progress: yaw arcs and the longest quaternion coordinates are bisected, and the joint gravity search must shrink the yaw and every active translation coordinate. Tightening must eventually yield to subdivision, and adaptive evaluations must retain a regional bound and feasible updates whose gap vanishes under refinement. The vanishing gaps follow from \Cref{app:gravity-accuracy} for conditional yaw, \Cref{app:joint-convergence} for joint gravity, and \Cref{app:quaternion-cover} for the quaternion search. An infinite tree with finitely many roots and finitely many children per node would contain an infinite path, along which progress makes the gap vanish. Once the gap of a least-bound cell is at most $\varepsilon$, its feasible value puts the global incumbent within $\varepsilon$ of the live lower bound, which prevents further subdivision. This contradiction proves finite termination without a budget. Beyond the count of \Cref{prop:gravity-accuracy} for conditional yaw, no polynomial bound on the number of regions follows.
\end{proof}

\subsubsection{Arithmetic cost and storage}\label{app:spatial-cost}

For one query, the dense scores cost $\mathcal O(n^2)$ operations, the assignment $\mathcal O(n^3)$, a covariance $\mathcal O(n)$ and an eigenvalue problem constant work. Evaluating $N_{\rm query}\le36$ directions and $N_{\rm wit}^{\rm eval}$ witnesses in a cell therefore costs $\mathcal O(N_{\rm query}n^3+N_{\rm wit}^{\rm eval}n+N_{\rm query}N_{\rm wit}^{\rm eval})$ before reuse, and each queue operation costs $\mathcal O(\log N_{\rm cell})$ for $N_{\rm cell}$ cells. Partial matching substitutes the augmented order, which is $n+m-k$ for prescribed cardinality. The storage is $\mathcal O(n^2)$ when queries are processed one at a time and $\mathcal O(N_{\rm query}n^2)$ in batches, plus $\mathcal O(n)$ for each cached assignment or witness.

\subsection{Partial Matching and Translation}\label{app:partial-translation}

For partial matching, a fixed translation permits the quaternion construction although the matched subsets vary. We put $x_i^t=x_i-t$. For a partial matching $M$, we define
\begin{equation}\label{eq:partial-fixed-translation}
 \begin{aligned}
 C_{M,t}&=\sum_{(i,j)\in M}x_i^t y_j^\top,\\
 c_{M,t}&=\sum_{(i,j)\in M}(\|x_i^t\|^2+\|y_j\|^2)\\
 &\quad+\lambda(n+m-2|M|).
 \end{aligned}
\end{equation}
For prescribed cardinality, the penalty is omitted. The negative half-cost is $q^\top\big(K(C_{M,t})-\tfrac12c_{M,t}I_4\big)q$, so the support, lift and witness bounds apply through a partial assignment oracle. When the empty matching is allowed, it gives a constant multiple of the identity. The translation must be subtracted before the covariance is formed; centering the whole sets is not valid for partial matching.

\subsubsection{Moment hull for unknown translation}\label{app:partial-moments}

\begin{proof}[Proof of \Cref{prop:partial-moments}]
With unknown translation, the covariance alone is insufficient, because the matched centroids vary. We instead collect the sufficient statistics of each matching,
\begin{equation}\label{eq:partial-moment-construction}
 \begin{gathered}
 k=|M|,\qquad \boldsymbol s_{X,M}=\sum_Mx_i,\qquad \boldsymbol s_{Y,M}=\sum_My_j,\\
 S_M=\sum_M(\|x_i\|^2+\|y_j\|^2),\qquad
 C_M=\sum_Mx_i y_j^\top
 \end{gathered}
\end{equation}
in a space of dimension at most 17. The vector moments are coordinate sums, not centroids. At fixed $R,t$, the residual is affine in these statistics:
\begin{equation}\label{eq:partial-moment-cost}
 S_M+k\|t\|^2-2\langle R,C_M\rangle_F-2t^\top \boldsymbol s_{X,M}+2t^\top R\boldsymbol s_{Y,M}.
\end{equation}
The penalty for unmatched points is also affine. Taking the infimum over $R,t$ gives a concave function, whose minimum over the finite hull of the statistics is attained at a vertex. Nonempty vertex matchings are evaluated by subset centering and proper rotation, and the empty matching is evaluated separately. Linear support queries are partial assignments. This vertex principle covers unequal cardinalities, unmatched points and full rigid motion.
\end{proof}

%% file: appendix/common_translation.tex
\subsection{Common Translation and Additive Accuracy}\label{app:common-translation}

\Cref{lem:common-translation} gives a simpler regional bound than the witness construction of \Cref{lem:joint-cover}, and \Cref{prop:joint-gravity-accuracy} uses it to control a rational joint grid. Both results keep one translation that is shared by every selected edge. To match the interface of the scan-matching solver, $x_i$ is a source point and $y_j$ a target point, and the fitted map is $R_z(\theta)x_i+t$. This convention reverses the names of the two sets in \eqref{eq:pw}.

\subsubsection{Common-translation bound}\label{app:common-translation-proof}

\begin{proof}[Proof of \Cref{lem:common-translation}]
Set $F_{\rm red}(u,t)=F(u,t)-k\|t-c_B\|^2$. Every selection has the same quadratic term in the translation, because the cardinality is fixed. Removing this term makes the cost of each fixed selection affine in $u$ and affine in $t$ separately, so their minimum is concave in each. Interpolating first over the box and then over the triangle, with box weights $\gamma_\nu$ and barycentric triangle weights $\lambda_r$ at the triangle vertices $v_r$, gives
\begin{equation}\label{eq:product-interpolation-proof}
\begin{split}
F_{\rm red}(u,t)&\ge\sum_\nu\gamma_\nu\sum_r\lambda_rF_{\rm red}(v_r,t_\nu)\\
&=\sum_\nu\gamma_\nu\ell_\nu(u)-k\|\boldsymbol h_t\|^2.
\end{split}
\end{equation}
Every corner satisfies $\|t_\nu-c_B\|^2=\|\boldsymbol h_t\|^2$, also for degenerate boxes. Taking the smallest corner interpolant and using $F\ge F_{\rm red}$ proves the claim; no uniqueness is assumed.
\end{proof}

Evaluating a residual off the circle requires the correction $(1-\|u\|^2)\|x_i\|^2$. In three dimensions with known gravity, the box has eight corners and the correction uses only the horizontal norm of the source point; the proof is unchanged.

\subsubsection{Additive accuracy}\label{app:joint-gravity-accuracy}
Earlier polynomial approximation schemes cover planar RMS matching~\cite{agarwal2006}, Euclidean transport under motion~\cite{cabello2008} and other transport objectives~\cite{soubrier2026}. The result here quantifies the additive accuracy of our particular common-translation construction.

\begin{proof}[Proof of \Cref{prop:joint-gravity-accuracy}]
The grid controls two errors: the one from the quadratic term in the translation, and the one from replacing an exterior yaw query by a feasible sample. We divide each translation coordinate into $N_{t,j}$ equal intervals. We sample each quarter circle by
\begin{equation}\label{eq:rational-yaw-parameterization}
 p(\xi)=\frac{(1-\xi^2,2\xi)}{1+\xi^2},\qquad \xi=j/N_\theta,
\end{equation}
and rotate these samples through right angles. Adjacent endpoints $u_-,u_+$ have the rational tangent apex $(u_-+u_+)/(1+u_-^\top u_+)$. At each of the eight corners of a box, we query the two endpoints, their tangent apex, and the physical point at the midpoint of the parameter interval. The first three queries supply values for the lower bound, and the physical samples supply feasible values. Let $f_{\min}$ be the smallest queried value. \Cref{lem:common-translation} gives the lower bound $L=f_{\min}-k\|\boldsymbol h_t\|^2$. The minimum of the affine costs is $\kappa_B$-Lipschitz in the yaw coordinates.

To bound the second error, write $\delta=\xi_{j+1}-\xi_j\le1/N_\theta$ and $\phi(\xi)=2\arctan \xi$. The half-width $\alpha$ of the arc satisfies $\alpha\le\delta$ and $\alpha\le\pi/4$, since $|\phi'|\le2$. The angular midpoint differs from $\phi((\xi_j+\xi_{j+1})/2)$ by at most $\delta^2/2$, by $|\phi''|\le4$ and the error of midpoint interpolation. The apex has radius $\sec\alpha$, so its Euclidean distance to this physical sample is at most
\begin{equation}\label{eq:rational-apex-distance}
 \delta^2/2+(\sec\alpha-1)
 \le\delta^2/2+\alpha^2\le2/N_\theta^2.
\end{equation}
Here $\sec\alpha-1=(1-\cos\alpha)/\cos\alpha\le\alpha^2$ because $\cos\alpha\ge1/2$. If the smallest queried value occurs at an endpoint, that query is already physical and needs no projection. Thus the best sampled feasible value satisfies
\begin{equation}\label{eq:rational-grid-gap}
 U-L\le k\|\boldsymbol h_t\|^2+2\kappa_B/N_\theta^2\le\varepsilon.
\end{equation}
Taking the minima of the lower bounds and of the feasible values across cells preserves this gap. The number of queries is the number of grid cells times four quarters, eight corners and four evaluations. All query coordinates and costs are rational, and integer comparisons determine the grid counts. Ties and degenerate translation boxes remain valid.
\end{proof}

The bit lengths of the queries grow polynomially in the length of the input and in the logarithms of the grid counts. The total work depends polynomially on the numerical scale and on $1/\varepsilon$, so this is not an exact algorithm for gravity or SE(3) registration that is polynomial in the bit length.

%% file: appendix/spatial_algorithms.tex
\section{Spatial Search Algorithms}\label{app:spatial-algorithms}

The searches of \Cref{sec:spatial} share one queue and differ in their regional bounds and subdivisions. All of them minimize summed residuals under the convention $x_i\approx Ry_j+t$. \Cref{app:search-invariants} proves their guarantees in exact arithmetic. \Cref{prop:gravity-accuracy} counts the queries of the conditional yaw search (\Cref{alg:gravity-yaw}), and \Cref{prop:joint-gravity-accuracy} counts those of the separate rational grid of \Cref{app:joint-gravity-accuracy}.

\subsection{The Shared Queue and Its Interruption Rule}\label{app:spatial-queue}

The queue of \Cref{thm:spatial-search} retains each region until its evaluation or its complete child cover is ready (\Cref{alg:spatial-search}). The operation $\textsc{Evaluate}$ returns a regional bound and the fits of feasible matchings, and $\textsc{Children}$ supplies a complete finite cover of a region. \Cref{alg:joint-gravity-search,alg:se3-search} instantiate these operations. A feasible fit may improve the incumbent also when its pose lies outside the evaluated cell.

\begin{algorithm*}[tp]
\caption{Budgeted spatial search with a complete live cover}\label{alg:spatial-search}
\begin{algorithmic}[1]
\REQUIRE A feasible incumbent of cost $U$; a finite root cover with valid lower bounds; host operations $\textsc{Evaluate}$ and $\textsc{Children}$; $\varepsilon>0$; optional budget.
\ENSURE A feasible matching and pose and an interval $[L,U]$ containing the optimum.
\STATE Represent every root in a min-heap, initially marked unevaluated.
\WHILE{the heap is nonempty}
\STATE $L\leftarrow\min\{U,\min_{\Omega\text{ in heap}}L_\Omega\}$.
\STATE If $U-L\le\varepsilon$ or the budget is exhausted, stop.
\STATE Select a least-bound record $\Omega$ without removing its domain from the live cover.
\IF{$L_\Omega\ge U$}
\STATE Remove $\Omega$; continue.
\ENDIF
\IF{$\Omega$ is unevaluated}
\STATE Run $\textsc{Evaluate}(\Omega,U)$; retain each completed feasible improvement of the incumbent.
\STATE If interrupted, keep $\Omega$ with its old lower bound and stop.
\STATE $L_\Omega\leftarrow\max\{L_\Omega,\text{the returned regional lower bound}\}$.
\STATE Remove $\Omega$ if $L_\Omega\ge U$; otherwise update its heap key and mark it evaluated.
\ELSE
\STATE Construct the complete list $\textsc{Children}(\Omega)$ without removing $\Omega$.
\STATE If interrupted, keep $\Omega$ and stop.
\STATE Replace $\Omega$ atomically by all children, unevaluated and inheriting $L_\Omega$.
\ENDIF
\ENDWHILE
\STATE $L\leftarrow\min\{U,\min_{\Omega\text{ in heap}}L_\Omega\}$, with the empty-heap minimum equal to $+\infty$.
\STATE Return the stored incumbent and $[L,U]$.
\end{algorithmic}
\end{algorithm*}

Roots may be created lazily, provided that every pending anchor and all 32 quaternion roots are represented in the queue by records that carry valid bounds. The previous bound of a region remains in force until all queries that define a new bound are complete. Staged tightening follows \Cref{sec:bound-cost}, under the progress conditions of \Cref{app:search-invariants}. If regions are pruned at the threshold $U-\varepsilon$, the smallest discarded bound must be retained for $L$.

\subsection{Joint Gravity Search}\label{app:joint-gravity-search}

This search implements the anchor cover and the witness bound of \Cref{lem:joint-cover}, proved in \Cref{app:joint-cover}. It first partitions the matching family by cardinality, so that every cost within a branch has the same quadratic term in the translation. Prescribed cardinality has the single branch $k$ with $p_k=0$. The penalized model uses the branches $1\le k\le\min(n,m)$ with $p_k=\lambda(n+m-2k)$, and the empty matching has value $\lambda(n+m)$. For prescribed $k=0$ and for empty inputs under the penalized model, the search returns zero and $\lambda(n+m)$, respectively, without searching.

We define $\textsc{FitYaw}(M)$ by centering the selected points, minimizing their yaw sinusoid, and restoring $t=\bar x_M-R_z(\theta)\bar y_M$, where the bars denote the means of the selected points. This fit optimizes the pose of a retained matching and is feasible also outside the cell that generated the matching; centering the whole sets is not valid for partial matching. Any admissible matching with $k$ pairs initializes a finite incumbent without reference correspondences.

Within each branch, an anchor defines a bounded domain for the residual translation. In the anchor coordinates of \Cref{app:joint-gravity}, we force the anchor edge, with cost $\|e\|^2$, and augment the assignment of the remaining $k-1$ pairs with dummies; this includes $k=1$. The costs of exterior queries use the affine expansion of \eqref{eq:gravity-edge}, as in \Cref{app:conditional-gravity}.

\begin{algorithm*}[tp]
\caption{Joint gravity search through cardinality and anchor covers}\label{alg:joint-gravity-search}
\begin{algorithmic}[1]
\REQUIRE Nonempty leveled point sets; prescribed $k>0$ or penalty $\lambda>0$; $\varepsilon>0$; optional budget.
\ENSURE Feasible matching, yaw, and translation with a residual interval.
\STATE Fit any admissible nonempty matching to obtain an incumbent. For the penalized model also test the empty matching.
\STATE Save its cost $U_0$; form the cardinality branches and constants $p_k$ described above.
\FOR{each branch with $p_k\le U_0$ and every anchor $(i,j)$}
\STATE Set $r_k\leftarrow\sqrt{(U_0-p_k)/k}$ and create the root $\Omega=(k,i,j,[-\pi,\pi],[-r_k,r_k]^3)$ with lower bound $p_k$.
\ENDFOR
\STATE Run \Cref{alg:spatial-search} with the regional evaluation and subdivision rules below.
\STATE Restore the original-frame pose using $A_x,A_y$ from \Cref{sec:gravity}; return its matching and interval.
\STATE \textbf{Regional evaluation} for $\Omega=(k,i,j,I,B)$:
\STATE For every matching returned below, apply $\textsc{FitYaw}$ and test its physical cost plus $p_k$ against the incumbent.
\STATE Compute $c_{ab}^{\rm edge}\leftarrow\min_{\theta\in I,\,e\in B}\|x'_a-R_z(\theta)y'_b-e\|^2$ for every edge.
\STATE Solve the forced-anchor $k$-assignment on $c^{\rm edge}$; set $\ell_{\rm edge}$ to its value plus $p_k$.
\STATE If $\ell_{\rm edge}\ge U$ or $|I|\ge\pi$, return $\ell_{\rm edge}$.
\STATE Let $c_B$ be the center of $B$, and form the three yaw-triangle vertices and eight box corners.
\FOR{each of their 24 product vertices $v$}
\STATE Query the forced-anchor $k$-assignment to obtain $f_v=\min_M h_M(v)$ and a matching, where $h_M=E_M-k\|e-c_B\|^2$ and $E_M$ includes $p_k$.
\ENDFOR
\STATE Select one returned matching $W$; compute $\Delta_W=\max_v\{h_W(v)-f_v\}$.
\STATE Compute $E_W^{\min}=\min_{\theta\in I,\,e\in B}E_W(\theta,e)$ by the clipped-translation minimization in \Cref{app:joint-gravity}.
\STATE Return $\max\{\ell_{\rm edge},E_W^{\min}-\Delta_W\}$.
\STATE \textbf{Subdivision}: split an initial full circle into three arcs of width $2\pi/3$; thereafter bisect yaw and the nonconstant box coordinates in cyclic order along each path.
\end{algorithmic}
\end{algorithm*}

The regional evaluation first tests the inexpensive independent-edge bound and uses the witness bound only on the cells that remain unresolved. \Cref{app:joint-witness-fit} derives $\ell_{\rm edge}$ and the constrained minimum $E_W^{\min}$ of the witness, and \Cref{app:geometry-identities} gives equivalent calculations at lower cost. Other subdivisions may be used if they preserve the cover and shrink every active coordinate (\Cref{app:joint-convergence}). Pending roots may use the current $U$ for tighter radii, provided that all anchors and cardinality bounds are retained.

The difference bound of \Cref{app:gravity-difference} replaces the 24 assignments by one and gives $\max\{\ell_{\rm edge},E_\tau^{\min}+\eta^-_\tau\}$, where $\eta^-_\tau\le0$. The forced anchor contributes equally to the costs of all matchings, so its contribution cancels when the differences are formed. The forbidden edges and the prescribed shift between real and dummy edges are retained. With the closed-form minimum over the box (\Cref{app:geometry-identities}), the 24 vertices reduce to three evaluations in the yaw, and the full witness bound may follow on the cells that remain unresolved.

\subsection{Complete SE(3) Search in Quaternion Charts}\label{app:se3-search}

This search combines the fit for a fixed permutation of \Cref{thm:spatial-covariance} with the regional enclosure of \Cref{lem:quaternion-enclosure}. Complete matching permits centering before the quaternion search. For equally weighted complete sets of common size $n>0$, $\textsc{FitSE3}$ is the quaternion fit for a fixed permutation of \Cref{app:covariance}, including eigenvalue ties. The fitted permutation is feasible without reassignment, and empty inputs have value zero.

The regional evaluation starts with the cheaper rotation-cap bounds and proceeds to the exterior queries only for the boxes that remain unresolved. We take the smaller of the two cap bounds of \Cref{app:rotation-cap}, with the permutations of the incumbent, of the center and of the ordinary cap bound as witnesses, and we fit every returned permutation. The set $\mathcal W$ of witnesses for the bound~\eqref{eq:face-bound} contains the incumbent and all permutations returned at the center, by the cap bounds and by completed exterior queries.

\begin{algorithm*}[tp]
\caption{Complete SE(3) search by exterior quaternion assignments}\label{alg:se3-search}
\begin{algorithmic}[1]
\REQUIRE Equally weighted point sets $X,Y\subset\mathbb R^3$ of common size $n>0$; $\varepsilon>0$; optional budget.
\ENSURE A feasible permutation and proper rigid pose with a summed-residual interval.
\STATE Save $m_x,m_y$; center the point sets to obtain $a_i,b_j$ and form $S$ as in \Cref{sec:se3}.
\STATE Solve the physical assignment at $R=I$; apply $\textsc{FitSE3}$ to initialize the incumbent.
\STATE $b_0\leftarrow\sum_i r_i^\uparrow s_i^\uparrow$, where $r_i=\|a_i\|$, $s_j=\|b_j\|$.
\STATE Create all 32 chart/orthant root records, each with lower bound $\max\{0,S-2b_0\}$.
\STATE Run \Cref{alg:spatial-search} using the regional evaluation below and longest-coordinate bisection of raw chart boxes.
\STATE Return the stored permutation, $R$, $t=m_x-Rm_y$, and residual interval.
\STATE \textbf{Regional evaluation} for one chart box $\Omega$:
\STATE Initialize $\mathcal W$ with the incumbent permutation. Normalize the raw center to $q_c$, solve its physical assignment, and fit and retain that permutation.
\STATE Compute the ordinary and stronger cap bounds described above; let their minimum be $b_{\rm cap}$.
\STATE If $\max\{0,S-2b_{\rm cap}\}\ge U$, return that residual lower bound.
\STATE Form the eight raw corners $\tilde q_a$ and the 36 matrices $B_{ab}$ in \eqref{eq:outer-quaternion}.
\FOR{each $a\le b$}
\STATE Solve $f_{ab}\leftarrow\max_\sigma\sum_i a_i^\top K^*(B_{ab})b_{\sigma(i)}$ and record a maximizing permutation $\sigma_{ab}$.
\STATE Apply $\textsc{FitSE3}(\sigma_{ab})$ and update the feasible incumbent if its physical cost is smaller.
\ENDFOR
\STATE For each $\tau\in\mathcal W$, compute $\delta_\tau\leftarrow\max_{a\le b}\{f_{ab}-\langle K(C_\tau),B_{ab}\rangle_F\}$.
\STATE $b_\Omega\leftarrow\min\{b_{\rm cap},\max_{a\le b}f_{ab},\min_{\tau\in\mathcal W}[\lambda_{\max}(K(C_\tau))+\delta_\tau]\}$.
\STATE Return the residual lower bound $\max\{0,S-2b_\Omega\}$.
\end{algorithmic}
\end{algorithm*}

The order shown, with the cap bounds first, is one valid schedule of the evaluation. An adaptive schedule may replace it by the difference bound~\eqref{eq:lifted-difference}, by repaired duals or by the full set of queries, while retaining the inherited bounds; \Cref{app:spatial-bound-reuse} proves that these alternatives are valid. For termination without a budget, a box that remains unresolved receives the physical assignment at its center and the complete exterior-query bound before it is subdivided, unless the replacement also has a gap that vanishes under refinement (\Cref{app:quaternion-cover,app:search-invariants}). The partial extension of \Cref{app:partial-translation} fixes the translation and does not solve unrestricted partial SE(3) registration.

%% file: appendix/experiments.tex
\section{Experimental Details}\label{app:experiments}

\input{appendix/computing.tex}

\subsection{Planar Conventions}\label{app:planar-conventions}

\paragraph{Tolerance.} Support comparisons and collinearity tests use double-precision relative tolerance $10^{-10}$ scaled by $\sum_i|a_i|\max_j|b_j|$, tightened to $10^{-12}$ when short edges would merge. Costs and failures follow \Cref{sec:exp-planar}.

\paragraph{Point sets.} All inputs are centered. Gaussian and anisotropic samples have zero mean and covariance $I_2$ and $\operatorname{diag}(1,0.09)$, respectively. Square and disk samples are uniform on $[-1,1]^2$ and on the unit disk. Mixtures have five equally weighted components with centers drawn from $\mathcal N(0,9I_2)$ and standard deviation $0.5$, and noisy circles add $\mathcal N(0,0.05^2I_2)$ to uniform samples of the unit circle. A noisy copy is $y_i=Rx_{\sigma(i)}+\sigma_{\mathrm{noise}}\eta_i$ with standard Gaussian $x_i$ and $\eta_i$ and uniformly random $\sigma$ and $R$; the study of polygon geometry uses $\sigma_{\mathrm{noise}}=0.1$ and $R\in\mathrm{SO}(2)$. For MNIST~\cite{lecun1998}, we sample pixels in proportion to their intensity, jitter the points within their pixels and normalize to unit root-mean-square radius.

\subsection{Polygon Geometry}\label{app:geometry}

Across 966 instances, vertex-count log--log exponents for $n\ge20$ are 1.20--1.30; at $n=200$, family means are 929--1,736 against the 39,800 bound. Gaussian sets at $n=100$ average \FixedGauss\ AM fixed points. For $n\le100$, \FracTransposition\ of transitions are swaps and \FracEnclosing\ of cycles enclose other points.

\paragraph{Scaling (\Cref{fig:scaling}).} Each family contributes ten pairs at $n=10,20,50,100$, five at $n=200$ and three at $n=400$: 144 centered $\mathrm{SO}(2)$ problems. Median-time log--log exponents for $n\ge50$ are 3.74 for reconstruction and 2.82 for pruning. The 360-angle grid stays within $1\%$ throughout but exceeds the numerical reference on 15 pairs.

\subsection{Solver Comparison}\label{app:solvers}

AM (PCA) tests every admissible sign of covariance eigenvectors: two starts for $\mathrm{SO}(2)$, four for $\mathrm{O}(2)$, retaining the best fit.

\paragraph{Real and structured pairs.} The 1,000 pairs per family are uniform MPEG-7 pairs~\cite{latecki2000}, unrelated Gaussian sets ($n=50$), and random lattice types at $\sigma_{\mathrm{noise}}=0.12$, with the preprocessing of \Cref{sec:where,sec:exp-planar}.

\begin{table*}[tp]
\caption{\textbf{Failure rates on real and structured pairs.} Each family has 1,000 pairs: MPEG-7 ($n=64$), Gaussian ($n=50$), and lattice neighborhoods ($n=24$). LAPs are mean numbers of assignments; ``fail'' and ``$>1\%$'' are the percentages of pairs with a numerical failure and with more than $1\%$ excess, respectively. Bold marks fewer failures than every AM variant or fewer LAPs than both grids.}\label{tab:failures-full}\label{tab:failures}
\begin{center}
\input{tables/realpairs_full.tex}

\end{center}
\end{table*}

\paragraph{Synthetic families (\Cref{tab:failures-synth}).} Six families each contribute 100, 100, 50, 10 centered instances at $n=20,50,100,200$, totaling 1,560 $\mathrm{O}(2)$ problems. AM stops at a repeated matching, starting from $R=I$, principal axes or ten random poses. Principal-axis starts are selected by final cost; Adamo et al.~\cite{adamo2025} instead use initial cost.

Best-first branch-and-bound (B\&B) searches the angle with the bound
\begin{equation}
h(\theta)\le h(\theta_c)+\Lambda\sin\delta
\end{equation}
on intervals of half-width $\delta\le\pi/2$ about $\theta_c$, where $\Lambda=\sum_i|x_i|\max_j|y_j|$ and
\begin{equation}\label{eq:angular-grid-support}
h(\theta)=\max_\sigma\Re(e^{\mathrm i\theta}z_\sigma(\bar x,y)).
\end{equation}
The bound holds because every correlation has modulus at most $\Lambda$ and, for centered inputs, the polygon contains the origin, so that $h\ge0$. For $n\le50$, the search starts at half-width $\pi/8$ and allows at most $10^5$ assignments, with a requested objective gap of $10^{-6}\max(v_{\rm ref},10^{-12})$. Among its 1,200 runs, \BnBCapped\ reach the limit, on near-copies whose optima are close to zero.

\begin{table*}[tp]
\caption{\textbf{Failure rates (\%)} against the reference value on 1,560 centered $\mathrm{O}(2)$ instances of synthetic families, pooled over $n\in\{20,50,100,200\}$ (branch-and-bound, B\&B: $n\le50$). ``Copy $\sigma_{\mathrm{noise}}$'' aligns a Gaussian point set with a noisy transformed copy of itself at noise level $\sigma_{\mathrm{noise}}$.}\label{tab:failures-synth}
\begin{center}
\input{tables/failures_synthetic.tex}

\end{center}
\end{table*}

\paragraph{Convex initialization.} Grave et~al.~\cite{grave2019} and Even et~al.~\cite{even2024} initialize by minimizing $\|\mathsf G_XP-P\mathsf G_Y\|_F^2$ over doubly stochastic $P$, with $(\mathsf G_X)_{ij}=\langle x_i,x_j\rangle$ and $(\mathsf G_Y)_{ij}=\langle y_i,y_j\rangle$. For centered inputs, $\mathsf G_X\mathbf1=\mathsf G_Y\mathbf1=0$, so $P=\mathbf1\mathbf1^\top/n$ attains the value zero. The relaxed initializer can therefore minimize its objective without identifying a permutation. The Frank--Wolfe step of Even et~al.~\cite{even2024} can leave this tie toward a permutation.

\begin{table}[t]
\caption{\textbf{Mean time (ms) per pair}, 200 pairs per set, one core, including initialization. For grids, the table reports the smallest tested number of angles with at most two numerical failures. Parentheses give the number of angles or, for AM with 100 starts, the percentage of failed pairs. Bold marks the shortest time in each row.}\label{tab:tta}
\begin{center}
\input{tables/time_to_accuracy.tex}

\end{center}
\end{table}

\begin{figure*}[t]
\centering
\includegraphics[width=\textwidth]{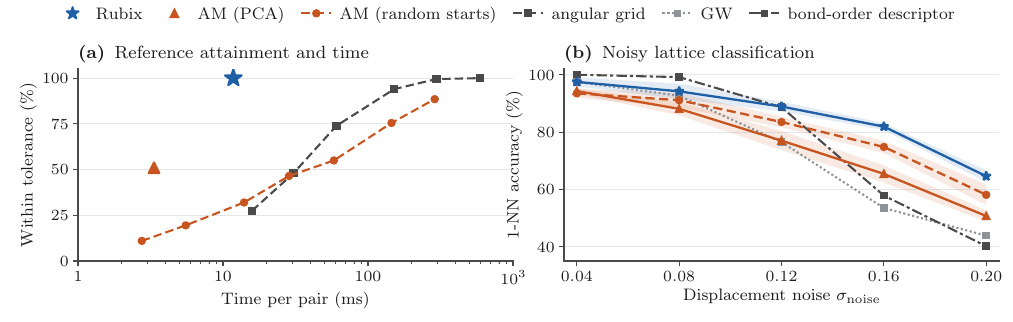}
\caption{\textbf{Accuracy and computation in planar applications.} (a)~Fraction of 200 MPEG-7 pairs that attain the numerical reference against mean time, including initialization: AM with 1--100 random starts, grids with 36--1,440 angles, AM (PCA), and Rubix. (b)~1-NN accuracy on noisy lattice neighborhoods; the solver curves and their bands of $\pm1.96$ standard errors use five draws, while GW and the bond-order descriptor use three separate draws (\Cref{app:frontier,app:where}).}\label{fig:where}
\end{figure*}

\subsection{Runtime and Failure Rates}\label{app:frontier}

Timing uses the first 200 pairs per family, including GW initialization, which contributes no assignment count. Grids use $N_{\mathrm{grid}}\in\{36,72,144,360,720,1440\}$; AM uses nested random starts (\Cref{sec:exp-fail}).

\subsection{When Global Alignment Matters}\label{app:where}

\paragraph{PCA failure modes.} We draw $n=50$ points from $\mathcal N(0,\operatorname{diag}(1,\chi_{\mathrm{cov}}))$, add covariance-scaled Gaussian noise of standard deviation $0.05$, redraw fraction $\rho_{\mathrm{redraw}}$ from the same law, and uniformly permute and rotate. Here $\chi_{\mathrm{cov}}=\lambda_{\min}/\lambda_{\max}$ measures isotropy and $\rho_{\mathrm{redraw}}=0$ retains a noisy copy. A $6\times6$ parameter grid has 150 pairs per cell. AM has more than $1\%$ excess on at most $1\%$ when $\chi_{\mathrm{cov}}=0.05$ or $\rho_{\mathrm{redraw}}=0$, but on 55--71\% when $\chi_{\mathrm{cov}}\ge0.6$ and $\rho_{\mathrm{redraw}}\ge0.75$: nearly isotropic, dissimilar sets defeat PCA initialization. Binning MPEG-7's 979,300 pairs by their larger covariance ratio gives rates $2.1\%$ at $\chi_{\mathrm{cov}}<0.2$ and $62\%$ at $\chi_{\mathrm{cov}}\ge0.8$.

\paragraph{Shape retrieval.} Rubix alone predicts 14 queries correctly and AM (PCA) alone two (\Cref{tab:mpeg7,fig:motivation}). Treating queries as independent gives exact McNemar $p=0.004$ and bullseye Wilcoxon $p<10^{-21}$.

\input{appendix/planar_results_gallery.tex}

\paragraph{Lattice neighborhoods.} Replicates redraw patches and central particles. Bands are normal-approximation means $\pm1.96$ standard errors. Rubix exceeds AM (PCA) in every draw and noise level, with pooled McNemar $p<10^{-17}$ at $\sigma_{\mathrm{noise}}=0.12$ and $0.16$.

\begin{table}[t]
\caption{\textbf{Noisy 2D lattice classification accuracy (\%).} Each column pools the same 1,000 test neighborhoods across five draws, with ten labeled neighborhoods per class and $n=24$. All four rows use 1-NN with PW distances and differ only in the alignment solver. Full noise curves appear in \Cref{fig:where}(b). Bold marks the highest accuracy in each column.}\label{tab:materials}
\centering
\input{tables/materials_results.tex}
\end{table}

\paragraph{Robustness and baselines (\Cref{tab:robust,tab:baselines}).} The bond-orientational descriptor concatenates the sorted distances from the central particle to its 23 neighbors with the magnitudes of
\begin{equation}
\psi^{\mathrm{bond}}_{\ell_{\mathrm{harm}},N_{\mathrm{nbr}}}=\frac1{N_{\mathrm{nbr}}}\sum_{j=1}^{N_{\mathrm{nbr}}}e^{\mathrm i\ell_{\mathrm{harm}}\phi_j},
\end{equation}
for the harmonic orders $\ell_{\mathrm{harm}}\in\{3,4,6\}$ and the neighbor counts $N_{\mathrm{nbr}}\in\{3,4,6,12\}$, where $\phi_j$ is the polar angle of the $j$th neighbor about the central particle. We classify by Euclidean distances between descriptors. \Cref{tab:robust} varies the size of the neighborhoods and the number of labels per class.

\begin{table*}[tp]
\begin{minipage}[t]{0.36\textwidth}
\caption{\textbf{Lattice 1-NN accuracy (\%)} at $\sigma_{\mathrm{noise}}=0.12$ by neighborhood size $n$ and labeled neighborhoods per class, pooled over three draws of 200 test neighborhoods. Rubix is more accurate in every setting (McNemar $p\le0.002$). Bold marks higher accuracy than AM.}\label{tab:robust}
\begin{center}
\input{tables/robust.tex}
\end{center}
\end{minipage}\hfill
\begin{minipage}[t]{0.6\textwidth}
\caption{\textbf{Lattice classification accuracy (\%)} of PW computed by Rubix, AM (PCA), and AM (GW), and of rotation-invariant alternatives: GW distances, SEINT~\cite{lin2026seint} in its original classification setting and in the best of its 14 tested settings, and a bond-orientational descriptor. All use 1-NN with ten labeled neighborhoods per class ($n=24$, three draws of 200 test neighborhoods). Bold marks the highest accuracy in each column.}\label{tab:baselines}
\begin{center}
\input{tables/baselines.tex}
\end{center}
\end{minipage}
\end{table*}

\subsubsection{SEINT Distance Comparisons}\label{app:seint}
Original NumPy SEINT~\cite{lin2026seint} uses the same centered, scaled inputs. Its original setting maximizes discrepancy over 50 random radius-2 references; ISEINT averages it. Fourteen settings cross these variants with 50--2,000 random references and deterministic references. We report the original and best settings, the latter deterministic ISEINT (\Cref{tab:seint-sweep}).

Clustering uses square roots of SEINT values to match $\PW$; raw values give the same conclusions, and nearest-neighbor results are unchanged. Timing uses the same 300 pairs and processor.

Both SEINT settings give fewer MPEG-7 bullseye hits than Rubix (Wilcoxon $p<10^{-179}$). Deterministic ISEINT trails on lattices at $\sigma_{\mathrm{noise}}=0.12$, $0.16$, $0.20$ (McNemar $p<10^{-20}$); \Cref{sec:where} reports its low-noise advantage.

\begin{table*}[t]
\caption{\textbf{SEINT at matched error rates.} Accuracy (\%) and median time in milliseconds per distance for 14 settings of SEINT~\cite{lin2026seint} and for Rubix, on all 1,400 MPEG-7 silhouettes and three lattice draws (\Cref{tab:baselines}). The times use the same 300 pairs and processor core. ISEINT averages the discrepancies over the references, and SEINT takes their maximum.}\label{tab:seint-sweep}
\centering
\input{tables/seint_sweep.tex}
\end{table*}

\subsubsection{Polycrystals and Clustering}\label{app:polycrystals-clustering}

\paragraph{Polycrystals (\Cref{fig:poly}).} Eight uniform sites partition a $50\times50$ square into Voronoi grains. Each lattice fills two grains at uniform random orientations and offsets. We remove cross-grain neighbors closer than $0.55$ and add Gaussian noise $\sigma_{\mathrm{noise}}=0.12$. Ten pure-lattice neighborhoods per type provide references at the same noise. Scored atoms lie at least $3.5$ from the square boundary, with all 24 nearest neighbors in their grain.

\input{appendix/materials_gallery.tex}

\paragraph{Clustering.} At $\sigma_{\mathrm{noise}}=0.08$, $0.12$, $0.16$, three draws of 50 neighborhoods per lattice form four groups by lowest-cost k-medoids over 20 starts or by average linkage. Labels serve only ARI and normalized mutual information (NMI) evaluation (\Cref{tab:clustering}).

\begin{table*}[tp]
\caption{\textbf{Unsupervised clustering of 200 lattice neighborhoods} (50 per lattice), using PW and SEINT~\cite{lin2026seint} distances. ARI and NMI measure the agreement with the true types, averaged over three draws. SEINT uses its original classification setting and its best tested setting (deterministic ISEINT; \Cref{app:seint}). Bold marks the highest score per column within each clustering method.}\label{tab:clustering}
\begin{center}
\input{tables/clustering.tex}
\end{center}
\end{table*}

\input{appendix/published_planar.tex}
\input{appendix/native_planar.tex}

%% file: appendix/computing.tex
\subsection{Hardware}\label{app:setup}\label{app:computing}

The experiments use Apple M2 and AMD Ryzen Threadripper PRO 5995WX processors. Generalized-CVO and RAP also use a 40\,GB NVIDIA A100 GPU on a node with two Xeon Gold 6330 processors. Paired solver comparisons use the same processor and one numerical thread. The SE(3) timings on real scans cover the matching solver only and reuse the native RAP poses computed on the GPU.

%% file: tables/realpairs_full.tex
\begin{tabular}{@{}lrrrrrrrrr@{}}
\toprule
 & \multicolumn{3}{c}{MPEG-7} & \multicolumn{3}{c}{Gauss.} & \multicolumn{3}{c}{lattice}\\
\cmidrule(lr){2-4}\cmidrule(lr){5-7}\cmidrule(lr){8-10}
method & LAPs & fail & $>1\%$ & LAPs & fail & $>1\%$ & LAPs & fail & $>1\%$\\
\midrule
Rubix & \textbf{27} & \textbf{0.0} & \textbf{0.0} & \textbf{30} & \textbf{0.0} & \textbf{0.0} & \textbf{31} & \textbf{0.0} & \textbf{0.0}\\
AM (PCA)~\cite{adamo2025} & 7 & 53.0 & 33.4 & 7 & 72.4 & 62.2 & 5 & 78.5 & 77.2\\
AM (10 random)~\cite{adamo2025} & 62 & 51.6 & 35.0 & 36 & 55.3 & 44.1 & 24 & 51.8 & 48.5\\
AM (GW)~\cite{alvarezmelis2018gw} & 3 & 72.3 & 61.3 & 3 & 85.7 & 80.8 & 2 & 88.7 & 87.7\\
grid 36 & 36 & 72.9 & 12.6 & 36 & 59.0 & 10.7 & 36 & 3.8 & 1.3\\
grid 360 & 360 & 4.7 & 0.0 & 360 & 0.7 & 0.0 & 360 & 0.0 & 0.0\\
\bottomrule
\end{tabular}

%% file: tables/failures_synthetic.tex
\begin{tabular}{lrrrrrr}
\toprule
method & Gauss. & mixt. & circle & copy .05 & copy .2 & copy .5\\
\midrule
AM (identity)~\cite{adamo2025} & 97 & 87 & 98 & 90 & 92 & 93\\
AM (PCA, 4 starts)~\cite{adamo2025} & 69 & 21 & 80 & 1 & 12 & 53\\
AM (10 random)~\cite{adamo2025} & 70 & 24 & 80 & 28 & 40 & 62\\
grid 36 & 48 & 52 & 37 & 2 & 38 & 44\\
grid 360 & 5 & 8 & 1 & 0 & 2 & 3\\
B\&B ($10^{-6}$) & 0 & 0 & 0 & 0 & 0 & 0\\
\bottomrule
\end{tabular}

%% file: tables/time_to_accuracy.tex
\begin{tabular}{@{}lrrr@{}}
\toprule
Set & Rubix & Grid ($N_{\mathrm{grid}}$) & AM (100)~\cite{adamo2025}\\
\midrule
MPEG-7 & \textbf{11.8} & 296 (720) & 288 (11.5\%)\\
Gauss. & \textbf{8.3} & 82.6 (360) & 99.0 (6.5\%)\\
lattice & \textbf{2.2} & 3.9 (72) & 21.2 (2.5\%)\\
\bottomrule
\end{tabular}

%% file: appendix/planar_results_gallery.tex
\begin{figure*}[tp]
\centering
\includegraphics[width=\textwidth]{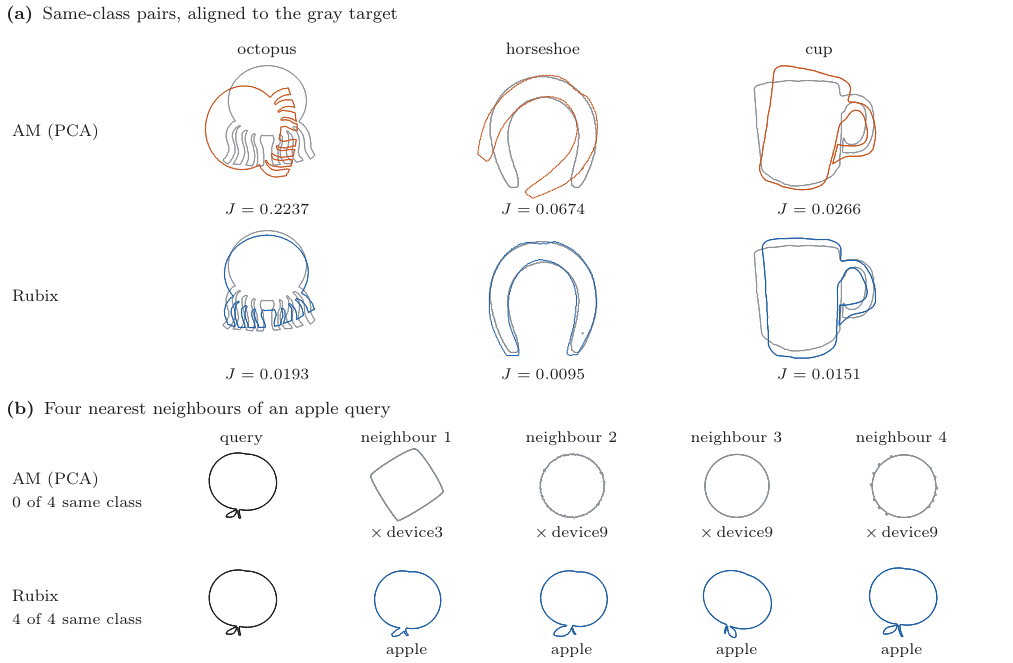}
\caption{\textbf{MPEG-7 alignment and retrieval.} The silhouettes are aligned from 64 boundary points, and the full outlines are shown. (a)~Three pairs of the same class aligned to the gray target by AM (PCA, orange) and by Rubix (blue), with the mean squared PW costs $J$ and one scale per class. (b)~A selected apple query and its four nearest neighbors in their computed orientations; $\times$ marks shapes of other classes. AM and Rubix retrieve 0 of 4 and 4 of 4 same-class neighbors here; \Cref{tab:mpeg7} evaluates all queries.}\label{fig:motivation}
\end{figure*}

%% file: tables/materials_results.tex
\begin{tabular}{@{}lrrr@{}}
\toprule
 & \multicolumn{3}{c}{Noise $\sigma_{\mathrm{noise}}$}\\
PW solver & 0.08 & 0.12 & 0.16\\
\midrule
AM (PCA)~\cite{adamo2025} & 88.1 & 77.0 & 65.4\\
AM (10 starts)~\cite{adamo2025} & 91.1 & 83.5 & 74.8\\
Grid 36 & \textbf{94.2} & \textbf{88.9} & \textbf{82.0}\\
Rubix & \textbf{94.2} & \textbf{88.9} & 81.9\\
\bottomrule
\end{tabular}

%% file: tables/robust.tex
\begin{tabular}{@{}rrrr@{}}
\toprule
$n$ & \shortstack{Labels\\per class} & Rubix & AM (PCA)~\cite{adamo2025}\\
\midrule
12 & 10 & \textbf{83.5} & 80.2\\
18 & 10 & \textbf{90.3} & 83.0\\
24 & 5 & \textbf{83.5} & 72.0\\
24 & 10 & \textbf{91.5} & 77.8\\
24 & 20 & \textbf{90.8} & 83.5\\
36 & 10 & \textbf{94.5} & 79.2\\
\bottomrule
\end{tabular}

%% file: tables/baselines.tex
\begin{tabular}{@{}lrrrrr@{}}
\toprule
$\sigma_{\mathrm{noise}}$ & 0.04 & 0.08 & 0.12 & 0.16 & 0.20\\
\midrule
Rubix & 97.8 & 92.7 & \textbf{88.7} & \textbf{82.3} & \textbf{64.0}\\
AM (PCA)~\cite{adamo2025} & 94.7 & 87.8 & 77.2 & 66.2 & 50.7\\
AM (GW)~\cite{alvarezmelis2018gw} & 93.8 & 87.3 & 72.5 & 60.7 & 45.2\\
GW~\cite{alvarezmelis2018gw,flamary2021} & 97.5 & 92.7 & 76.5 & 53.5 & 43.8\\
SEINT~\cite{lin2026seint} & 97.7 & 81.0 & 50.0 & 36.3 & 26.3\\
SEINT, best setting~\cite{lin2026seint} & \textbf{100.0} & 93.7 & 69.3 & 44.8 & 31.2\\
Descriptor (ours) & \textbf{100.0} & \textbf{99.2} & \textbf{88.7} & 57.8 & 40.2\\
\bottomrule
\end{tabular}

%% file: tables/seint_sweep.tex
\begin{tabular}{@{}lrrrrrrrrr@{}}
\toprule
 & \multicolumn{3}{c}{MPEG-7 retrieval} & \multicolumn{6}{c}{Lattice 1-NN accuracy (\%) at $\sigma_{\mathrm{noise}}$}\\
\cmidrule(lr){2-4}\cmidrule(l){5-10}
Setting & Bullseye & 1-NN & ms & 0.04 & 0.08 & 0.12 & 0.16 & 0.20 & ms\\
\midrule
SEINT, 50 refs & 36.2 & 63.1 & 0.22 & 97.7 & 81.0 & 50.0 & 36.3 & 26.3 & 0.17\\
SEINT, 100 refs & 36.0 & 62.3 & 0.28 & 97.8 & 81.3 & 51.3 & 34.5 & 26.3 & 0.20\\
SEINT, 200 refs & 35.9 & 61.8 & 0.43 & 98.0 & 82.7 & 53.5 & 34.2 & 28.2 & 0.25\\
SEINT, 500 refs & 36.7 & 63.1 & 0.78 & 99.7 & 87.5 & 55.0 & 34.5 & 29.2 & 0.40\\
SEINT, 1000 refs & 36.6 & 63.0 & 1.58 & 99.2 & 87.5 & 51.5 & 34.2 & 31.3 & 0.64\\
SEINT, 2000 refs & 37.2 & 62.0 & 2.51 & 99.0 & 84.7 & 55.3 & 35.2 & 28.3 & 1.13\\
ISEINT, 50 refs & 36.7 & 63.1 & 0.22 & 99.8 & 89.5 & 58.2 & 35.7 & 27.8 & 0.17\\
ISEINT, 100 refs & 36.9 & 63.6 & 0.28 & 99.8 & 90.2 & 57.8 & 36.2 & 26.5 & 0.20\\
ISEINT, 200 refs & 36.6 & 62.9 & 0.43 & 99.8 & 91.2 & 60.0 & 36.3 & 27.0 & 0.25\\
ISEINT, 500 refs & 36.9 & 64.0 & 0.78 & 100.0 & 90.8 & 61.0 & 37.5 & 28.0 & 0.40\\
ISEINT, 1000 refs & 36.8 & 63.1 & 1.58 & 100.0 & 91.2 & 60.8 & 37.5 & 27.8 & 0.64\\
ISEINT, 2000 refs & 36.8 & 63.3 & 2.51 & 99.8 & 91.0 & 60.2 & 37.0 & 27.8 & 1.13\\
SEINT, deterministic & 38.8 & 59.1 & 0.34 & 99.8 & 91.8 & 70.0 & 42.7 & 30.5 & 0.16\\
ISEINT, deterministic & 41.6 & 67.4 & 0.34 & 100.0 & 93.7 & 69.3 & 44.8 & 31.2 & 0.16\\
\midrule
Rubix & 67.4 & 92.5 & 5.35 & 97.8 & 92.7 & 88.7 & 82.3 & 64.0 & 1.00\\
\bottomrule
\end{tabular}

%% file: appendix/materials_gallery.tex
\begin{figure*}[tp]
\centering
\includegraphics[width=\textwidth]{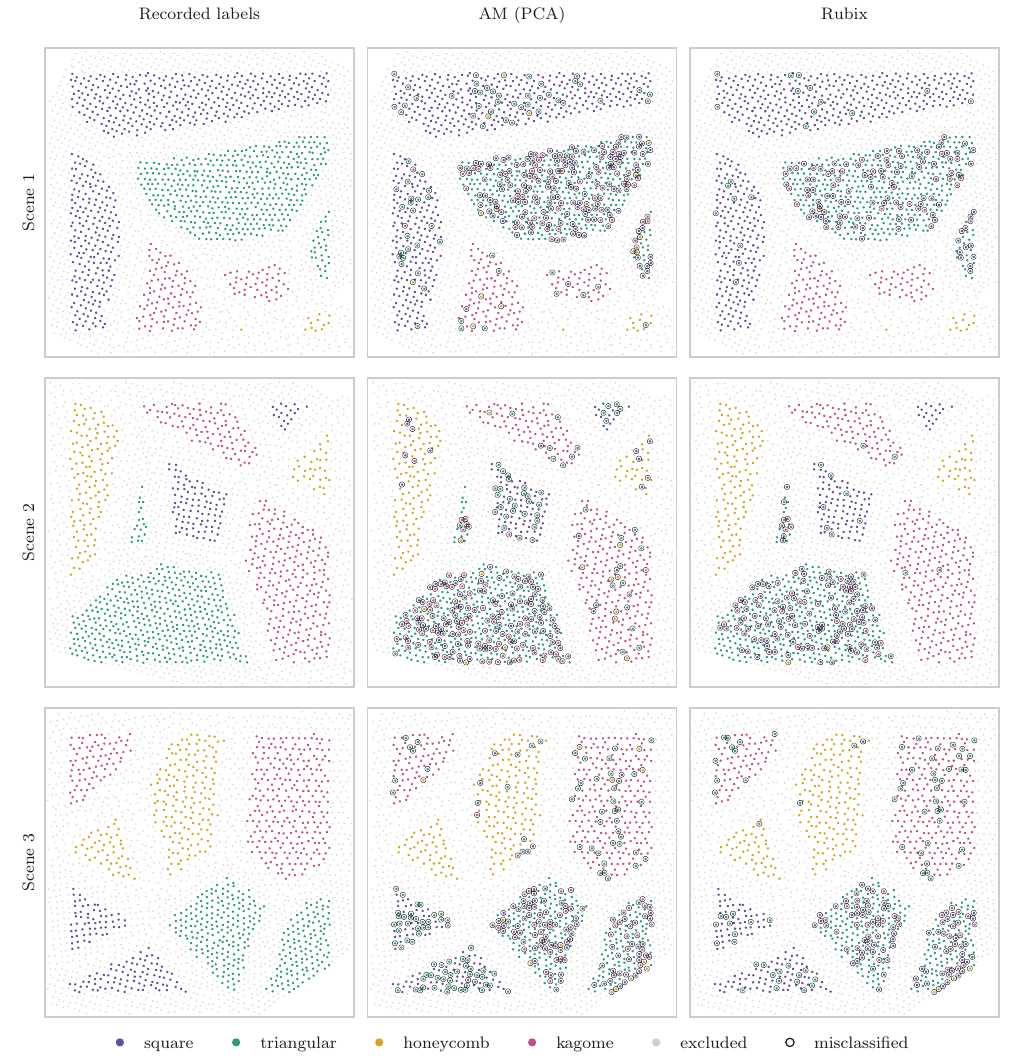}
\caption{\textbf{Fewer classification errors across three polycrystals.} The rows show the three scenes, and the columns show the true lattice labels and the predictions with AM (PCA) and with Rubix. Black rings mark errors. Gray atoms are not scored: their neighborhoods cross a grain boundary, or they lie within 3.5 units of the boundary of the domain. Both methods score the same 1,190, 1,071 and 1,006 atoms, drawn at a common physical scale. The errors fall from 285 to 139, from 257 to 149 and from 255 to 179, which is 330 fewer overall; in the first scene, Rubix corrects 186 errors of AM and introduces 40.}\label{fig:poly}\label{fig:materials-gallery}
\end{figure*}

%% file: tables/clustering.tex
\begin{tabular}{@{}llcccccc@{}}
\toprule
 & & \multicolumn{2}{c}{$\sigma_{\mathrm{noise}}=0.08$} & \multicolumn{2}{c}{$\sigma_{\mathrm{noise}}=0.12$} & \multicolumn{2}{c}{$\sigma_{\mathrm{noise}}=0.16$}\\
\cmidrule(lr){3-4}\cmidrule(lr){5-6}\cmidrule(lr){7-8}
Clustering & Distance & ARI & NMI & ARI & NMI & ARI & NMI\\
\midrule
k-medoids & Rubix & 0.55 & 0.68 & \textbf{0.52} & \textbf{0.60} & \textbf{0.39} & \textbf{0.43}\\
 & AM (PCA)~\cite{adamo2025} & 0.35 & 0.40 & 0.22 & 0.26 & 0.06 & 0.08\\
 & SEINT~\cite{lin2026seint} & 0.21 & 0.27 & 0.06 & 0.09 & 0.01 & 0.03\\
 & SEINT, best setting~\cite{lin2026seint} & \textbf{0.70} & \textbf{0.72} & 0.28 & 0.31 & 0.07 & 0.11\\
\midrule
average linkage & Rubix & \textbf{0.68} & \textbf{0.80} & \textbf{0.66} & \textbf{0.80} & \textbf{0.10} & 0.20\\
 & AM (PCA)~\cite{adamo2025} & 0.43 & 0.59 & 0.29 & 0.45 & 0.07 & \textbf{0.21}\\
 & SEINT~\cite{lin2026seint} & 0.09 & 0.22 & 0.00 & 0.06 & 0.00 & 0.04\\
 & SEINT, best setting~\cite{lin2026seint} & 0.16 & 0.31 & 0.04 & 0.09 & 0.00 & 0.03\\
\bottomrule
\end{tabular}

%% file: appendix/published_planar.tex
\subsection{Published Planar Baselines}\label{app:published-planar}

\paragraph{Fiedler--W.} Original and reproduced Fiedler--W~\cite{adamo2025} retain its graph construction, with determinant corrections restricting it to $\mathrm{SO}(2)$, 100 iterations and tolerance $10^{-10}$. Both are reported because eight of 600 outputs differ, one for an unresolved reason.

\paragraph{Other PW solvers.} Our reproduction of Ping-Pong~\cite{even2024} uses $\gamma^{\mathrm{PP}}_j=1/(j+2)$, at most 1,000 Frank--Wolfe and 100 alternating steps. Its uniform initializer is stationary, so the first assignment depends on tie-breaking (\Cref{app:solvers}). Stochastic Wasserstein--Procrustes~\cite{grave2019} uses learning rate 500, 100 initialization updates and 6,659 stochastic updates over five epochs, then assignment without AM refinement. Single-core timings include initialization and evaluation.

\paragraph{APM (\Cref{tab:published-similarity,tab:global-solvers}).} The authors' code~\cite{lian2017} runs under Octave in similarity mode, with zero regularization, distance tolerance $0.001$, batches of $2^9$ rectangles and LAPJV cost resolution $10^{-4}$. The direct comparison uses C++ Rubix at yaw gap $10^{-10}$. Both transform $Y$ toward $X$, eliminating translation and nonnegative scale on positive-variance sources. Timings include setup and language overhead.

\input{tables/published_similarity.tex}

\paragraph{RPM-PA (\Cref{tab:global-solvers}).} Our implementation follows Lian et al.~\cite{lian2021polyhedral}, eliminating planar similarity and searching the convex hull of matching statistics. Complete matching reduces the affine hull to $d_{\mathrm{RPM}}=2$ dimensions. We denote the paper's $n_u$ and $\varepsilon_0$ by $d_{\mathrm{RPM}}$ and $\varepsilon_{\mathrm{RPM}}$. Search stops when every facet's normalized support satisfies $\mu_{\mathrm{RPM}}\le1+d_{\mathrm{RPM}}\varepsilon_{\mathrm{RPM}}$, or after 120\,s. RPM-PA and Rubix share single-threaded Python, one processor core and SciPy's assignment solver.

\paragraph{Rubix support cuts inside APM.} Retaining assignment-derived support half-spaces tightens most APM rectangles without new assignments, but median assignment counts change by less than $0.7\%$ on all 600 pairs, with identical optima: the initial batch of $2^9$ rectangles dominates before pruning.

%% file: tables/published_similarity.tex
\begin{table}[t]
\centering
\caption{\textbf{Similarity alignment against original APM.} Both optimize a bijection, translation, rotation and nonnegative scale on 200 pairs per family. All objectives agree within APM's mean tolerance $10^{-6}$. Times are median milliseconds including setup; ratios are paired APM/Rubix medians. Bold marks lower Rubix times and ratios above one. Language and solver details appear in \Cref{app:published-planar}.}\label{tab:published-similarity}
\begin{tabular}{@{}lrrr@{}}
\toprule
Family & APM~\cite{lian2017} (ms) & Rubix (ms) & Ratio\\
\midrule
MPEG-7 & 259.0 & \textbf{8.12} & $\mathbf{32.4}\times$\\
Gaussian & 191.6 & \textbf{4.41} & $\mathbf{43.7}\times$\\
Lattice & 144.7 & \textbf{0.74} & $\mathbf{194.6}\times$\\
\bottomrule
\end{tabular}
\end{table}

%% file: appendix/native_planar.tex
\subsection{Planar Laser-Scan Registration}\label{app:native-planar}

\paragraph{Data and settings.} Original GLORES~\cite{glores2020} and PLICP~\cite{censi2008} use IILABS~\cite{iilabs2025} sequences \texttt{slippage}, \texttt{nav\_a\_omni}, \texttt{loop}. Scans at ten acquisition-time quantiles pair with scans one and five seconds later, yielding 60 pairs. Using 128 or 256 beams and retained fractions $\rho_{\mathrm{keep}}$ of $0.7$ or $0.9$ gives 240 configurations. The search spans all yaw and translations in $[-5,5]^2$\,m at fixed scale, with relaxation thresholds of $0.1$, tolerance $1\%$, and soft and hard limits of 120 and 150\,s. PLICP starts at zero or 16 equally spaced yaws, using centroid translation for nonzero starts and native rejection, and minimizes error per retained point. The 28 configurations without reference poses enter timing only.

\paragraph{Bounds.} We use the common-translation construction of \Cref{lem:common-translation,app:common-translation}. Tangent triangles use half-widths in $(0,\pi/3]$; below $10^{-7}$ we use the minimum vertex value. Cardinality stays fixed per cell. All variants split wide arcs into three and share roundoff safeguards, including a coordinate-scaled guard; nonimproving queries stop early. GLORES has no polynomial guarantee. When the search terminates normally, its lower bound is $0.99U$; when it is interrupted, the lower bound is $\min(L_{\rm raw},0.99U)$, which accounts for the cells that were pruned within the tolerance.

\paragraph{Timing and statistics.} Timings include input, all preprocessing and bound computation. Bootstrap resampling stays within sequences and keeps each pair's four configurations together.

%% file: appendix/unbalanced.tex
\subsection{Unbalanced Alignment Details}\label{app:unbalanced}

\subsubsection{Data, Baselines and Selection}\label{app:unbalanced-data}

\paragraph{Anchored synthetic sets.} For $n=m=40$, we draw $k=\mathrm{round}(40\rho_{\mathrm{share}})$ shared and $40-k$ independent points per set uniformly in the unit disk. Shared target points receive Gaussian noise $\sigma_{\mathrm{noise}}$ and a uniform random rotation about the origin, our anchor; both sets are randomly permuted. Each cell of $\rho_{\mathrm{share}}\in\{1,0.75,0.5,0.3\}$ and $\sigma_{\mathrm{noise}}\in\{0.01,0.03\}$ has 30 validation and 100 test instances. Success requires rotation error at most $5^\circ$. Soft plans yield mutual arg-max correspondences of mass at least $0.3$ of the maximum, and rotations from plan-weighted cross-covariance about the anchor. F1 compares returned and true shared pairs.

\paragraph{Defective crystal neighborhoods.} The four lattices have unit nearest-neighbor spacing and Gaussian displacement $\sigma_{\mathrm{noise}}=0.08$. Each noncentral atom is removed with probability $p_{\mathrm{def}}$. Interstitial counts are Poisson with mean $p_{\mathrm{def}}N_{\mathrm{clean}}/2$, where $N_{\mathrm{clean}}$ counts defect-free atoms within radius $2.6$; positions are uniform within that radius and at least $0.5$ from every atom. Neighborhoods contain the central site at the origin and either all atoms within $2.6$ (unbalanced) or its 23 nearest atoms (balanced), in physical units. Each class has 10 defect-free references at the same noise, 15 validation queries and 50 tests per $p_{\mathrm{def}}\in\{0,0.1,0.2,0.3\}$.

\paragraph{Real laser scans.} We reuse the 240 configurations of \Cref{sec:native-planar}; success on the 212 motion-capture references requires errors at most $0.25$\,m and $2^\circ$. GW methods normalize within-scan distances by the target's median pairwise distance and fit plan-weighted rigid Procrustes poses. Optional refinement uses 50 trimmed point-to-point ICP iterations at the prescribed retained fraction; zero-pose ICP controls for refinement alone.

\paragraph{Methods and settings.} Rubix-U applies \Cref{alg:gravity-yaw} to \eqref{eq:partial-spatial} with zero heights, anchor-centered rotation, absolute gap $10^{-6}\lambda(n+m)$, and limits of 20,000 assignments and 10\,s. AM-U alternates penalized assignment and rotation from 16 equally spaced yaws, retaining the best objective. Rubix-B optimizes the complete anchored bijection. These variants separate optimization from partial matching, except that lattice Rubix-B also changes neighborhood extraction.

UGW~\cite{sejourne2021} uses the authors' log-domain solver and two-plan cost, Euclidean internal distances and masses $1/n_0$, with $n_0=40$, $24$, $100$ for synthetic sets, lattices and scans. Fused UGW~\cite{thual2022fugw} uses POT's~\cite{flamary2021} majorization--minimization solver, zero entropy and anchor cost $(\|x_i\|-\|y_j\|)^2$. Partial GW~\cite{chapel2020partial} uses POT's Frank--Wolfe solver; SEINT uses its best setting (\Cref{app:seint}).

For the synthetic sets, the tested penalties are $\lambda\in\{2,5,10,20,50,100\}\times10^{-4}$; for lattices, we use $\{0.01,0.02,0.05,0.1,0.2\}$. The UGW grid uses entropy and marginal penalties $(\varepsilon_{\mathrm{UGW}},\rho_{\mathrm{UGW}})\in\{10^{-3},10^{-2}\}\times\{10^{-2},10^{-1},1,10\}$. For fused UGW, we test $(\rho_{\mathrm{FUGW}},\alpha_{\mathrm{FUGW}})\in\{10^{-2},10^{-1},1\}\times\{0.3,1,3\}$ for the marginal penalty and feature weight; neither of these penalties is the shared fraction $\rho_{\mathrm{share}}$. For partial GW, we obtain matched masses by multiplying $\min(n,m)/n_0$ by $\{0.3,0.5,0.75,1\}$ for the synthetic sets, $\{0.5,0.7,0.85,1\}$ for the lattices and $\{0.5,0.7,0.9\}$ for the laser scans. Validation and test use the same grids.

\paragraph{Selection and tests.} Each synthetic or lattice cell selects the highest validation success or accuracy, breaking ties by grid order. Tables place each baseline's best test value beside its validation-selected result. Without scan validation data, GW uses its best test setting. Paired McNemar tests compare Rubix-U with the best other method.

\subsubsection{Failure Diagnostic and Example Selection}\label{app:unbalanced-diagnostic}

\paragraph{Solver and objective failures.} Each GW method restarts at its selected setting with mass $1/n_0$ per true pair; partial GW rescales to its prescribed mass, never above the true mass at selected settings. A solver failure requires successful rotation recovery with an objective more than a relative $10^{-6}$ below the failed run. Otherwise, the label ``objective failure'' means no witness was found, not that a nonconvex GW optimum was established. AM has a solver failure when Rubix closes its gap at the same penalty and recovers the rotation. Both Rubix variants attain costs no greater than the true matching at the true rotation, so their misses are objective failures.

\paragraph{Qualitative examples.} \Cref{fig:unbalanced-examples} follows the first-case rule in its caption, excluding the two ETH pairs with fused-UGW numerical errors. The teaser uses the first correctly classified query per lattice at noise $0.08$ and defect rate $0.2$. Settings follow the selection above.

\paragraph{Numerical adaptations.} SEINT requires equal sizes and uses 24-point neighborhoods. Partial GW requests $(1-10^{-9})$ times nominal mass to satisfy POT's floating-point feasibility test at fraction one. Numerical failures count as misses; lattice timings are comparable only within a cell.

\subsubsection{Detailed Planar and Scan Results}\label{app:unbalanced-planar-results}

\Cref{fig:unbalanced} summarizes all conditions; \Cref{tab:unbalanced-full,tab:unbalanced-full-lattice} retain every cell.

\begin{table*}[t]
\centering
\caption{\textbf{Synthetic partial-overlap results.} Rotation recovery within $5^\circ$ (\%), with 100 tests per overlap/noise cell. Entries use validation-selected settings; parentheses give each baseline's best test value. The last column compares Rubix-U with the best competitor by McNemar's test.}\label{tab:unbalanced-full}
\input{tables/unbalanced_full.tex}
\end{table*}

\begin{table*}[t]
\centering
\caption{\textbf{Defective-lattice results.} 1-NN accuracy (\%) on 200 queries per defect/noise cell against defect-free references. Entries use validation-selected settings; parentheses give each baseline's best test value. SEINT compares 24-atom neighborhoods. The last column compares Rubix-U with the best competitor by McNemar's test.}\label{tab:unbalanced-full-lattice}
\input{tables/unbalanced_full_lattice.tex}
\end{table*}

\paragraph{Additional planar results.} Rubix closes all 4,800 gaps over the six penalties. AM exceeds the optimal objective found by Rubix on 95--589 of the 800 instances per penalty, and by more than 1\% on 30--522 (\Cref{tab:unbalanced-diagnosis}).

Validation-F1 selection gives test F1 ranges of 0.75--1.00 for Rubix, 0.59--0.99 for AM and 0.05--1.00 for GW. Rotation-selected penalties omit some noisy true pairs, reducing Rubix's F1 to 0.35--0.92. Single-core times are 10--18\,ms (Rubix), 2--3\,ms (AM), 2--4\,ms (partial GW), 32--88\,ms (UGW) and 58--146\,ms (fused UGW).

Without ICP, UGW and partial GW register 168 and 138 scan configurations. Sixteen-start native PLICP registers 185, statistically indistinguishable from GLORES with Rubix ($p=0.55$). Timings retain Python UGW's and C++ GLORES' respective implementation costs (\Cref{tab:unbalanced-scans}).

\paragraph{Scaling and real 3D subproblems.} At overlap $0.5$ and noise $0.01$, using settings selected for 40 points, Rubix recovers all 50 tests at sizes 20, 40, 80, 160 with closed gaps in median times 0.014, 0.025, 0.12, 0.43\,s. AM recovers 94--100\% in 0.002--0.18\,s, fused UGW 74--86\% in 0.024--5.4\,s, UGW 42--46\% in 0.067--1.4\,s and partial GW 10--22\% in 0.002--0.33\,s. Fitting yaw to GW plans and evaluating exact assignment on the six ETH subproblems (\Cref{sec:exp-eth}), best-setting UGW, fused UGW and partial GW reach the reference on three, four and two, respectively. AM refinement raises each to six; several other settings still fail.

\subsubsection{Full 3D Rotation and Penalty Sensitivity}\label{app:unbalanced-3d}

Full-rotation tests use $n=m=24$ points in the unit ball, overlap $\rho_{\mathrm{share}}\in\{1,0.75,0.5,0.3\}$, uniform rotations, Gaussian noise $0.01$ and independent nonshared points, with 20 validation and 50 tests per fraction. Anchored quaternion search receives 15\,s; AM alternates penalized assignment and Kabsch fitting from 60 icosahedral rotations. Both use penalties $\{5\times10^{-4},2\times10^{-3},10^{-2},5\times10^{-2},0.2\}$ and clipped assignment (\Cref{app:published-gravity}); GW retains its planar grids. Success requires geodesic error at most $5^\circ$. At selected penalties, Rubix closes all 200 gaps in median times 0.42--1.12\,s per fraction and at most 2.6\,s (\Cref{tab:unbalanced-3d}).

\begin{table*}[t]
\centering
\caption{\textbf{Partial overlap under full 3D rotation.} Rotations recovered within $5^\circ$ (\%) from 24-point sets that share the given fraction of points (50 test instances per column), with validation-selected settings. The last column is the median, over the four fractions, of the median time in seconds per solve on one core. Bold marks the highest recovery rates.}\label{tab:unbalanced-3d}
\input{tables/unbalanced_3d.tex}
\end{table*}

\paragraph{Penalty sensitivity.} A pair can enter an optimal matching only if its squared distance is at most $2\lambda$. Small penalties therefore omit true pairs, and large penalties admit false ones. Across tested penalties $2\times10^{-3}$--$5\times10^{-3}$, planar Rubix closes every gap and recovers at least 99\% at each overlap; 3D Rubix recovers every rotation from $5\times10^{-4}$ to $10^{-2}$ (\Cref{tab:unbalanced-lambda}). At $\lambda=0.2$, all 3D gaps close and misses are objective failures; at $\lambda=0.05$, nine 30\%-overlap searches remain open after 15\,s. At the smallest penalty, AM recovers only 36.5--52.5\% in 2D and 2--14\% in 3D.

\begin{table*}[t]
\centering
\caption{\textbf{Sensitivity to the unmatched-point penalty.} Rotation recovery within $5^\circ$ (\%) for Rubix and AM at each $\lambda$, using the same penalized objective: 200 planar tests per overlap fraction, pooled over noise, and 50 full-3D tests per fraction.}\label{tab:unbalanced-lambda}
\input{tables/unbalanced_lambda.tex}
\end{table*}

\subsubsection{Real ETH Scan Details and Results}\label{app:unbalanced-eth}

ETH uses one fixed pair of coresets, with 64 source and 80 target points, per scan pair and $0.3$\,m voxelization (\Cref{app:published-gravity}).

Translation is the reference or a $0.5$\,m horizontal offset in a uniform random direction. Rubix and 16-start AM use $\lambda=1$\,m$^2$; fused UGW receives height and horizontal radius about that translation. GW uses its best test setting, counting fused UGW's two numerical failures as misses. Yaw success requires $5^\circ$. Overlap measures are the reference-pose source-voxel fraction within $0.3$\,m of a target voxel and the number of penalized coreset matches within $\sqrt{2\lambda}=1.41$\,m (\Cref{fig:eth-overlap}). Rubix closes all 62 gaps.

\paragraph{Front-end translations.} Gravity-projected translations from Li et al.~\cite{li2024screw}, RAP~\cite{pan2026rap} and Quatro~\cite{lim2022quatro} lie within $0.5$\,m on 23, 19 and 4 pairs. GW retains reference-translation settings. Rubix closes 31/31 gaps per front end and leads other matchers after Li et al. and RAP, but trails their native yaws. Gains over UGW and partial GW are significant ($p\le1.3\times10^{-4}$); those over fused UGW and AM are not ($p\ge0.10$; \Cref{tab:unbalanced-eth}).

\paragraph{Gravity search.} Our Python port of Li et al.'s MATLAB code~\cite{li2023gravity} searches translation by branch-and-bound, maximizing one-to-one matches whose gravity angles agree within $\varepsilon_{\mathrm{Li}}$, then votes for yaw. It receives no translation, making this a contextual comparison.

The translation box spans the ETH scans, with 60\,s per solve and $\varepsilon_{\mathrm{Li}}=0.5^\circ$ (published), $1^\circ$, $2^\circ$, $4^\circ$. Every setting recovers at most two yaws within $5^\circ$ and no pose within $0.5$\,m. At $2^\circ$ and $4^\circ$, 27 and 31 gaps close, and the median optimum matches all 64 sources: similar gravity angles saturate the count even at incorrect translations. At $0.5^\circ$ and $1^\circ$, at most four gaps close.

\begin{table}[t]
\centering
\caption{\textbf{Partial overlap on real 3D scans.} Yaw recoveries within $5^\circ$ on the identified ETH pairs, given reference, offset or front-end translations. GW settings are selected at the reference translation; AM uses 16 starts. Native Li et al., RAP and Quatro recover 25, 25 and 6 yaws. Bold marks the most recoveries.}\label{tab:unbalanced-eth}
\input{tables/unbalanced_eth.tex}
\end{table}

\begin{table*}[t]
\centering
\caption{\textbf{Failure diagnostic by overlap.} Solver/objective failures among 200 tests per fraction (\Cref{app:unbalanced-diagnostic}). A GW objective failure means no better correct-rotation witness was found. Both Rubix variants close every gap.}\label{tab:unbalanced-diagnosis}
\input{tables/unbalanced_diagnosis.tex}
\end{table*}

\begin{table}[t]
\centering
\caption{\textbf{Registration of real laser scans with partial overlap.} Pose successes within $0.25$\,m and $2^\circ$ among the 212 configurations with a reference, and median seconds per configuration. GLORES~\cite{glores2020} and PLICP~\cite{censi2008} use the native C++ implementations of \Cref{sec:native-planar}; UGW~\cite{sejourne2021} and partial GW~\cite{chapel2020partial} run in Python at their best setting on these configurations, with optional refinement by trimmed ICP. Bold marks the most successes.}\label{tab:unbalanced-scans}
\input{tables/unbalanced_scans.tex}
\end{table}

\begin{figure*}[tp]
\centering
\includegraphics[width=\textwidth]{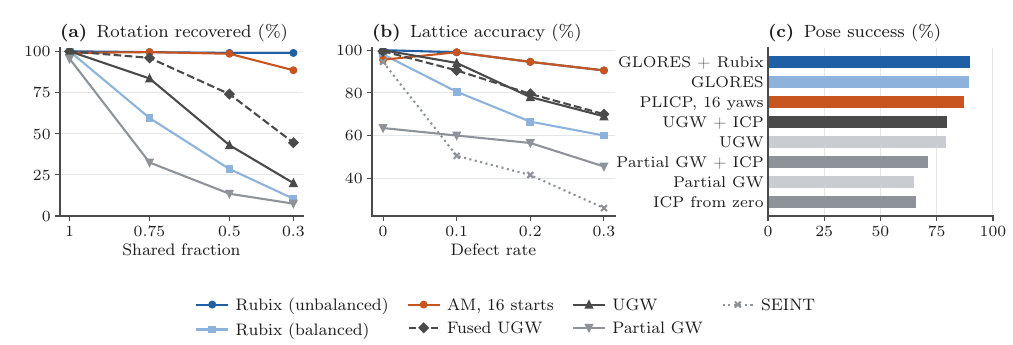}
\caption{\textbf{Unbalanced alignment across all conditions.} (a)~Rotations recovered within $5^\circ$ against the shared fraction, over 200 synthetic tests per fraction. (b)~Lattice 1-NN accuracy against the defect rate at thermal noise $0.08$. (c)~Pose success on the 212 laser-scan configurations with a reference. The settings are validation-selected in (a,b); the GW pipelines use their best test setting in (c).}\label{fig:unbalanced}
\end{figure*}

%% file: tables/unbalanced_full.tex
\begin{tabular}{@{}llrrrrrrr@{}}
\toprule
$\rho_{\mathrm{share}}$ & $\sigma_{\mathrm{noise}}$ & Rubix-U & Rubix-B & AM-U & UGW & FUGW & PGW & $p$\\
\midrule
1 & 0.01 & 100 & 100 & 99 (100) & 100 (100) & 100 (100) & 98 (100) & 1.00\\
1 & 0.03 & 100 & 100 & 99 (100) & 100 (100) & 100 (100) & 93 (93) & 1.00\\
0.75 & 0.01 & 100 & 61 & 100 (100) & 83 (83) & 97 (97) & 38 (38) & 1.00\\
0.75 & 0.03 & 99 & 58 & 99 (99) & 84 (84) & 95 (95) & 27 (29) & 1.00\\
0.5 & 0.01 & 100 & 29 & 98 (98) & 49 (49) & 86 (86) & 13 (15) & 0.50\\
0.5 & 0.03 & 98 & 28 & 99 (99) & 37 (37) & 62 (62) & 14 (14) & 1.00\\
0.3 & 0.01 & 100 & 12 & 95 (95) & 17 (17) & 47 (47) & 9 (9) & 0.06\\
0.3 & 0.03 & 98 & 9 & 82 (82) & 23 (23) & 42 (42) & 6 (8) & $1.4\cdot10^{-4}$\\
\bottomrule
\end{tabular}

%% file: tables/unbalanced_full_lattice.tex
\begin{tabular}{@{}llrrrrrrrr@{}}
\toprule
$p_{\mathrm{def}}$ & $\sigma_{\mathrm{noise}}$ & Rubix-U & Rubix-B & AM-U & UGW~\cite{sejourne2021} & FUGW~\cite{thual2022fugw} & PGW~\cite{chapel2020partial} & SEINT~\cite{lin2026seint} & $p$\\
\midrule
0 & 0.08 & 100 & 98 & 95.5 (100) & 100 (100) & 99.5 (100) & 63.5 (63.5) & 94.5 & 1.00\\
0.1 & 0.08 & 99 & 80.5 & 99 (99) & 94 (94) & 90.5 (91.5) & 60 (60) & 50.5 & 1.00\\
0.2 & 0.08 & 94.5 & 66.5 & 94.5 (94.5) & 78 (78) & 79.5 (79.5) & 56.5 (56.5) & 41.5 & 1.00\\
0.3 & 0.08 & 90.5 & 60 & 90.5 (90.5) & 69 (69) & 70 (70) & 45.5 (57) & 26 & 1.00\\
\bottomrule
\end{tabular}

%% file: tables/unbalanced_3d.tex
\begin{tabular}{@{}lrrrrrr@{}}
\toprule
& \multicolumn{4}{c}{Shared fraction} & & \\
\cmidrule(lr){2-5}
Method & 1 & 0.75 & 0.5 & 0.3 & All & Median s\\
\midrule
Rubix, unbalanced & \textbf{100.0} & \textbf{100.0} & \textbf{100.0} & \textbf{100.0} & \textbf{100.0} & 0.505\\
AM, 60 starts & \textbf{100.0} & 78.0 & 56.0 & 34.0 & 67.0 & 0.018\\
UGW~\cite{sejourne2021} & \textbf{100.0} & 82.0 & 16.0 & 0.0 & 49.5 & 0.053\\
Fused UGW~\cite{thual2022fugw} & \textbf{100.0} & 94.0 & 86.0 & 38.0 & 79.5 & 0.016\\
Partial GW~\cite{chapel2020partial} & \textbf{100.0} & 36.0 & 12.0 & 2.0 & 37.5 & 0.002\\
\bottomrule
\end{tabular}

%% file: tables/unbalanced_lambda.tex
\begin{tabular}{@{}lrrrrrrrr@{}}
\toprule
& \multicolumn{4}{c}{Rubix, unbalanced} & \multicolumn{4}{c}{AM, 16 starts}\\
\cmidrule(lr){2-5}\cmidrule(l){6-9}
Planar, $10^4\lambda$ & 1 & 0.75 & 0.5 & 0.3 & 1 & 0.75 & 0.5 & 0.3\\
\midrule
2 & 100 & 97.5 & 88.5 & 74.5 & 52.5 & 47 & 43 & 36.5\\
5 & 100 & 99.5 & 99 & 88 & 66 & 57.5 & 52.5 & 50\\
10 & 100 & 100 & 100 & 94.5 & 84.5 & 74.5 & 62 & 59.5\\
20 & 100 & 100 & 100 & 99 & 93 & 86 & 79.5 & 72.5\\
50 & 100 & 100 & 100 & 99 & 99 & 98 & 95.5 & 87.5\\
100 & 100 & 100 & 100 & 92.5 & 100 & 99.5 & 98.5 & 87.5\\
\midrule
Full 3D, $10^4\lambda$ & 1 & 0.75 & 0.5 & 0.3 & 1 & 0.75 & 0.5 & 0.3\\
\midrule
5 & 100 & 100 & 100 & 100 & 14 & 2 & 6 & 2\\
20 & 100 & 100 & 100 & 100 & 40 & 32 & 22 & 6\\
100 & 100 & 100 & 100 & 100 & 86 & 78 & 64 & 34\\
500 & 100 & 100 & 84 & 12 & 100 & 84 & 56 & 10\\
2000 & 100 & 54 & 8 & 0 & 100 & 52 & 0 & 0\\
\bottomrule
\end{tabular}

%% file: tables/unbalanced_eth.tex
\begin{tabular}{@{}>{\raggedright\arraybackslash}p{0.75in}rrrrr@{}}
\toprule
& \multicolumn{2}{c}{Reference} & \multicolumn{3}{c}{Front end}\\
\cmidrule(lr){2-3}\cmidrule(l){4-6}
Method & Exact & Offset & \shortstack{Li\\\cite{li2024screw}} & \shortstack{RAP\\\cite{pan2026rap}} & \shortstack{Quatro\\\cite{lim2022quatro}}\\
\midrule
Rubix & \textbf{27} & \textbf{25} & \textbf{20} & \textbf{23} & 8\\
AM & 22 & 20 & 17 & 17 & 5\\
UGW~\cite{sejourne2021} & 10 & 9 & 6 & 7 & 7\\
Fused UGW~\cite{thual2022fugw} & 23 & 19 & 19 & 20 & \textbf{10}\\
Partial GW~\cite{chapel2020partial} & 6 & 7 & 5 & 6 & 6\\
\bottomrule
\end{tabular}

%% file: tables/unbalanced_diagnosis.tex
\begin{tabular}{@{}lrrrrr@{}}
\toprule
& \multicolumn{4}{c}{Shared fraction} & \\
\cmidrule(lr){2-5}
Method & 1 & 0.75 & 0.5 & 0.3 & All\\
\midrule
Rubix & 0/0 & 0/1 & 0/2 & 0/2 & 0/5\\
Balanced Rubix & 0/0 & 0/81 & 0/143 & 0/179 & 0/403\\
AM, 16 starts & 2/0 & 1/0 & 3/0 & 20/3 & 26/3\\
UGW & 0/0 & 33/0 & 101/13 & 73/87 & 207/100\\
Fused UGW & 0/0 & 8/0 & 49/3 & 6/105 & 63/108\\
Partial GW & 9/0 & 135/0 & 173/0 & 184/1 & 501/1\\
\bottomrule
\end{tabular}

%% file: tables/unbalanced_scans.tex
\begin{tabular}{@{}lrr@{}}
\toprule
Method & Success & Median s\\
\midrule
GLORES~\cite{glores2020} + Rubix & \textbf{190} & 1.61\\
GLORES~\cite{glores2020} & 189 & 2.33\\
PLICP~\cite{censi2008}, 16 yaws & 185 & 0.23\\
PLICP~\cite{censi2008} from zero & 146 & 0.01\\
UGW~\cite{sejourne2021} + ICP & 169 & 0.86\\
Partial GW~\cite{chapel2020partial} + ICP & 151 & 0.37\\
ICP from zero & 139 & 0.01\\
UGW~\cite{sejourne2021} & 168 & 0.86\\
Partial GW~\cite{chapel2020partial} & 138 & 0.33\\
\bottomrule
\end{tabular}

%% file: appendix/registration_experiments.tex
\input{appendix/eth_experiments.tex}

\subsection{Spatial Registration Details}\label{app:registration-experiments}

Reference poses serve only evaluation; all gaps are numerical. \Cref{app:computing} describes the hardware.

\subsubsection{Yaw Search across Matching Models}\label{app:gravity-subsolver-protocol}

At each size in \Cref{sec:exp-gravity}, scenes cross target/source ratios $1$ and $1.25$, unmatched fractions $0.1$, $0.3$, $0.5$, vertical standard deviations $0.05$, $0.5$, and two replicates with noise $0.005$. Complete matching centers the source and $n$ targets selected by deterministic farthest-point sampling. Partial matching retains all points at the shared translation, with $\lfloor\min(n,m)/2\rfloor$ pairs or $\lambda=0.02$, independent of true overlap. Bounds share initial yaw and feasible updates. Limits are 20,000 assignments and 9.75\,s, with a 10\,s safety limit. Numerical guards are heuristic, without outward rounding.

\subsubsection{Joint Search over Yaw and Translation}\label{app:joint-gravity-protocol}

Joint scenes use vertical spreads $0.05$, $0.5$, noise $0.005$, and cardinality $\lfloor\min(n,m)/2\rfloor$ or penalty $0.02$. Bounds share eight initial yaws, four assignment--centroid alternations per yaw, cardinality branches and anchor creation. Full-face uses 24 vertices; the gated policy has no proved bound on its overhead. The million-query limit includes forced-cardinality queries requiring no assignment. Every bound recovers all 48 poses within $5^\circ$ and $0.1$\,m.

\subsubsection{Full Spatial Rotation}\label{app:se3-protocol}

The scenes of \Cref{sec:exp-se3} use the noise levels $0.05$ and $0.2$ and point distributions with the coordinate standard deviations $(1,1,1)$, $(1,0.6,0.3)$ and $(1,1,0.05)$. \Cref{fig:se3-qualitative-gallery-1} shows selected alignments.

\paragraph{Adaptive bound selection.} Additional bounds initially receive one quarter of the time spent on the cheaper search, raised to equal time when their average relative contraction reaches one half. The policy starts on an estimated 32-cell cover, then chooses by contraction per second (\Cref{tab:spatial-budget}).

\input{appendix/se3_qualitative.tex}

\subsubsection{SE(3) Matching on Subsets of Real Scans}\label{app:se3-real-protocol}

ETH scans~\cite{li2024screw} use independent 20,000-return samples, $0.3$\,m voxels and farthest-point coresets of 64 and 80 points. Exact-cardinality assignment at RAP's full proper pose~\cite{pan2026rap} selects 32 pairs. Both solvers optimize every bijection and proper rotation after centering, without gravity or an added transformation.

Both variants share all search machinery; Rubix adds exterior quaternion queries and witness bounds at every cell. Limits are one core, 15\,s and 200,000 cells, targeting a mean-squared gap of $10^{-5}$\,m$^2$. Timing covers the solver, excluding subset selection, RAP and dense refinement. Shared scans make comparisons descriptive.

The centroid control instead selects subsets at identity rotation and centroid translation, with all other settings fixed (\Cref{app:se3-real-results}). Neither test uses held-out environments.

%% file: appendix/eth_experiments.tex
\subsection{Matching Subproblems and Pipelines on ETH Scans}\label{app:eth-protocol}

Each scan in \Cref{sec:exp-eth} uses 20,000 uniformly sampled finite returns, shared across initial yaws, in meters. Nominal vertical axes differ by $0.022^\circ$--$0.643^\circ$. Reference rotations are projected to $\mathrm{SO}(3)$ for evaluation only.

\paragraph{Pipeline.} We retain normals whose absolute vertical component is below $0.8$, estimated from 12 neighbors. Penalized ranking uses $\lambda=1$\,m$^2$; partial refinement allows 12 iterations and two seconds. Four-degree-of-freedom gravity ICP~\cite{kubelka2022} uses 20-neighbor target normals, $85\%$ trimming, 100 iterations and ten seconds. The 16-yaw comparison uses centroid-optimal translation and mean tolerance $10^{-5}$\,m$^2$, equivalent to summed tolerance $0.00032$\,m$^2$ over 32 points, with heuristic roundoff guards. AM uses the same attainment tolerance but supplies no lower bound.

\input{tables/eth_pipeline_results.tex}

\paragraph{Complete systems for context.} Gravity ICP~\cite{kubelka2022} selects zero yaw or the best of 16 starts by trimmed point-to-plane residual. Generalized-CVO~\cite{zhang2026gcvo} selects among 16 yaws by geometry-only similarity; RAP~\cite{pan2026rap} uses pretrained mini-SpinNet features, one generation and ten flow steps. Selection never uses reference poses. At $5^\circ$ and $0.5$\,m, one-start ICP, 16-start ICP, Generalized-CVO and RAP recover 2, 4, 0 and 4 of six poses, with median complete times 0.165, 1.32, 11.2 and 13.4\,s. Inputs and hardware differ.

\input{appendix/gravity_matching_results.tex}
\input{appendix/published_gravity.tex}

\input{appendix/native_gravity_comparison.tex}

\input{appendix/native_extensions.tex}

%% file: tables/eth_pipeline_results.tex
\begin{table*}[t]
\centering
\caption{\textbf{Complete pipelines on six ETH pairs.} Single-core times include all stages on identical observations (\Cref{app:eth-protocol}). Solvers return identical poses; success requires errors at most $5^\circ$ and $0.5$\,m. Bold marks Rubix times below the vertex bound's.}\label{tab:eth-pipeline}
\begin{tabular}{lrrrrrr}
\toprule
Pair & AM (ms) & Edge (ms) & Vertex (ms) & Rubix (ms) & Rotation ($^\circ$) & Translation (m)\\
\midrule
Office 1--2 & 175.9 & 266.9 & 176.1 & 176.6 & 0.546 & 0.124\\
Office 2--3 & 205.2 & 300.2 & 225.9 & \textbf{200.0} & 86.237 & 7.254\\
Office 3--4 & 142.9 & 227.7 & 142.4 & 143.1 & 179.164 & 2.894\\
Courtyard 1--2 & 181.5 & 368.1 & 182.0 & 182.4 & 0.046 & 0.037\\
Courtyard 2--3 & 260.6 & 497.1 & 266.7 & 273.4 & 0.909 & 1.365\\
Courtyard 3--4 & 179.3 & 392.6 & 185.2 & \textbf{182.5} & 0.047 & 0.029\\
\bottomrule
\end{tabular}
\end{table*}

%% file: appendix/gravity_matching_results.tex
\begin{figure*}[tp]
\centering
\includegraphics[width=\textwidth]{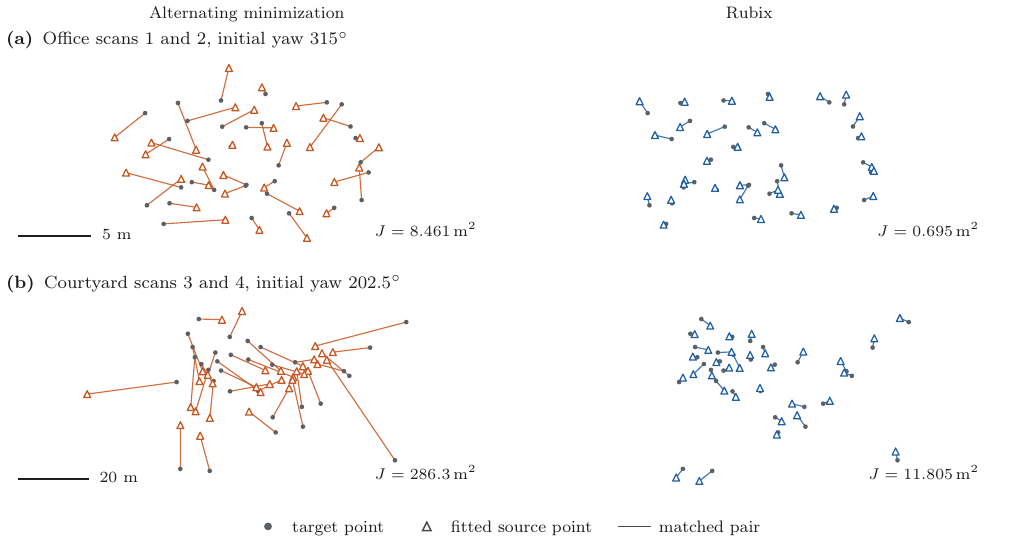}
\caption{\textbf{Gravity-aligned matching on ETH scans.} AM and Rubix start from the same initial yaw on 32-point subsets, with the nominal vertical of the scanner as the gravity direction. From AM to Rubix, the mean squared 3D residual falls from $8.461$ to $0.695$\,m$^2$ in the Office example and from $286.3$ to $11.805$\,m$^2$ in the Courtyard example. Each row shows every point and every matched edge from one viewpoint and at one scale, with the heights retained. The examples are selected for large errors of AM; \Cref{tab:eth-blocks} reports all 96 initializations.}\label{fig:gravity-matching-results}
\end{figure*}

%% file: appendix/published_gravity.tex
\subsection{Published Registration Methods as Front Ends}\label{app:published-gravity}

We use all identified ETH pairs of Li et al.~\cite{li2024screw} from Arch, Courtyard, Facade, Office and Trees. Pairs share scans; nominal scanner vertical supplies gravity.

\paragraph{Native methods.} Quatro~\cite{lim2022quatro} runs from its original reimplementation with $0.3$\,m voxels, normal and FPFH radii of $0.5$ and $0.75$\,m, respectively, and a noise bound of $0.3$\,m. RAP~\cite{pan2026rap} uses pretrained mini-SpinNet in two fixed trials. Li et al.'s original solver~\cite{li2024screw} runs unchanged under Octave; its unavailable front end is reconstructed with $0.1$\,m voxels, ISS radii of $0.6$ and $0.4$\,m, eigenvalue ratios of $0.975$, and normal and FPFH radii of $0.2$ and $0.5$\,m, respectively. It supplies $N_{\mathrm{cand}}=10$ targets per source. Sensitivity tests use $N_{\mathrm{cand}}=1$ and $5$.

\paragraph{Matching stage.} Farthest-point sampling of $0.3$\,m voxels selects 64 sources and 80 targets. Each of 17 subset problems receives one second and 200,000 assignments; multistart adds 16 uniform yaws. All variants share C++ host and assignment code. Penalized assignment of the smaller set uses the costs $\min(c_{ij}-2\lambda,0)$, whose negative entries give an optimal partial matching. Ranking uses $\lambda=1$\,m$^2$, and gravity ICP uses at most 6,000 vertical-surface points and ten seconds.

\paragraph{Native poses.} Rotations must have orthogonality and determinant errors at most $10^{-2}$ before projection to the nearest proper rotation. The incumbent retains native translation and yaw $\theta=\operatorname{atan2}(R_{10}-R_{01},R_{00}+R_{11})$, or zero in the degenerate case. Stopping adds a $2.63\times10^{-10}$--$2.30\times10^{-8}$\,m$^2$ guard to the summed target $0.00032$\,m$^2$.

\paragraph{Timing and success.} Total times include native startup and inference, preprocessing, matching, ranking and ICP; paired stages share native outputs and processor. Success requires $5^\circ$ and $0.5$\,m, whereas Li et al.~\cite{li2024screw} report continuous errors and scene-wise arithmetic means.

\input{tables/gravity_published_insertion.tex}

%% file: tables/gravity_published_insertion.tex
\begin{table*}[t]
\centering
\caption{\textbf{A common matching stage after three front ends.} Every global solver closes all subset gaps within one second each at mean-squared tolerance $10^{-5}$\,m$^2$ plus roundoff allowance. LAPs are total assignments; times are medians, including all stages in ``Total.'' Projected native poses solve no subset problem; AM supplies no lower bound. Bold marks reductions from the vertex bound. Success requires $5^\circ$ and $0.5$\,m; native success counts are 4/31, 37/62 and 23/31 (\Cref{app:published-gravity}).}\label{tab:gravity-published-insertion}
\begin{tabular}{@{}lrrrrr@{}}
\toprule
Refinement after the native method & Closed gaps & LAPs & Stage (ms) & Total (s) & Success\\
\midrule
\multicolumn{6}{@{}l}{Quatro~\cite{lim2022quatro}}\\
Projected native + gravity ICP & 0/527 & 0 & 0 & 0.678 & 6/31\\
AM + gravity ICP & 0/527 & 1,066 & 0.64 & 0.657 & 11/31\\
AM (multistart) + gravity ICP & 0/527 & 32,505 & 24.7 & 0.685 & 11/31\\
Edge bound + gravity ICP & 527/527 & 2,656,883 & 468 & 1.158 & 11/31\\
Vertex bound + gravity ICP & 527/527 & 32,165 & 14.8 & 0.669 & 11/31\\
Rubix + gravity ICP & 527/527 & \textbf{14,534} & \textbf{10.4} & \textbf{0.667} & 11/31\\
\midrule
\multicolumn{6}{@{}l}{RAP~\cite{pan2026rap}: two trials per pair}\\
Projected native + gravity ICP & 0/1054 & 0 & 0 & 11.774 & 32/62\\
AM + gravity ICP & 0/1054 & 2,124 & 0.61 & 11.818 & 32/62\\
AM (multistart) + gravity ICP & 0/1054 & 64,804 & 24.9 & 11.838 & 32/62\\
Edge bound + gravity ICP & 1054/1054 & 5,235,886 & 466 & 12.302 & 32/62\\
Vertex bound + gravity ICP & 1054/1054 & 64,300 & 14.8 & 11.834 & 32/62\\
Rubix + gravity ICP & 1054/1054 & \textbf{28,915} & \textbf{10.4} & \textbf{11.827} & 32/62\\
\midrule
\multicolumn{6}{@{}l}{Li et al.~\cite{li2024screw}}\\
Projected native + gravity ICP & 0/527 & 0 & 0 & 25.410 & 17/31\\
AM + gravity ICP & 0/527 & 1,063 & 0.60 & 25.321 & 20/31\\
AM (multistart) + gravity ICP & 0/527 & 32,365 & 25.1 & 25.347 & 20/31\\
Edge bound + gravity ICP & 527/527 & 2,623,251 & 461 & 25.759 & 20/31\\
Vertex bound + gravity ICP & 527/527 & 32,168 & 14.7 & 25.339 & 20/31\\
Rubix + gravity ICP & 527/527 & \textbf{14,477} & \textbf{10.4} & \textbf{25.332} & 20/31\\
\bottomrule
\end{tabular}
\end{table*}

%% file: appendix/native_gravity_comparison.tex
\subsection{Native Methods and Their Inputs}\label{app:native-gravity-comparison}

Native comparisons (\Cref{tab:native-gravity-comparison}) use \Cref{app:published-gravity}'s acceptance rules and Quatro's original reimplementation, without Quatro++~\cite{lim2024quatro}. Quatro~\cite{lim2022quatro}, RAP~\cite{pan2026rap} and ICP~\cite{kubelka2022} receive independent 20,000-point samples; Li et al.~\cite{li2024screw} and Cai et al.~\cite{cai2019} share full-scan ISS/FPFH candidates with ten targets per source. Different inputs make this contextual.

\paragraph{Gravity ICP and timing.}
Native libpointmatcher gravity ICP~\cite{kubelka2022} uses point-to-plane updates, 85\% trimming, 100 iterations and at most 6,000 independently voxelized vertical-surface points per set. It selects by residual from zero or 16 uniform yaws with centroid translation; the original system's five-run IMU/odometry selection and intervals are unavailable. Table times include startup/input but exclude Quatro/RAP serialization and Li/Cai features. Including features gives median times 25.074\,s (Li), 27.452\,s (FMP+BnB), 35.547\,s (BnB). Complete ICP times are 0.569\,s from one start and 2.731\,s from 16.

\paragraph{The solver of Cai et al.~\cite{cai2019}.}
The common-host comparisons isolate interpolation (\Cref{tab:native-with-without}); consensus is unchanged and there is no meaningful speedup. At the solver's resolution floor, one search stops with interval $[1,3]$, so termination alone need not close the numerical gap.

\input{tables/native_gravity_comparison.tex}

%% file: tables/native_gravity_comparison.tex
\begin{table*}[t]
\centering
\caption{\textbf{Native methods on the identified ETH scan pairs.} The results precede any added stage; the last two rows are ICP alone. Success requires a rotation error (RE) of at most $5^\circ$ and a translation error (TE) of at most $0.5$\,m. The medians include every valid output. The native time is the wall time of the process, summed over the ICP starts. The methods differ in input, objective and hardware, which limits direct comparisons of runtime; \Cref{app:native-gravity-comparison,app:computing} give the hardware and the preprocessing that the times omit.}
\label{tab:native-gravity-comparison}
\begin{tabular}{lrrrr}
\toprule
Native method & Success / 31 & Median RE ($^\circ$) & Median TE (m) & Native time (s) \\
\midrule
Quatro~\cite{lim2022quatro} & 4 & 81.521 & 7.281 & 0.289 \\
RAP, trial 1~\cite{pan2026rap} & 19 & 1.161 & 0.368 & 12.097 \\
RAP, trial 2~\cite{pan2026rap} & 18 & 1.143 & 0.294 & 11.181 \\
Li et al., $N_{\mathrm{cand}}=10$~\cite{li2024screw} & 23 & 0.150 & 0.064 & 2.921 \\
Cai et al., corrected FMP+BnB~\cite{cai2019} & 24 & 0.175 & 0.041 & 4.472 \\
Cai et al., corrected BnB~\cite{cai2019} & 24 & 0.175 & 0.041 & 3.671 \\
Gravity ICP, zero yaw~\cite{kubelka2022} & 5 & 15.472 & 5.765 & 0.111 \\
Gravity ICP, 16 yaws~\cite{kubelka2022} & 12 & 2.817 & 3.200 & 2.214 \\
\bottomrule
\end{tabular}
\end{table*}

%% file: appendix/native_extensions.tex
\subsection{Quatro Yaw Loss}\label{app:native-extensions}

For \Cref{tab:native-with-without}, we replace only the original yaw solver of Quatro~\cite{lim2022quatro}, which uses graduated non-convexity (GNC), and keep the order of the feature pairs, the clique, the scale, the translation estimator and the ICP. The translation-invariant measurements (TIMs) of Quatro give vectors $a_i,b_i\in\mathbb R^2$ and the cutoff $\tau_{\mathrm{clip}}=0.6\,\mathrm m$, with the objective
\begin{equation}\label{eq:quatro-clipped-objective}
 J(\theta)=\sum_i\min\{\|b_i-R_\theta a_i\|^2,\tau_{\mathrm{clip}}^2\}.
\end{equation}
Each residual $C_i-2A_i\cos\theta-2B_i\sin\theta$ is affine in the yaw coordinates, so the clipped sum is a finite minimum of affine functions and the arc bound of \Cref{lem:arc-interpolation} applies; a query costs $\mathcal O(N_{\mathrm{TIM}})$ operations for $N_{\mathrm{TIM}}$ TIMs. As a reference, we sort the yaws at which a residual crosses the threshold and minimize the sinusoid on each interval, which takes $\mathcal O(N_{\mathrm{TIM}}\log N_{\mathrm{TIM}})$ operations in exact arithmetic and covers constants, tangencies, coincident crossings and ties.

Every pair closes at a gap of $10^{-8}$ plus a floating-point allowance that is proportional to the machine precision and to the magnitudes of the coefficients, and the result agrees with the numerical reference from enumeration within $2.85\times10^{-14}$. The improvements in the loss are conditional on the fixed clique and translation. With 4--24 TIMs per pair, a faster yaw solver does not materially change the total time of the pipeline.

%% file: appendix/se3_qualitative.tex
\begin{figure*}[tp]
\centering
\includegraphics[width=\textwidth]{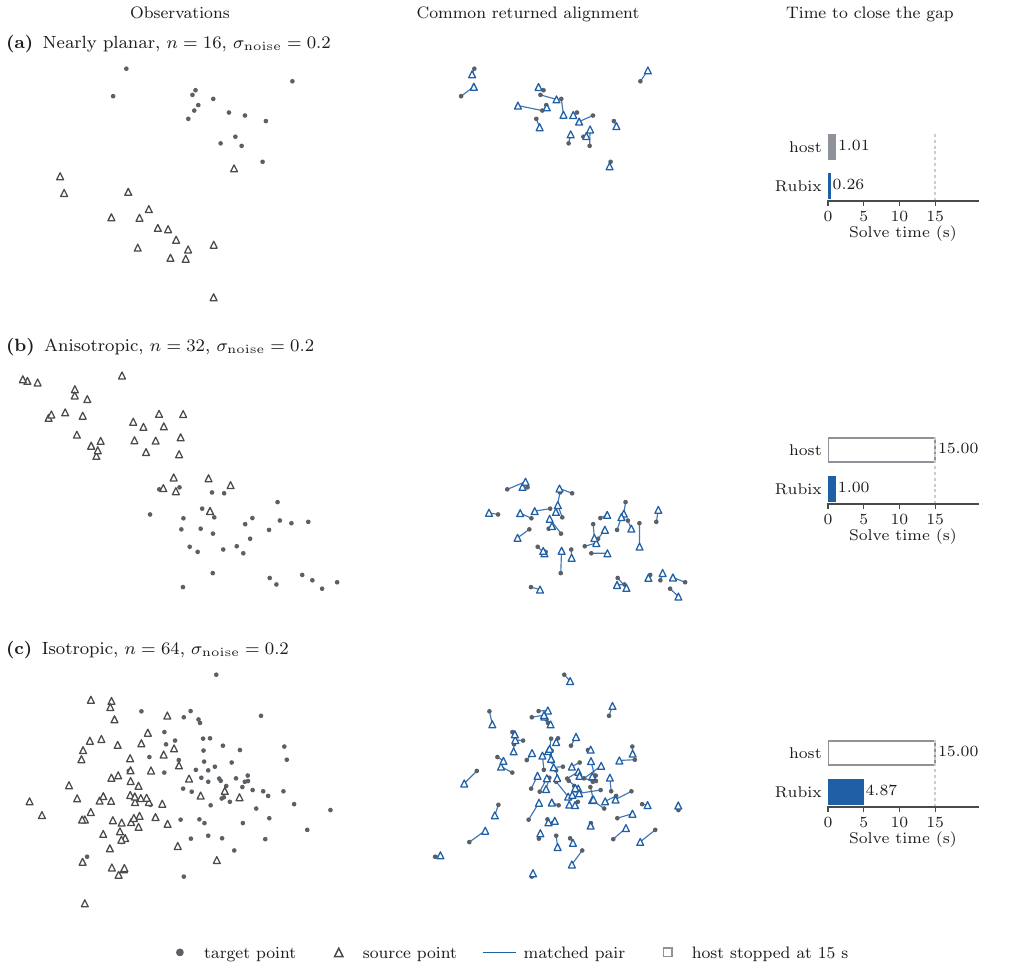}
\caption{\textbf{Less search for the same SE(3) alignment.} One favorable example per geometry, selected from all 72 scenes, including those on which the host reaches the time limit. The columns show the observations, the common fitted pose with its residuals, and the solve times; each row uses one viewpoint and scale. Hollow bars mark the limit of 15\,s. At a target of $10^{-5}$ for the numerical gap in the mean squared error, Rubix closes all three cases, whereas the host with its rotation-cap bounds alone stops above the target in (b,c). Both variants return identical poses and matchings throughout the study.}\label{fig:se3-qualitative-gallery-1}
\end{figure*}

%% file: appendix/additional_results.tex
\section[Additional and Contextual Results]{Additional and Contextual\\Results}\label{app:additional}

\subsection{Rubix in Published Methods}\label{app:native-with-without}

\paragraph{Bounds inside published solvers.} On identical hosts, Rubix gives no meaningful speedup for Cai et al.'s~\cite{cai2019} FMP+BnB or BnB: paired time ratios are $1.0004$ and $0.9275$, with unchanged consensus (\Cref{tab:native-with-without}). It lowers Quatro's~\cite{lim2022quatro} truncated yaw loss on 17/31 pairs without changing pose success (\Cref{app:native-extensions}). Native inputs and poses appear in \Cref{app:native-gravity-comparison}.

\paragraph{Including the front end.} End-to-end speedups over the independent-edge bound fall to $1.62$, $1.03$, $1.01$ after Quatro, RAP and Li et al.; front ends dominate the roughly 10\,ms matching stage (\Cref{tab:matching-with-without}). Single- and multistart AM match global solvers' success counts. Against projected native poses plus ICP, matching raises Quatro's successes from 6 to 11 and Li et al.'s from 17 to 20, leaving RAP at 32/62. Native RAP and Li et al. already succeed on 37 and 23 cases, so matching does not improve every native output.

\paragraph{Without a front end.} Sixteen equally spaced yaws give identical final outcomes for AM and Rubix in 496 cases. Yaw-dependent subset selection places optima near initialization, letting AM solve 492/496 subproblems and limiting the benefit of global optimization (\Cref{app:published-gravity}).

\input{tables/native_with_without.tex}

\subsection{Gravity-Aligned Go-ICP}\label{app:goicp}

Original Go-ICP~\cite{yang2016goicp} is restricted to yaw and translation, bisecting yaw intervals with displacement radius $2\sin(h/2)\|x_{i,xy}\|$ and retaining its translation search and stopping rule. All variants disable six-degree-of-freedom ICP. Rubix adds arc or common-translation bounds after native pruning (\Cref{lem:common-translation}). ETH uses 1,000 sources, full targets, trimming $0.3$ and normalized per-inlier tolerance $2\times10^{-6}$. Twenty Bunny cases use $10\%$ outliers, trimming $0.1$ and tolerance $10^{-5}$. Each receives 900\,s on one core.

Arc bounds reduce median rotation nodes by $1.40$ (ETH) and $1.64$ (Bunny), and translation nodes by $1.44$ and $1.30$; common-translation factors are $1.33$, $1.64$ and $1.15$, $1.38$. Objectives agree within $0.33\varepsilon_{\mathrm{GoICP}}$, where $\varepsilon_{\mathrm{GoICP}}$ is per-inlier tolerance times retained count; identical poses yield 20/20 Bunny and 3/31 ETH successes. Native/Rubix time ratios are $0.96$, $0.72$ for arc bounds and $0.66$, $0.99$ for common translation: fewer nodes do not save time. Rubix needs at least three objective evaluations per cell versus one distance-transform lookup per point. Native and arc variants finish 28/31 ETH searches; common translation finishes 25/31.

\subsection{Complete ETH Pipelines}\label{app:eth-pipelines}

Including every pipeline stage on the six pairs, Rubix gives median speedup $1.70$ over the independent-edge bound (range $1.50$--$2.15$; \Cref{tab:eth-pipeline}). All four solvers return identical poses, three within $5^\circ$ and $0.5$\,m. \Cref{fig:eth-qualitative} shows two successes; \Cref{app:eth-protocol} gives the protocol and complete-system comparisons.

\begin{figure*}[tp]
\centering
\includegraphics[width=\textwidth]{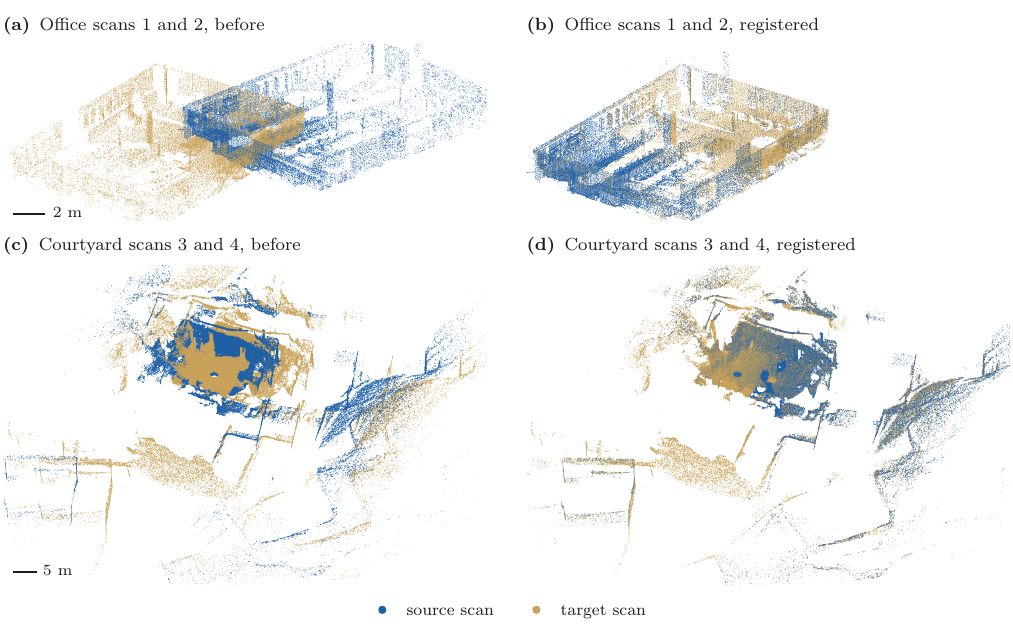}
\caption{\textbf{Gravity-aligned registration of ETH scans.} Office scans 1--2 (a,b; ceiling removed) and Courtyard scans 3--4 (c,d), before and after registration with the nominal vertical of the scanner as the gravity direction. All four solvers return the same pose. Relative to the independent-edge bound, Rubix reduces the time of the complete pipeline from 267 to 177\,ms and from 393 to 182\,ms; the rotation and translation errors are $0.546^\circ$ and $0.124$\,m for Office and $0.047^\circ$ and $0.029$\,m for Courtyard. Each row uses one viewpoint, one scale and one set of displayed points, and the projected scale bars are in meters. The two pairs are selected among the successful ones; \Cref{tab:eth-pipeline} reports all six.}\label{fig:eth-qualitative}
\end{figure*}

\subsection{Quaternion Bounds on Subsets of Real Scans}\label{app:se3-real-results}

On RAP-selected subsets, Rubix reduces cells throughout, by a median factor $1.70$ and $70.4\%$ overall, while assignments barely change (130,773 versus 132,648). It is slower: median host/Rubix time ratio $0.56$, with speedups on only 3/31 pairs. Both close all gaps with identical objectives, poses and matchings. Centroid selection similarly reduces cells throughout (factor $1.78$), yet gives time ratio $0.58$ and speedups on 4/31 pairs (\Cref{tab:se3-real}).

\input{tables/se3_real.tex}

At $5^\circ$ and $0.5$\,m, refitted RAP subsets succeed on 10/31 pairs versus native RAP's 19/31: nine shared successes, ten RAP-only and one refit-only. Refit median errors are $1.82^\circ$ and $0.90$\,m. Centroid selection succeeds on 0/31, with medians $14.28^\circ$ and $9.29$\,m. Global subset fitting therefore need not recover the full-scan pose.

\subsection{GLORES Examples and PLICP}\label{app:glores-examples}

Native PLICP~\cite{censi2008} rises from 146/212 to 185/212 successes with 16 yaws. Rubix accelerates GLORES at the same tolerance and essentially unchanged pose success, 190/212 versus 189/212 (\Cref{sec:native-planar,app:native-planar,fig:native-planar}).

\begin{figure*}[tp]
\centering
\includegraphics[width=\textwidth]{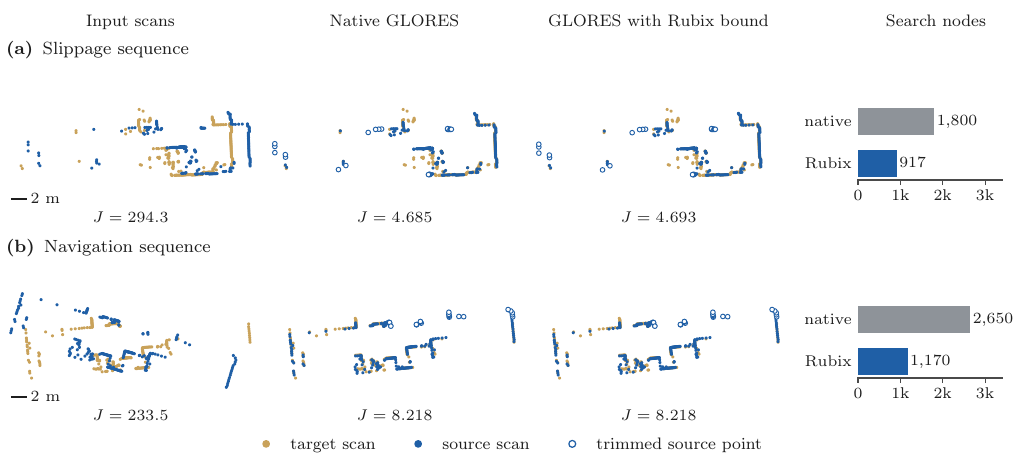}
\caption{\textbf{Less search inside native GLORES~\cite{glores2020}.} Scans from the Slippage and Navigation sequences before registration and after fitting with the native bound and with the Rubix bound. $J$ is the sum of squared residuals over the retained points, in m$^2$; hollow points are trimmed (216 of 239 and 220 of 244 points are retained). Both solvers reach the $1\%$ tolerance with equivalent poses. Rubix reduces the nodes from 1,800 to 917 and from 2,650 to 1,170, and the times from 2.85 to 1.96\,s and from 3.97 to 2.29\,s. Each row uses one display rotation and scale. For each sequence we show the case with the largest objective before registration among those with 256 beams and $\rho_{\mathrm{keep}}=0.9$ on which both poses succeed; \Cref{tab:native-planar} covers all 240 configurations.}\label{fig:native-planar}
\end{figure*}

\subsection{Sensitivity to the Gravity Direction}\label{sec:exp-gravity-bias}

To test gravity error, we tilt both supplied directions by $0$, $0.25$, $0.5$, $1$, $2$ and $5^\circ$ with independent azimuths. The 48 complete landmark scenes cross $n\in\{32,128\}$, vertical spreads $0.05$, $0.5$, noise levels $0.005$, $0.02$, and six replicates in meters. Noise and azimuths stay fixed across tilt magnitudes. All 288 joint yaw--translation fits close summed gaps at $10^{-6}$ (\Cref{fig:gravity-prior}).

\begin{figure*}[t]
\centering
\includegraphics[width=\textwidth]{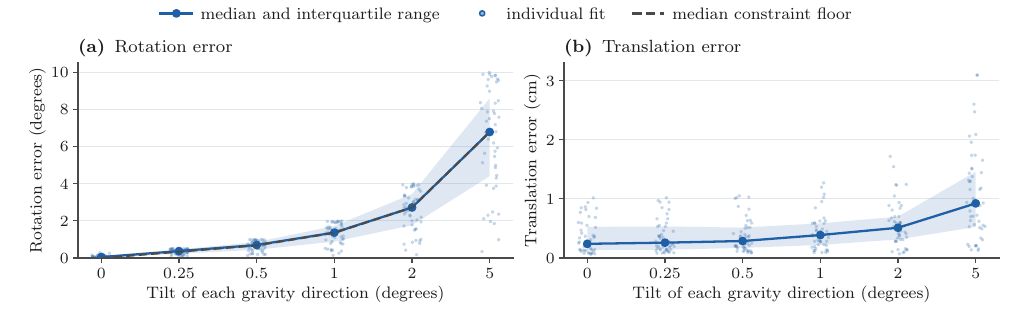}
\caption{\textbf{Sensitivity to a simulated error in the gravity direction.} Independent tilts of both supplied directions at six magnitudes give 288 fits on 48 scenes. The horizontal positions are equally spaced and labeled by the tilt angle; faint points show every error, blue lines the medians and bands the interquartile ranges. (a)~Full rotation error, and the median of the floor $\angle(R_*\widehat g_y,\widehat g_x)$ imposed by the constraint (dashed). (b)~Translation error after centering, in centimeters. Every fit reaches a gap of $10^{-6}$ in the summed squared error.}\label{fig:gravity-prior}
\end{figure*}

The median rotation error rises from $0.051^\circ$ with the exact gravity direction to $0.699^\circ$ and $1.369^\circ$ at tilts of $0.5^\circ$ and $1^\circ$. This error is imposed by the constraint: a rotation that maps $\widehat g_y$ to $\widehat g_x$ differs from the true rotation $R_*$ by an angle of at least $\angle(R_*\widehat g_y,\widehat g_x)$. At the two tilts, the medians of this lower bound are $0.679^\circ$ and $1.358^\circ$, close to the observed errors. The 95th percentile of the translation error rises from $0.86$\,cm to $1.06$\,cm at a tilt of $1^\circ$. Closing the gap of the objective thus solves the constrained problem, but it does not remove the pose error that a biased gravity direction imposes.

\subsection{Bound-Cost Policies for Full Rotation}\label{app:spatial-ablations}

\Cref{tab:spatial-budget} compares five SE(3) policies: rotation-cap bounds alone, one-assignment differences, full exterior queries, reused assignment duals and adaptive selection (\Cref{app:se3-protocol}). Full exterior queries close the most gaps in the least total time.

\begin{table}[t]
\caption{\textbf{Complete SE(3) alignment on 72 scenes} with the same host. Each run has a limit of 15\,s and a target of $10^{-5}$ for the gap in the mean squared error. The CPU time is summed over all runs, including the unresolved ones; every policy uses one core. Bold marks more closed gaps or less CPU time than the host with its rotation-cap bounds alone.}\label{tab:spatial-budget}
\centering
\begin{tabular}{lrr}
\toprule
Added bound & Closed & CPU (s)\\
\midrule
None (rotation-cap host) & 63/72 & 189.2\\
One-assignment difference & 63/72 & 206.8\\
Full exterior supports & \textbf{72/72} & \textbf{59.1}\\
Reused assignment duals & 61/72 & 237.5\\
Adaptive bound selection & \textbf{67/72} & \textbf{139.2}\\
\bottomrule
\end{tabular}
\end{table}

Full queries mainly help host timeouts. On 27 high-noise common successes, median host/full-query time ratio is $0.92$; a separate low-noise comparison gives $0.76$. At low noise, adaptation is $1.17$ times faster than full queries but closes fewer gaps.

%% file: tables/native_with_without.tex
\begin{table*}[t]
\centering
\caption{\textbf{Published solvers with and without Rubix bounds.} Values are medians; paired ratios divide baseline by Rubix runtime or Quatro yaw loss. Bold marks improvements. GLORES includes preprocessing/bounds at $1\%$ tolerance; Cai timings exclude features and use common hosts with identical consensus. Quatro's clipped yaw loss improves by more than $10^{-8}$ on 17/31 pairs, by at most $0.879821$\,m$^2$, with unchanged success (4/31 before, 6/31 after ICP); it is not timed.}\label{tab:native-with-without}
\begin{tabular}{@{}p{0.44\textwidth}rrrr@{}}
\toprule
Solver and baseline & Runs & Without & With Rubix & Paired ratio\\
\midrule
\multicolumn{5}{@{}l}{\textit{Timed host (s)}}\\
GLORES~\cite{glores2020}: native bound & 240 & 2.330 & \textbf{1.608} & $\mathbf{1.3313}\times$\\
Cai FMP~\cite{cai2019}: common host & 31 & 2.396 & 2.413 & $\mathbf{1.0004}\times$\\
Cai BnB~\cite{cai2019}: common host & 31 & 1.230 & 1.232 & $0.9275\times$\\
\midrule
\multicolumn{5}{@{}l}{\textit{Conditional yaw TLS loss (m$^2$)}}\\
Quatro~\cite{lim2022quatro}: native GNC yaw & 31 & 1.5396 & \textbf{1.4816} & $\mathbf{1.0073}\times$\\
\bottomrule
\end{tabular}
\end{table*}

\begin{table*}[t]
\centering
\caption{\textbf{Matching stages with baseline and Rubix bounds.} Arrows show baseline $\to$ Rubix under identical subsets, starts, objectives, pose selection and ICP. Edge and vertex denote the independent-edge and vertex bounds. Times are medians; ratios are paired baseline/Rubix medians. Totals add identical front-end times. Quatro and Li run once per pair; RAP runs twice. All 2,108 global subset gaps close with unchanged $5^\circ$, $0.5$\,m success counts; bold marks time improvements.}\label{tab:matching-with-without}
\begin{tabular}{@{}lrrrrr@{}}
\toprule
Native method: baseline bound & Stage (ms) & Stage ratio & Total (s) & Total ratio & Success\\
\midrule
Quatro~\cite{lim2022quatro}: edge & $468\to\mathbf{10.4}$ & $\mathbf{46.94}\times$ & $1.158\to\mathbf{0.667}$ & $\mathbf{1.617}\times$ & $11/31\to11/31$\\
Quatro~\cite{lim2022quatro}: vertex & $14.8\to\mathbf{10.4}$ & $\mathbf{1.47}\times$ & $0.669\to\mathbf{0.667}$ & $\mathbf{1.006}\times$ & $11/31\to11/31$\\
RAP~\cite{pan2026rap}: edge & $466\to\mathbf{10.4}$ & $\mathbf{46.31}\times$ & $12.302\to\mathbf{11.827}$ & $\mathbf{1.033}\times$ & $32/62\to32/62$\\
RAP~\cite{pan2026rap}: vertex & $14.8\to\mathbf{10.4}$ & $\mathbf{1.48}\times$ & $11.834\to\mathbf{11.827}$ & $1.000\times$ & $32/62\to32/62$\\
Li et al.~\cite{li2024screw}: edge & $461\to\mathbf{10.4}$ & $\mathbf{45.79}\times$ & $25.759\to\mathbf{25.332}$ & $\mathbf{1.011}\times$ & $20/31\to20/31$\\
Li et al.~\cite{li2024screw}: vertex & $14.7\to\mathbf{10.4}$ & $\mathbf{1.46}\times$ & $25.339\to\mathbf{25.332}$ & $1.000\times$ & $20/31\to20/31$\\
\bottomrule
\end{tabular}
\end{table*}

%% file: tables/se3_real.tex
\begin{table*}[t]
\centering
\caption{\textbf{SE(3) matching on real ETH subsets.} All identified pairs use 32-point subsets, 15\,s and mean-squared gap target $10^{-5}$\,m$^2$. Cells include initial regions. CPU time excludes common subset preparation and RAP; medians are elapsed solver times. Both bounds return identical poses and matchings; success requires $5^\circ$ and $0.5$\,m.}\label{tab:se3-real}
\begin{tabular}{llrrrrr}
\toprule
Subset selection & Bound & Gap closed & Cells & CPU (s) & Median (ms) & Pose success\\
\midrule
RAP full pose & Rotation-cap host & 31/31 & 28,074 & 9.73 & 203.1 & 10/31\\
RAP full pose & With Rubix & 31/31 & 8,302 & 12.10 & 300.5 & 10/31\\
\midrule
Centroid control & Rotation-cap host & 31/31 & 39,920 & 15.23 & 393.7 & 0/31\\
Centroid control & With Rubix & 31/31 & 12,232 & 18.72 & 553.0 & 0/31\\
\bottomrule
\end{tabular}
\end{table*}